\documentclass{article}

\PassOptionsToPackage{numbers}{natbib}
    \usepackage[final]{neurips_2021}

\usepackage[utf8]{inputenc} % allow utf-8 input
\usepackage[T1]{fontenc}    % use 8-bit T1 fonts
\usepackage{hyperref}       % hyperlinks
\usepackage{url}            % simple URL typesetting
\usepackage{booktabs}       % professional-quality tables
\usepackage{amsfonts}       % blackboard math symbols
\usepackage{nicefrac}       % compact symbols for 1/2, etc.
\usepackage{microtype}      % microtypography
\usepackage{xcolor}         % colors

\usepackage{amsmath,amsfonts,amssymb,mathtools, amsthm}
\usepackage[noend]{algorithmic}
\usepackage{centernot}
\usepackage{multirow}
\usepackage[inline]{enumitem}
\usepackage{tikz, subcaption}
\usepackage[T1]{fontenc}
\usetikzlibrary{shapes, arrows, positioning}
\usepackage[ruled,vlined,linesnumbered]{algorithm2e}

\SetKw{Break}{break}
\usepackage{array}
\newcolumntype{M}[1]{>{\centering\arraybackslash}m{#1}}
\newcolumntype{N}{@{}m{0pt}@{}}

\DeclareMathAlphabet\mathbfcal{OMS}{cmsy}{b}{n}

\newtheorem{theorem}{Theorem}%[section]
\newenvironment{customthm}[1]{\renewcommand\thetheorem{#1}\theorem}{\endtheorem}
\newtheorem{lemma}{Lemma}
\newtheorem{proposition}{Proposition}
\newenvironment{customprp}[1]{\renewcommand\theproposition{#1}\proposition}{\endproposition}
\newtheorem{definition}{Definition}

\newcommand{\independent}{\perp\mkern-9.5mu\perp}
\newcommand{\notindependent}{\centernot{\independent}}

\newcommand{\CI}[4]{(#1 \independent #2 \vert #3)_{#4}}
\newcommand{\notCI}[4]{(#1 \notindependent #2 \vert #3)_{#4}}
\newcommand{\msep}[4]{(#1 \perp #2 \vert #3)_{#4}}

\newcommand{\Mb}[2]{\textit{Mb}_{#2}(#1)}
\newcommand{\MB}[1]{\textit{Mb}_{#1}}
\newcommand{\Pa}[1]{\textit{Pa}(#1)}
\newcommand{\Ch}[1]{\textit{Ch}(#1)}
\newcommand{\Adj}[1]{\textit{Adj}(#1)}
\newcommand{\N}[1]{\textit{N}(#1)}
\newcommand{\De}[1]{\textit{De}(#1)}
\newcommand{\Anc}[1]{\textit{Anc}(#1)}
\newcommand{\Dis}[1]{\textit{Dis}(#1)}
\newcommand{\PaP}[1]{\textit{Pa}^+(#1)}

\newcommand{\G}{\mathcal{G}}
\newcommand{\deltaplus}[1]{\Delta^+_{\text{in}}(#1)}
\newcommand{\V}{\mathbfcal{V}}
\newcommand{\GV}{\mathcal{G}_\mathbfcal{V}}
\newcommand{\PV}{P_\mathbfcal{V}}
\newcommand{\POS}{P_{\mathbfcal{O}\vert \mathbfcal{S}}}
\newcommand{\A}{\mathcal{A}}

\newcommand{\setminusA}{\!\setminus\!}

\title{Recursive Causal Structure Learning in the Presence of Latent Variables and Selection Bias}

\author{%
    Sina Akbari \\
    Department of Computer and \\ Communication Sciences \\
    EPFL, Lausanne, Switzerland \\
    \texttt{sina.akbari@epfl.ch}
    \And
    Ehsan Mokhtarian \\
    Department of Computer and \\ Communication Sciences \\
    EPFL, Lausanne, Switzerland \\
    \texttt{ehsan.mokhtarian@epfl.ch}
    \And
    AmirEmad Ghassami \\
    Department of Computer Science \\
    Johns Hopkins University, Baltimore, USA \\
    \texttt{aghassa1@jhu.edu}
    \And
    Negar Kiyavash \\
    College of Management of Technology \\
    EPFL, Lausanne, Switzerland \\
    \texttt{negar.kiyavash@epfl.ch}
}

\begin{document}

\maketitle

\begin{abstract}
    We consider the problem of learning the causal MAG of a system from observational data in the presence of latent variables and selection bias.
    Constraint-based methods are one of the main approaches for solving this problem, but the existing methods are either computationally impractical when dealing with large graphs or lacking completeness guarantees.
    We propose a novel computationally efficient recursive constraint-based method that is sound and complete. The key idea of our approach is that at each iteration a specific type of variable is identified and removed. This allows us to learn the structure efficiently and recursively, as this technique reduces both the number of required conditional independence (CI) tests and the size of the conditioning sets. 
    The former substantially reduces the computational complexity, while the latter results in more reliable CI tests.
    We provide an upper bound on the number of required CI tests in the worst case. To the best of our knowledge, this is the tightest bound in the literature.
    We further provide a lower bound on the number of CI tests required by any constraint-based method. 
    The upper bound of our proposed approach and the lower bound at most differ by a factor equal to the number of variables in the worst case. 
    We provide experimental results to compare the proposed approach with the state of the art on both synthetic and real-world structures.
\end{abstract}
\section{Introduction}\label{sec: intro}
    Learning the causal structure among the set of variables in the system is the initial step for performing statistical inference tasks such as estimating the reward of a policy in off-policy evaluation \cite{thomas2016data, jiang2016doubly, murphy2003optimal}, etc.
    In the literature, structure learning is for the most part done under the assumption that all the variables in the system are observed \cite{spirtes2000causation, margaritis1999bayesian, pellet2008using, mokhtarian2020recursive, tsamardinos2003time}.
    However, in many applications in real-life systems, this assumption is violated. 
    Moreover, the accessible data may contain selection bias, i.e., some of the variables may have been conditioned on.
    
    The problem of causal structure learning is significantly more challenging when unmeasured (latent) confounders and selection variables exist in the system.
    This is because the set of directed acyclic graphs (DAGs) as independence models, which is the predominant modeling approach in the absence of unobserved variables, is not closed under marginalization and conditioning \cite{richardson2002ancestral}.
    That is, there does not necessarily exist a DAG over the observed variables that demonstrate a one-to-one map with the conditional independence relationships in the observational distribution $P_{\mathbfcal{O}\vert \mathbfcal{S}}$, where $\mathbfcal{O}$ and $\mathbfcal{S}$ denote the observed variables and selection variables, respectively.
    To address this problem, several extensions of the DAG models, such as acyclic directed mixed graphs (ADMGs) \cite{richardson2003markov}, induced path graphs (IPGs) \cite{spirtes2000causation}, and maximal ancestral graphs (MAGs) \cite{richardson2002ancestral} are introduced in the literature.
    
    The main approaches for structure learning include constraint-based and score-based methods \cite{spirtes2000causation, tsirlis2018scoring, ogarrio2016hybrid, colombo2012learning}.
    There are also methods that require specific assumptions on the data generating modules, such as requiring linearity\cite{zhang2020simultaneous}, linearity and non-Gaussianity of the noises \cite{shimizu2006linear} or additivity of the noise with specific types of non-linearity \cite{hoyer2009nonlinear}
    (See \cite{zhang2017learning} for a summary of structure learning approaches.)
    Constraint-based methods are the most commonly used methods for structure learning in the presence of latent variables and selection bias \cite{spirtes2000causation, tillman2008integrating, colombo2012learning, strobl2019constraint, pellet2008finding}. 
    The main idea in these methods is to find the structure which is most consistent with the conditional independence (CI) relationships in the data \cite{spirtes2000causation}. 
    However, the sheer number of CI tests required by these methods prohibits applying them to systems with large number of variables.
    
    Several methods are proposed in the literature to reduce the number of CI tests needed in constraint-based methods, specifically when there are no latent and selection variables in the system.
    For instance, \cite{spirtes2000causation} proposed the seminal PC algorithm for graphs with bounded degree, which has polynomial complexity in the number of vertices. 
    \cite{margaritis1999bayesian, pellet2008using, mokhtarian2020recursive} proposed using Markov boundary information to reduce the number of required CI tests.
    If the size of the Markov boundaries or the in-degree of the variables is bounded, these methods achieve quadratic complexity in the number of the variables.
    However, the majority of the work on causal structure learning in the presence of latent and selection variables do not provide any analysis for the required number of CI tests.
    As an exception, for sparse graphs, and given the exact value of the maximum degree of the MAG as side information, \cite{claassen2013learning} proposed an algorithm that requires a polynomial number of CI tests in the number of variables.
    Additionally, \cite{colombo2012learning} proposed a modification of the FCI algorithm, called RFCI, with specific attention to its time complexity. 
    However, RFCI is not complete; that is, the output of this algorithm does not capture all the CI relationships in the data.
    
    In this paper, we propose a novel recursive constraint-based method for causal structure learning in the presence of latent confounders and selection bias. We use MAGs as the graphical representation of the system.
    The main idea of our recursive approach is that in each iteration, we choose a particular variable of the system, say $X$, and locally identify its adjacent variables.
    Then, we recursively learn the structure over the rest of the variables using the marginal distribution $P_{\mathbfcal{O} \setminus \{X\}\vert \mathbfcal{S}}$.
    Note that the choice of $X$ cannot be arbitrary.
    For instance, consider the DAG $\GV$ in Figure \ref{fig: a.rem}, where the variable $U$ is latent. 
    The causal MAG over $\mathbfcal{O}=\{T,W,Y,Z\}$ is shown in Figure \ref{fig: b.rem} as MAG $\mathcal{M}$. 
    As seen in Figure \ref{fig: c.rem}, if we start with the choice of $X=T$, 
    we can correctly learn the subgraph of $\mathcal{M}$ over $\{Y,W,Z\}$, whereas if we start with $X=W$, we will end up learning the graph in Figure \ref{fig: d.rem}, which has two extra edges (highlighted in red) between $Y,Z$ and $Y,T$  that do not exist in $\mathcal{M}$ (we will revisit this example in Section \ref{sec: L-MARVEL}).

    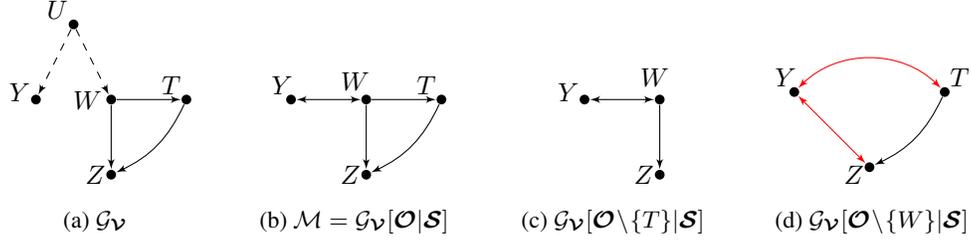
\begin{figure}[t] 
	    \centering
		\tikzstyle{block} = [circle, inner sep=1.3pt, fill=black]
		\tikzstyle{input} = [coordinate]
		\tikzstyle{output} = [coordinate]
		\begin{subfigure}[b]{0.24\textwidth}
    	    \centering
    	    \begin{tikzpicture}
                \tikzset{edge/.style = {->,> = latex'}}
                % vertices
                \node[block] (u) at  (0.5,2) {};
                \node[] ()[above left=-0.1cm and -0.1cm of u]{$U$};
                \node[block] (t) at  (2,1) {};
                \node[] ()[above left=-0.1cm and -0.1cm of t]{$T$};
                \node[block] (y) at  (0,1) {};
                \node[] ()[above left=-0.2cm and -0.1cm of y]{$Y$};
                \node[block] (x) at  (1,1) {};
                \node[] ()[above left=-0.3cm and -0.05cm of x]{$W$};
                \node[block] (z) at  (1,0) {};
                \node[] ()[above left=-0.3cm and -0.1cm of z]{$Z$};
                %edges
                \draw[edge] (x) to (z);
                \draw[edge, dashed] (u) to (y);
                \draw[edge, dashed] (u) to (x);
                \draw[edge] (x) to (t);
                \draw[edge, bend left=20] (t) to (z);
            \end{tikzpicture}
        \caption{$\GV$}\label{fig: a.rem}
		\end{subfigure}
        \begin{subfigure}[b]{0.24\textwidth}
    	    \centering
    	    \begin{tikzpicture}
                \tikzset{edge/.style = {->,> = latex'}}
                % vertices
                \node[block] (t) at  (2,1) {};
                \node[] ()[above left=-0.1cm and -0.1cm of t]{$T$};
                \node[block] (y) at  (0,1) {};
                \node[] ()[above left=-0.1cm and -0.2cm of y]{$Y$};
                \node[block] (x) at  (1,1) {};
                \node[] ()[above left=-0.05cm and -0.22cm of x]{$W$};
                \node[block] (z) at  (1,0) {};
                \node[] ()[above left=-0.3cm and -0.1cm of z]{$Z$};
                %edges
                \draw[edge] (x) to (z);
                %\draw[edge] (x) to (y);
                \draw[edge,style={<->}] (y) to (x);
                \draw[edge] (x) to (t);
                \draw[edge, bend left=20] (t) to (z);
                %\path (z2) to node {\dots} (zp);
            \end{tikzpicture}
        \caption{$\mathcal{M}=\GV[\mathbfcal{O}\vert\mathbfcal{S}]$}
        \label{fig: b.rem}
		\end{subfigure}
		\begin{subfigure}[b]{0.24\textwidth}
    	    \centering
    	    \begin{tikzpicture}
                \tikzset{edge/.style = {->,> = latex'}}
                % vertices
                \node[block] (y) at  (0,1) {};
                \node[] ()[above left=-0.2cm and -0.1cm of y]{$Y$};
                \node[block] (x) at  (1,1) {};
                \node[] ()[above left=-0.0cm and -0.3cm of x]{$W$};
                \node[block] (z) at  (1,0) {};
                \node[] ()[above left=-0.3cm and -0.1cm of z]{$Z$};
                %edges
                \draw[edge] (x) to (z);
                \draw[edge,style={<->}] (y) to (x);
            \end{tikzpicture}
        \caption{$\GV[\mathbfcal{O}\setminusA\{T\}\vert\mathbfcal{S}]$}
        \label{fig: c.rem}
		\end{subfigure}
		\begin{subfigure}[b]{0.24\textwidth}
    	    \centering
    	    \begin{tikzpicture}
                \tikzset{edge/.style = {->,> = latex'}}
                % vertices
                \node[block] (t) at  (2,1) {};
                \node[] ()[above right=-0.1cm and -0.1cm of t]{$T$};
                \node[block] (y) at  (0,1) {};
                \node[] ()[above left=-0.1cm and -0.2cm of y]{$Y$};
                \node[block] (z) at  (1,0) {};
                \node[] ()[above left=-0.4cm and -0.1cm of z]{$Z$};
                %edges
                \draw[edge, bend left=20] (t) to (z);
                \draw[edge,red,style={<->},bend left=45] (y) to (t);
                \draw[edge,red,style={<->}] (y) to (z);
            \end{tikzpicture}
        \caption{$\GV[\mathbfcal{O}\setminusA\{W\}\vert\mathbfcal{S}]$}
        \label{fig: d.rem}
		\end{subfigure}
        \caption{Effect of removing a variable on the MAG over the remaining variables.}
        \label{fig: rem}
    \end{figure}
    
    Our main contributions are as follows.
    \begin{itemize}[leftmargin=*]
        \item 
            We introduce the notion of a \emph{removable} variable in MAGs, which is a variable that can be removed from the causal graph without changing the m-separation relations (Definition \ref{def: removable}). 
            We further represent a method to test the removability of a variable given the observational data (Theorem \ref{thm: test removability}). 
        \item 
            We propose an algorithm called L-MARVEL for causal structure learning in the presence of latent and selection variables. 
            We show that our method is sound and complete (Theorem \ref{thm: sound and complete}) and performs $\mathcal{O}(n^2 + n{\deltaplus{\mathcal{M}}}^2 2^{\deltaplus{\mathcal{M}}})$ CI tests in the worst case (Proposition \ref{prp: upper-bound}), where $n$ denotes the number of variables and $\deltaplus{\mathcal{M}}$ is the maximum size of the union of parents, district and the parents of district of a vertex in the MAG (Equation \eqref{eq: def deltaplus}).
        \item  
            We show that any constraint-based algorithm requires $\Omega(n^2+n{\deltaplus{\mathcal{M}}}2^{\deltaplus{\mathcal{M}}})$ CI tests in the worst case (Theorem \ref{thm: lwrBound}). 
            Comparing this lower bound with our upper bound demonstrates the efficiency of our proposed method.
    \end{itemize}
    To sum up, the purpose and desirability of the proposed recursive algorithm for structure learning is two fold. 
    First, since we choose specific (\textit{removable}) variables in each iteration (with the property of having small Markov boundary), we ensure that the number of required CI tests in each iteration, and hence in total, remains small. 
    Therefore, we can significantly reduce the time complexity compared to non-recursive approaches.
    Second, by virtue of the gradual reduction of the order of the graph over the iterations, the size of the conditioning sets used in each CI test is reduced, which results in more reliable CI tests with smaller errors and more accurate results.
    
    This paper is organized as follows.
    In Section \ref{sec: preliminaries}, we review the preliminaries, present the terminology, and formally describe the problem.
    In Section \ref{sec: L-MARVEL}, we present the L-MARVEL method along with its analysis. In Section \ref{sec: complexity} we also provide the universal lower bound on the complexity of every constraint-based method. Finally, Section \ref{sec: experiment} presents a comprehensive set of experiments to compare L-MARVEL with various algorithms on synthetic and real-world structures.
\section{Preliminaries and problem description} \label{sec: preliminaries}
\subsection{Terminology}
    A \emph{mixed graph} $\mathcal{G}$ over the set of vertices $\mathbf{V}$ is a graph containing three types of edges $-,\to$ and $\leftrightarrow$. The two ends of an edge are called \emph{marks}. There are two kinds of marks: \emph{arrowhead} (\textgreater) and \emph{tail} ($-$). If there exists a \emph{directed edge} $X\to Y$ in the graph, we say $X$ is a \emph{parent} of $Y$ and $Y$ is a \emph{child} of $X$. For a \emph{bi-directed edge}  $X\leftrightarrow Y$, we say $X$ and $Y$ are \emph{spouses}. For an \emph{undirected edge} $X-Y$, $X$ and $Y$ are called \emph{neighbors}. In all of the aforementioned cases, we say $X$ and $Y$ are \emph{adjacent}. The \emph{skeleton} of $\mathcal{G}$ is an undirected graph with the same set of vertices $\mathbf{V}$ where there is an edge between $X$ and $Y$ if they are adjacent in $\mathcal{G}$. A path from $X$ to $Y$ where every vertex on the path is a child of its preceding vertex is called a \emph{directed path}. If a directed path exists from $X$ to $Y$, $X$ is called an \emph{ancestor} of $Y$. We assume every vertex is an ancestor of itself. We denote by $\Pa{X}$, $\Ch{X}$, $\N{X}$, $\Adj{X}$, and $\Anc{X}$, the set of parents, children, neighbors, adjacent vertices, and ancestors of $X$, respectively. The \emph{district set} of a variable $X$, denoted by $\Dis{X}$, is the set of variables that have a path to $X$ comprised of only bidirectional edges. By $\PaP{X}$ we denote the union of parents, district set, parents of district set, and the neighbors of a variable\footnote{The motivation behind this definition is that the local Markov property does not necessarily hold when causal sufficiency is violated, but if $X\not\in \Anc{\Dis{X}}$, then $\PaP{X}$ separates $X$ from its non-descendants. See the supplementary material for proofs.}, i.e., 
    \begin{equation}\label{eq: parentplus}
        \PaP{X}=\Pa{X}\cup\Dis{X}\cup\Pa{\Dis{X}}\cup\N{X}.
    \end{equation}
    Uppercase capitals indicate single vertices, whereas bold letters denote sets of vertices. For a set of vertices $\mathbf{X}$,  $\Anc{\mathbf{X}}=\cup_{X\in\mathbf{X}}\Anc{X}$. A non-endpoint vertex $X$ on a path is called a \emph{collider}, if both of the edges incident to $X$ on the path have an arrowhead at $X$. A path $\mathcal{P}$ is a \emph{collider path} if every non-endpoint vertex on $\mathcal{P}$ is a collider on $\mathcal{P}$. A path $\mathcal{P}$ between the vertices $X$ and $Y$ is called an \emph{m-connecting} or \emph{active} path relative to a set $\mathbf{Z}\subseteq \mathbf{V}\setminusA\{X,Y\}$, if
    \begin{enumerate*}[label=(\roman*)]
        \item 
            every non-collider on $\mathcal{P}$ is not a member of $\mathbf{Z}$, and 
        \item
            every collider on $\mathcal{P}$ belongs to $\Anc{\{X,Y\}\cup\mathbf{Z}}$.
    \end{enumerate*}
     
    \begin{definition}[m-separation] \label{def: m-sep}
        Suppose $\G$ is a mixed graph. A set $\mathbf{Z}$ m-separates $X$ and $Y$ in $\G$, denoted by $\msep{X}{Y}{\mathbf{Z}}{\G}$, if there is no m-connecting path between $X$ and $Y$ relative to $\mathbf{Z}$ in $\mathcal{G}$\footnote{DAGs are a subclass of mixed graphs. Note that for DAGs, this definition reduces to d-separation. See \cite{pearl1988probabilistic} for the definition of d-separation.}. We call $\mathbf{Z}$ a separating set for $X$ and $Y$. We drop the subscript $\mathcal{G}$ whenever it is clear from context.
    \end{definition}
    A \emph{directed cycle} exists in a mixed graph if $X\to Y$ and $Y\in \Anc{X}$. An \emph{almost directed cycle} exists in a mixed graph when $X\leftrightarrow Y$ and $Y\in \Anc{X}$. A mixed graph is said to be \emph{ancestral}, if it does not contain directed cycles or almost-directed cycles, and for any undirected edge $X-Y$, $X$ and $Y$ have no parents or spouses. An ancestral graph is called \emph{maximal} if for any pair of non-adjacent vertices, there exists a set of vertices that m-separates them. A mixed graph is called a \emph{Maximal Ancestral Graph} (MAG) if it is ancestral and maximal. A MAG is called a \emph{directed acyclic graph} (DAG) if it has only directed edges. 
    
    A DAG $\mathcal{G}$ can be projected into a unique MAG over a subset of its vertices with the following projection, referred to as an embedded pattern in \cite{verma1991equivalence}.
    \begin{definition}[Latent projection] \label{def: DAG to MAG}
        Suppose $\mathcal{G}$ is a DAG over $\mathbf{V} = \mathbf{O} \cup \mathbf{L}\cup \mathbf{S}$. The projection of $\mathcal{G}$ over $\mathbf{O}$ conditioned on $\mathbf{S}$, denoted by $\mathcal{G}[\mathbf{O}\vert\mathbf{S}]$, is a MAG over vertices $\mathbf{O}$ constructed as follows:
        \begin{enumerate}[label=(\roman*)]
            \item 
                Skeleton: $X,Y\in \mathbf{O}$ are adjacent in $\mathcal{G}[\mathbf{O}\vert\mathbf{S}]$ if there exists an inducing path\footnote{An \emph{inducing path} between $X$ and $Y$ relative to $\langle\mathbf{L},\mathbf{S}\rangle$, where $\mathbf{L}$ and $\mathbf{S}$ are disjoint sets not containing $X$ and $Y$, is a path on which every non-collider is a member of $\mathbf{L}$ and every collider 
                belongs to $\Anc{\{X,Y\}\cup\mathbf{S}}$.} in $\mathcal{G}$ between $X$ and $Y$ relative to $\langle\mathbf{L},\mathbf{S}\rangle$.
            \item
                Orientation: For each pair of adjacent variables $X,Y$ in $\mathcal{G}[\mathbf{O}\vert\mathbf{S}]$, the edge between $X$ and $Y$ is oriented as $X\to Y$ if $X \in \Anc{\{Y\}\cup\mathbf{S}}$ and $Y \notin \Anc{\{X\}\cup\mathbf{S}}$; as $X \leftrightarrow Y$ if $X \notin \Anc{\{Y\}\cup\mathbf{S}}$ and $Y \notin \Anc{\{X\}\cup\mathbf{S}}$; and as $X-Y$ if $X \in \Anc{\{Y\}\cup\mathbf{S}}$ and $Y \in \Anc{\{X\}\cup\mathbf{S}}$.
        \end{enumerate}
    \end{definition}
    The above projection is the unique projection which satisfies the following property \cite{richardson2002ancestral}.
    \begin{equation} \label{eq: msep iff dsep}
        \msep{X}{Y}{\mathbf{Z}}{\mathcal{G}[\mathbf{O}\vert \mathbf{S}]} 
        \iff
        \msep{X}{Y}{\mathbf{Z} \cup \mathbf{S}}{\mathcal{G}}.
    \end{equation}
    Two MAGs are called \emph{Markov equivalent} if they impose the same m-separations. 
    A class of Markov equivalent MAGs can be represented as a (maximally informative) \emph{partially-oriented ancestral graph} (PAG), where the PAG contains the skeleton and all the invariant edge marks in the class.
    
    Let $P$ be the joint distribution over a set of variables $\mathbf{V}$. For $X, Y \in \mathbf{V}, \mathbf{Z}\subseteq \mathbf{V}\setminusA \{X,Y\}$, a conditional independence (CI) test in $P$ on the triplet $\langle X,\mathbf{Z},Y\rangle$ yields independence, denoted by $\CI{X}{Y}{\mathbf{Z}}{P}$, if $P(X\vert Y,\mathbf{Z})=P(X \vert \mathbf{Z})$. We drop the subscript $P$ when it is clear from context. Suppose $\G$ is a DAG over $\mathbf{V}$, i.e., each vertex of $\G$ corresponds to a variable of $\mathbf{V}$. We say $P$ is \emph{faithful} with respect to $\G$, if $\CI{X}{Y}{\mathbf{Z}}{P} \iff \msep{X}{Y}{\mathbf{Z}}{\G}$, i.e., the conditional independence in distribution $P$ is equivalent to m-separation in the DAG $\mathcal{G}$.
    \subsection{Problem description}
    We consider a system with the set of variables $\V = \mathbfcal{O} \cup \mathbfcal{L}\cup \mathbfcal{S}$ and the joint distribution $\PV$, where $\mathbfcal{O}$, $\mathbfcal{L}$, and $\mathbfcal{S}$ denote the set of observed, latent, and selection variables, respectively. 
    Each variable $X \in \V$ is generated as $X = f_X(\Pa{X},\epsilon_X)$, where $f_X$ is a deterministic function, $\Pa{X} \subseteq \V\setminusA\{X\}$ is the set of parents of $X$, i.e., the set of variables that have a direct causal effect on $X$, and $\epsilon_X$ is the exogenous noise corresponding to $X$. We assume all noise variables are jointly independent. 
    This model is referred to as structural equations model (SEM) \cite{pearl2009causality}. 
    The causal graph of the system, which represents the causal relations among the variables, is denoted by $\GV$. $\GV$ is a directed graph over $\V$, i.e., each vertex is associated with a variable\footnote{We will use vertex and variable interchangeably throughout the paper.}, and a directed edge exists from each variable in $\Pa{X}$ to $X$, for all $X\in \V$. We assume that $\GV$ is a DAG, and its latent projection over $\mathbfcal{O}$ conditioned on $\mathbfcal{S}$ is denoted by $\mathcal{M}:=\GV[\mathbfcal{O}\vert \mathbfcal{S}]$. We will call $\mathcal{M}$ the ground truth MAG. Further, we assume that $\PV$ is faithful with respect to $\GV$, which along with Equation \eqref{eq: msep iff dsep} implies that for each $X,Y\in \mathbfcal{O}$ and $\mathbf{Z}\subseteq \mathbfcal{O} \setminusA \{X,Y\}$,
    \begin{equation}
        \msep{X}{Y}{\mathbf{Z}}{\mathcal{M}}
        \iff 
        \CI{X}{Y}{\mathbf{Z}}{\POS}. 
    \end{equation}
    Given the observational data from $\POS$, i.e., the marginal distribution over the observed variables, conditioned on the selection variables, we consider the problem of learning the PAG that represents the Markov equivalence class (MEC) of $\mathcal{M}$.

\section{L-MARVEL Algorithm} \label{sec: L-MARVEL}
    \begin{algorithm}[t]
        \caption{L-MARVEL.}
        \label{alg: L-MARVEL}
        \begin{algorithmic}[1]
            \STATE {\bfseries Input:} $\mathbfcal{O},\, \POS$
            \STATE {\bfseries Output:} PAG $\hat{\mathcal{M}}$
            \STATE $\MB{\mathbfcal{O}} \gets \text{ComputeMb}(\mathbfcal{O},\, \POS)$
            \STATE $\A \gets \text{Initialization}(\mathbfcal{O},\, \MB{\mathbfcal{O}})$
            \STATE $\A \gets \text{L-MARVEL}(\mathbfcal{O},\, \POS,\, \MB{\mathbfcal{O}},\, \A)$
            \STATE \text{Create} $\hat{\mathcal{M}}$ \text{according to adjacencies in} $\mathcal{A}$ \text{and orient it maximally using rules 0-10 of \cite{zhang2008causal}}
        \end{algorithmic}
        \hrulefill
        \begin{algorithmic}[1]
            \STATE {\bfseries Function} L-MARVEL($\mathbf{V},\, P_{\mathbf{V}|\mathbfcal{S}},\, \MB{\mathbf{V}},\, \A$)
            \IF{$|\mathbf{V}|=1$}
                \RETURN $\A$
            \ELSE
                \STATE $(X_{1},X_{2},\dots X_{|\mathbf{V}|})\gets$ Sort $\mathbf{V}$ based on the Markov boundary size in ascending order.
                \FOR{$i=1$ to $|\mathbf{V}|$}
                    \STATE $(\Adj{X_i},\,\A) \gets \textbf{FindAdjacent}(X_i,\, \Mb{X_i}{\mathbf{V}},\, P_{\mathbf{V}|\mathbfcal{S}},\, \A)$ 
                    \STATE $\textit{isR}\gets \textbf{IsRemovable}(X_i,\, \Mb{X_i}{\mathbf{V}},\, P_{\mathbf{V}|\mathbfcal{S}},\, \Adj{X_i})$ \hfill \% Main step of the algorithm. 
                    \IF{\textit{isR} is true}
                        \STATE $(\MB{\mathbf{V}\setminusA X_i},\, \A) \gets \textbf{UpdateMb}(X_i,\, \Adj{X_i},\, \MB{\mathbf{V}},\, P_{\mathbf{V}|\mathbfcal{S}},\, \A)$
                        \RETURN $\text{L-MARVEL}(\mathbf{V}\setminusA \{X_i\},\, P_{\mathbf{V}\setminusA \{X_i\}|\mathbfcal{S}},\, \MB{\mathbf{V}\setminusA \{X_i\}},\, \A)$
                    \ENDIF
                \ENDFOR
            \ENDIF
        \end{algorithmic}
    \end{algorithm}
    %In this section, we present L-MARVEL which extends the approach
    %\cite{mokhtarian2020recursive} proposed MARVEL algorithm for causal structure learning when all the variables are observable. 
    In this section, we present \emph{Latent MARVEL} (L-MARVEL) algorithm to learn the PAG over $\mathbfcal{O}$ that represents the system.
    This algorithm relies on a notion similar to the MARVEL algorithm proposed by \cite{mokhtarian2020recursive} for DAG learning when all the variables are observable.
    % This is the [] version of MARVEL algorithm in \cite{mokhtarian2020recursive} which requires all the variables to be observable. 
    Our approach relies on the Markov boundary information as input.
    \begin{definition}[Markov boundary]
        Suppose $\mathbf{V} \subseteq \mathbfcal{O}$. Markov boundary of $X\in \mathbf{V}$ with respect to $\mathbf{V}$ is a minimal set of variables $\mathbf{Z}\subseteq\mathbf{V}\setminusA\{X\}$, such that $X$ is independent of the rest of the variables of $\mathbf{V}$ conditioned on $\mathbf{Z}\cup\mathbfcal{S}$.
    \end{definition}
    Under faithfulness, Markov boundary of $X\in \mathbf{V} \subseteq \mathbfcal{O}$ with respect to $\mathbf{V}$, denoted by $\Mb{X}{\mathbf{V}}$, is unique and it consists of all the variables that have a collider path to $X$ in $\GV[\mathbf{V}|\mathbfcal{S}]$ \cite{yu2018mining, pellet2008finding}. 
    We indicate by $\MB{\mathbf{V}}$, the Markov boundaries of all of the variables in $\mathbf{V}$ with respect to $\mathbf{V}$.
    
    Our learning procedure is outlined in Algorithm \ref{alg: L-MARVEL}. 
    Throughout the algorithm, the data structure $\mathcal{A}$ stores the pairs of vertices that have been identified to be adjacent, and the separating sets found for non-adjacent vertices so far.
    As the first step, the Markov boundary information with respect to $\mathbfcal{O}$ is identified using one of the standard methods in the literature, as discussed in Section \ref{sec: Mb}.
    Then $\mathcal{A}$ is initialized with separating sets implied by the Markov boundary information, i.e., for any $X$ and $Y\notin\Mb{X}{\mathbfcal{O}}$, $\Mb{X}{\mathbfcal{O}}$ is a separating set for $X,Y$. 
    $\mathcal{A}$ is updated when a new separating set is discovered for a pair of vertices, or two vertices are determined to be adjacent. 
    After initializing $\mathcal{A}$ in line 4, we call the L-MARVEL function over $\mathbfcal{O}$, which recursively identifies all the adjacent pairs of vertices, i.e., the skeleton of $\mathcal{M}$, and discovers a separating set for all non-adjacent pairs of vertices. 
    This information suffices to maximally orient the edge marks at the end of the algorithm using the complete set of orientation rules in \cite{zhang2008causal}.
    
    L-MARVEL works as follows. 
    It chooses a variable $X$, identifies $\Adj{X}$ (i.e., the set of variables adjacent to $X$), and then recursively learns the structure over $\mathbfcal{O}\setminusA\{X\}$, discarding $X$. 
    This is desirable as the problem size decreases at each iteration, which results in a substantial reduction in the computational complexity. 
    Moreover, performing CI tests of high order is avoided. 
    If the learned graph, i.e., $\GV[\mathbfcal{O}\setminusA\{X\}\vert\mathbfcal{S}]$ is equal to the induced subgraph of $\GV[\mathbfcal{O}\vert\mathbfcal{S}]$ over $\mathbfcal{O}\setminusA\{X\}$, we can add $X$ to this graph and connect it to its adjacent variables with an edge. 
    As discussed in the example in Figure \ref{fig: rem}, this is not true for an arbitrary $X$.
    We show that for certain vertices, called \emph{removable}, we can indeed apply such a recursive learning procedure. 
    Next, we define what makes a variable removable.
    
    \begin{definition} [Removable] \label{def: removable}
        Suppose $\mathcal{G}$ is a MAG over $\mathbf{V}$, $X\in \mathbf{V}$, and $\mathcal{H}$ is the induced subgraph of $\mathcal{G}$ over $\mathbf{V}\setminusA\{X\}$. $X$ is a removable vertex in $\mathcal{G}$ if $\mathcal{G}$ and $\mathcal{H}$ impose the same m-separation relations over $\mathbf{V}\setminusA\{X\}$. That is, for any vertices $Y,W\in \mathbf{V}\setminusA\{X\}$ and $\mathbf{Z}\subseteq\mathbf{V}\setminusA\{X,Y,W\}$,
        \begin{equation}\label{eq: d-sepEquivalence}
            \msep{Y}{W}{\mathbf{Z}}{\mathcal{G}}
		    \iff
		    \msep{Y}{W}{\mathbf{Z}}{\mathcal{H}}.
		\end{equation}
    \end{definition}
    In the case that $\mathcal{G}$ is a DAG, Definition \ref{def: removable} reduces to what \cite{mokhtarian2020recursive} proposed for DAGs. However, their tests for identifying removability fail when causal sufficiency is violated.
    Next, we provide a graphical characterization of removable variables. 
    \begin{theorem}\label{thm: graph-rep}
        Vertex $X$ is removable in a MAG $\mathcal{M}$ over the variables $\mathbf{V}$, if and only if 
        \begin{enumerate}
            \item for any $Y\in \Adj{X}$ and $Z\in \Ch{X}\cup \N{X}\setminus \{Y\}$, $Y$ and $Z$ are adjacent, and 
            \item for any collider path $u=(X,V_1,...,V_m,Y)$ and $Z\in\mathbf{V}\setminusA\{X,Y,V_1,...,V_m\}$ such that $\{X,V_1,...,V_m\}\subseteq\Pa{Z}$, $Y$ and $Z$ are adjacent.
        \end{enumerate}
    \end{theorem}
    Figure \ref{fig: graph-rep} represents the graphical constraints of Theorem \ref{thm: graph-rep}. 
    Figure \ref{fig: graph-rep not rem} depicts the first condition, where $Z$ is either a child or a neighbor of $X$, and $Y\in\Adj{X}$, while Figure \ref{fig: graph-rep rem} depicts a collider path where
    $X$ and $V_i$s are parents of $Z$. 
    Theorem \ref{thm: graph-rep} states that $X$ is removable if and only if the edges highlighted in red are present in both cases.
    See Appendix \ref{sec: apd graphical of removable} for a formal proof and further discussion on Theorem \ref{thm: graph-rep}.
    The next proposition clarifies why removable variables are exactly those that can be removed at each iteration in our recursive approach.
    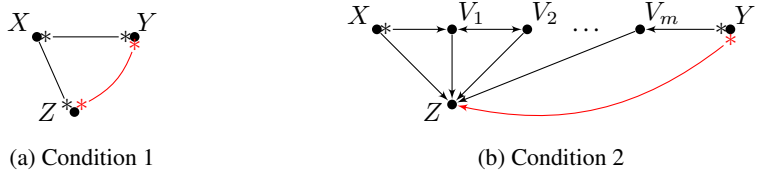
\begin{figure}[t] 
        \centering
    	\tikzstyle{block} = [circle, inner sep=1.3pt, fill=black]
    	\tikzstyle{input} = [coordinate]
    	\tikzstyle{output} = [coordinate]
    	\begin{subfigure}[b]{0.44\textwidth}
    	    \centering
    	    \begin{tikzpicture}
                \tikzset{edge/.style = {->,> = latex'}}
                % vertices
                \node[block] (x) at  (0,1) {};
                \node[] ()[above left=-0.1cm and -0.1cm of x]{$X$};
                \node[block] (z) at  (0.5,0) {};
                \node[block] (y) at  (1.3,1) {};
                \node[style={minimum size=0.1cm},inner sep=0.01pt](starxtz) [above left=-0.05cm and -0.05cm of z]{$*$};
                \node[style={minimum size=0.1cm},inner sep=0.01pt](starxty) [left= -0.05cm of y]{$*$};
                \node[style={minimum size=0.1cm},inner sep=0.01pt](starytx) [right= -0.05cm of x]{$*$};
                \node[color=red,style={minimum size=0.1cm},inner sep=0.01pt](starytz) [above right=-0.05cm and -0.05cm of z]{$*$};
                \node[color=red,style={minimum size=0.1cm},inner sep=0.01pt](starzty) [below =-0.05cm  of y]{$*$};
                \node[] ()[above right=-0.1cm and -0.15cm of y]{$Y$};

                \node[] ()[above left=-0.3cm and 0.05cm of z]{$Z$};
                %edges
                \draw[edge, color=red, style={-}, bend right = 20] (starytz) to (starzty);
                \draw[edge, style={-}] (starxtz) to (x);
                \draw[edge, style={-}](starxty) to (starytx);
            \end{tikzpicture}
        \caption{Condition 1}
        \label{fig: graph-rep not rem}
    	\end{subfigure}
        \begin{subfigure}[b]{0.44\textwidth}
    	    \centering
    	    \begin{tikzpicture}
                \tikzset{edge/.style = {->,> = latex'}}
                % vertices
                \node[block] (x) at  (0,1) {};
                \node[] ()[above left=-0.1cm and -0.1cm of x]{$X$};
                \node[block] (v1) at  (1,1) {};
                \node[] ()[above right=-0.1cm and -0.1cm of v1]{$V_1$};
                \node[style={minimum size=0.1cm},inner sep=0.01pt](starv1tx) [right= -0.05cm of x]{$*$};
                \node[block] (v2) at  (2,1) {};
                \node[] ()[above right=-0.1cm and -0.12cm of v2]{$V_2$};
                
                \node[right=0.4cm of v2]{\ldots};
                
                \node[block] (vm) at  (3.5,1) {};
                \node[] ()[above right=-0.1cm and -0.12cm of vm]{$V_m$};
                
                \node[block] (y) at  (4.7,1) {};
                \node[] ()[above right=-0.1cm and -0.12cm of y]{$Y$};
                \node[style={minimum size=0.1cm},inner sep=0.01pt](starvmty) [left= -0.05cm of y]{$*$};
                \node[block] (z) at  (1,0) {};
                \node[] ()[above left=-0.4cm and -0.05cm of z]{$Z$};
                \node[color=red,style={minimum size=0.1cm},inner sep=0.01pt](star) [below left=0cm and -0.15cm of y]{$*$};
                %edges
                \draw[edge] (starv1tx) to (v1);
                \draw[edge, style={<->}] (v1) to (v2);
                \draw[edge] (starvmty) to (vm);
                \draw[edge] (x) to (z);
                \draw[edge] (v1) to (z);
                \draw[edge] (v2) to (z);
                \draw[edge] (vm) to (z);
                \draw[edge, color=red, style={->}, bend left = 25] (star) to (z);
            \end{tikzpicture}
        \caption{Condition 2}
        \label{fig: graph-rep rem}
    	\end{subfigure}
        \caption{Graphical characterization of a removable variable. The edge marks indicated by a star ($*$) can be either a tail or an arrowhead.}
        \label{fig: graph-rep} 
    \end{figure}
    \begin{proposition} \label{prp: removable-subgraph}
        Suppose $\mathbf{V}\subseteq \mathbfcal{O}$ and $X\in \mathbf{V}$. $\GV[\mathbf{V}\setminusA \{X\} \vert \mathbfcal{S}]$ is equal to the induced subgraph of $\GV[\mathbf{V} \vert \mathbfcal{S}]$ over $\mathbf{V}\setminusA \{X\}$ if and only if $X$ is removable in $\GV[\mathbf{V} \vert \mathbfcal{S}]$.
    \end{proposition}
    Appendix \ref{sec: apd proof} includes the proofs of our results.
    Identifying a removable variable at each iteration is the core of L-MARVEL.
    We will discuss an efficient algorithm to determine whether a variable is removable in Section \ref{sec: test removability}.
    At each iteration, given the set of remaining variables $\mathbf{V}$, these variables are sorted in ascending order of their Markov boundary size. 
    Starting with the variable with the smallest Markov boundary, we search for its adjacent vertices within its Markov boundary. 
    If $Y\in\Mb{X_i}{\mathbf{V}}$ is not adjacent to $X_i$, then $X_i$ and $Y$ have a separating set in $\Mb{X_i}{\mathbf{V}}\setminusA\{Y\}$ \cite{pellet2008finding}.
    Hence, identifying $\Adj{X_i}$ can be performed using a brute-force search in the Markov boundary, using at most $\vert\Mb{X_i}{\mathbf{V}}\vert2^{(\vert\Mb{X_i}{\mathbf{V}}\vert-1)}$ CI tests. 
    In Section \ref{sec: complexity}, we show that the loop in line 6 of Algorithm \ref{alg: L-MARVEL} never reaches variables with large Markov boundaries, and this guarantees that both the number of CI tests and their order remains small throughout the algorithm. 
    We then determine whether $X_i$ is removable given $\Mb{X_i}{\mathbf{V}}$ and $\Adj{X_i}$, using the efficient algorithm we shall discuss in Section \ref{sec: test removability}.
    We continue this procedure until we identify the first removable variable $X=X_i$. 
    Then, we remove $X$ from the set of remaining variables, and update the Markov boundaries with respect to $\mathbf{V}\setminusA\{X\}$, which is the input to the next iteration. 
    The latter does not require the discovery of Markov boundaries from scratch, and is implemented as we shall see in Section \ref{sec: Mb}.
    
    The rest of the section is dedicated to showing how to efficiently identify a \emph{removable} variable (Section \ref{sec: test removability}), how to update the Markov boundary information (Section \ref{sec: Mb}), and the analysis of the algorithm (Section \ref{sec: complexity}).
    \subsection{Testing removability in MAGs}
        % The following theorem presents the conditions of removability of a variable using CI tests within the Markov boundary.
        % This theorem excludes a very specific structure of the MAG $\mathcal{M}$, where $\mathcal{M}$ has a cycle of length at least 4, that contains only undirected edges, and furthermore, this cycle has no chords. We shall discuss in Appendix \ref{sec: apd chordal}, which specific structure of the DAG $\GV$ this MAG represents, and why it is required to exclude this certain structure. As we shall see in Appendix \ref{sec: apd chordal}, such MAGs imply very restrictive structure over the selection variables. 
        The following theorem presents the conditions of removability of a variable using CI tests within the Markov boundary.
        This theorem excludes a particular structure of the MAG $\mathcal{M}$, where $\mathcal{M}$ has a cycle of the length of at least four that contains only undirected edges, and this cycle has no chords. We shall discuss in Appendix \ref{sec: apd chordal}, which specific structure of the DAG $\GV$ this MAG represents, and why it is required to exclude this specific structure. As we shall see in Appendix \ref{sec: apd chordal}, such MAGs imply a very restrictive structure over the selection variables. 
        
        \begin{theorem} \label{thm: test removability}
            Suppose the edge-induced subgraph of $\mathcal{M}$ over the undirected edges (i.e., the edges due to selection bias) is chordal. 
            Let $\mathcal{G} = \GV[\mathbf{V}|\mathbfcal{S}]$ for some $\mathbf{V}\subseteq \mathbfcal{O}$.
            $X\in\mathbf{V}$ is removable in $\mathcal{G}$ if and only if for every $Y\in \Adj{X}$ and $Z\in \Mb{X}{\mathbf{V}}$, at least one of the following holds.
            \begin{enumerate}
                \item[] \textbf{Condition 1:\:}
                    $\exists \mathbf{W}\subseteq \Mb{X}{\mathbf{V}} \setminusA \{Y,Z\}\!:\: Y\independent Z\vert{\mathbf{W}}.$
                \item[] \textbf{Condition 2:\:}
                    $\forall \mathbf{W}\subseteq \Mb{X}{\mathbf{V}} \setminusA \{Y,Z\}\!:\: Y\notindependent Z\vert{\mathbf{W}\cup \{X\}}.$
            \end{enumerate}
            Furthermore, the set of removable vertices in $\mathcal{G}$ is non-empty.
        \end{theorem}
        Using Theorem \ref{thm: test removability} and 
        given $\Adj{X}$ and $\Mb{X}{}$, Algorithm \ref{alg: removability} tests the removability of $X$ by performing 
        $\mathcal{O}(\left\vert\Adj{X}\right\vert\left\vert\Mb{X}{\mathbf{V}}\right\vert2^{\left\vert\Mb{X}{\mathbf{V}}\right\vert})$ CI tests. 
        Note that the removability test is only performed for variables with small Markov boundaries, which keeps both the number of CI tests and the size of the conditioning sets small, as we shall see in Section \ref{sec: complexity}. 
        
        \begin{algorithm}[ht]
               \caption{IsRemovable - Determine whether $X$ is removable.}
               \label{alg: removability}
            \begin{algorithmic}[1]
                \STATE {\bfseries Input:} ($X,\, \Mb{X}{\mathbf{V}},\, P_{\mathbf{V}|\mathbfcal{S}},\, \Adj{X}$)
                \FOR{$Y\in \Adj{X}, Z\in \Mb{X}{\mathbf{V}}$}
    		        \IF{Condition 1 of Theorem \ref{thm: test removability} does not holds}
    		            \IF{Condition 2 of Theorem \ref{thm: test removability} does not holds}
    		                \STATE \textbf{Return} False \hfill \% $X$ is not removable.
    		            \ENDIF
                    \ENDIF   
        		\ENDFOR
    		    \STATE \textbf{Return} True \hfill \% $X$ is removable.
            \end{algorithmic}
        \end{algorithm} \label{sec: test removability}
        Conditions of Theorem \ref{thm: test removability} can be checked in different orders, although we have witnessed in our experiments that checking these conditions in the order of Algorithm \ref{alg: removability} increases the accuracy.
    \subsection{Markov boundary discovery and updating Markov
    boundaries} \label{sec: Mb}
        L-MARVEL requires Markov boundary information for initialization.
        Several algorithms have been proposed in the literature for discovering the Markov boundaries\cite{margaritis1999bayesian, pellet2008using, tsamardinos2003algorithms, yu2018mining}. 
        For instance, TC \cite{pellet2008using} algorithm states that
        \[ 
            \notCI{X}{Y}{\mathbf{V}\setminusA\{X,Y\}}{P_{\mathbf{V}|\mathbfcal{S}}}
            \iff 
            X\in\Mb{Y}{\mathbf{V}} \text{ and } Y\in\Mb{X}{\mathbf{V}},
        \]
        where $\mathbf{V}\subseteq \mathbfcal{O}$.
        Grow-Shrink (GS) algorithm \cite{margaritis1999bayesian} and its modifications, including IAMB and its variants \cite{tsamardinos2003algorithms} address Markov boundary discovery by performing more CI tests with smaller conditioning sets. 
        These algorithms require a linear number of CI tests in the number of variables to determine the Markov boundary of a certain variable, i.e., quadratic number of CI tests to discover the entire Markov boundaries. However, given the challenging nature of Markov boundary discovery, these algorithms might fail to accurately discover this information in some settings. 
        We need to utilize one of these methods\footnote{In our experiments, we used TC.} to initially discover the Markov boundaries, but the subsequent update of the boundaries throughout the later iterations is performed within L-MARVEL as we shall next discuss.
        
        \textbf{Updating Markov boundaries:} 
        Let $\MB{\mathbf{V}}$ be the input to an iteration of L-MARVEL where $X$ is identified as removable and we need to learn $\MB{\mathbf{V}\setminusA\{X\}}$. 
        By definition of removability, the latent projection of $\GV$ over $\mathbf{V}\setminusA\{X\}$ is the induced subgraph of $\GV[\mathbf{V}\vert\mathbfcal{S}]$. 
        As a result, removing $X$ has two effects: 
        \begin{enumerate*}
            \item $X$ is removed from all Markov boundaries, and
            \item for $Y,Z\in\mathbf{V}\setminusA \{X\}$, if all of the collider paths between $Y$ and $Z$ in $\GV[\mathbf{V}\vert\mathbfcal{S}]$ pass through $X$, then $Y$ and $Z$ must be excluded from each others Markov boundary.
        \end{enumerate*} 
        Note that in the latter case, $Y,Z\in\Mb{X}{\mathbf{V}}$. 
        The latter update is performed using a single CI test, i.e., $ \notCI{Y}{Z}{\Mb{Z}{\mathbf{V}} \setminusA \{X,Y,Z\}}{}$, or equivalently, $\notCI{Y}{Z}{\Mb{Y}{\mathbf{V}} \setminusA \{X,Y,Z\}}{}$. 
        We choose the CI test with the smaller conditioning set among the two. If the dependency does not hold, we remove $Y,Z$ from each other's Markov boundary.
    \subsection{Analysis}\label{sec: complexity}
        First, we state the soundness and completeness of L-MARVEL in the following theorem. 
        \begin{theorem}\label{thm: sound and complete}
            Suppose the distribution $\PV$ over $\V = \mathbfcal{O}\cup\mathbfcal{L}\cup\mathbfcal{S}$ is faithful to the DAG $\GV$. If the conditional independence relations among all variables in $\mathbfcal{O}$ given $\mathbfcal{S}$ is provided to L-MARVEL, the output of L-MARVEL is the PAG representing the Markov equivalence class of $\GV[\mathbfcal{O}\vert\mathbfcal{S}]$.
        \end{theorem}
        Let $\deltaplus{\mathcal{H}}$ denote the maximum size of $\PaP{\cdot}$ (defined in \eqref{eq: parentplus}) in a MAG $\mathcal{H}$, i.e.,
        \begin{equation} \label{eq: def deltaplus}
            \deltaplus{\mathcal{H}}=\max_{X\in\mathcal{H}}{\left\vert \PaP{X}\right\vert}.
        \end{equation}
        Next, we provide an upper bound on the size of the Markov boundary of a removable variable.
        \begin{proposition}\label{prp: mb bound}
            If $X$ is a removable variable in MAG $\mathcal{H}$ with vertices $\mathbf{V}$, then $\vert\Mb{X}{\mathbf{V}}\vert\leq\deltaplus{\mathcal{H}}$.
        \end{proposition}
        L-MARVEL processes variables in the ascending order of their Markov boundary size at each iteration and stops when the first removable variable is identified. 
        Therefore, Proposition \ref{prp: mb bound} guarantees that all the processed vertices at each iteration have Markov boundaries smaller than $\deltaplus{\GV[\mathbf{V}\vert\mathbfcal{S}]}$, where $\mathbf{V}$ is the set of remaining variables. 
        This number gets smaller during the algorithm, as L-MARVEL keeps only a subgraph of the input. 
        Note that this bound applies to the size of the conditioning sets of CI tests performed in functions \textbf{FindAdjacent} and \textbf{IsRemovable}, since the conditioning sets are a subset of the Markov boundary. 
        Furthermore, it results in the following upper bound on the number of CI tests.
        \begin{proposition}\label{prp: upper-bound}
            The number of conditional independence tests Algorithm \ref{alg: L-MARVEL} performs on a MAG $\mathcal{M}$ of order $n$, in the worst case, is upper bounded by
            \begin{equation}\label{eq: upper bound}
                \mathcal{O}(n^2 + n{\deltaplus{\mathcal{M}}}^2 2^{\deltaplus{\mathcal{M}}}).
            \end{equation}
        \end{proposition}
        The quadratic term in the upper bound of Equation \eqref{eq: upper bound} is for initial Markov boundary discovery.
        Note that algorithms such as GS, TC, IAMB, etc. discover the Markov boundary of each variable requiring only linear number of CI tests in $n$. 
        
        To the best of our knowledge, this is the tightest bound in the literature.
        The following lower bound on all of the constraint based structure learning algorithm demonstrates the efficiency of L-MARVEL.
        \begin{theorem} \label{thm: lwrBound}
    		The number of conditional independence tests of the form $\CI{X}{Y}{\mathbf{Z}}{}$ required by any constraint-based algorithm on a MAG ${\mathcal{M}}$ of order $n$, in the worst case, is lower bounded by
    		\begin{equation} \label{eq: lwrbound}
    		    \Omega(n^2+n{\deltaplus{\mathcal{M}}}2^{\deltaplus{\mathcal{M}}}).
    		\end{equation}
    	\end{theorem}
    	Comparing this lower bound with our achievable upper bound, we can see that the complexity of L-MARVEL in the worst case is merely different by a factor which is at most the number of observed variables in the worst case.
\section{Experiments} \label{sec: experiment}
    We report empirical results on both synthetic (random graphs) and real-world structures available in the Bayesian network repository\footnote{bnlearn.com/bnrepository/}, the benchmark for structure learning in the literature. 
    We evaluate and compare L-MARVEL\footnote{ The implementation of L-MARVEL is available in github.com/Ehsan-Mokhtarian/L-MARVEL.} to various algorithms, namely the constraint-based methods FCI \cite{spirtes2000causation}, RFCI \cite{colombo2012learning}, and MBCS* \cite{pellet2008finding}, and the hybrid method M3HC \cite{tsirlis2018scoring} in terms of both computational complexity and accuracy.
    Following the convention in \cite{colombo2012learning, tsirlis2018scoring, pellet2008finding, chobtham2020bayesian}, the data is generated according to a linear SEM with additive Gaussian noise, where all the variables of the system (including latent and selection) are generated as linear combinations of their parents plus a Gaussian noise.
    For each system, we simulate data from $\POS$, the data available to all the algorithms.
    We use TC \cite{pellet2008using} algorithm to learn the initial Markov boundaries. 
    To make a fair comparison among the algorithms, we feed the Markov boundary information to all the algorithms, that is, algorithms start from a graph where the edges between vertices that are not in each other's Markov boundary are already deleted. 
    This is similar to the ideas in \cite{pellet2008using}. 
    For CI tests, we use Fisher Z-transformation \cite{fisher1915frequency} with significance level $\alpha=0.01$ for all the algorithms, and $\alpha=2/n^2$ for TC \cite{pellet2008using}.
    In all the experiments, each point on the plots and each entry of the table represents an average of 50 datasets, where the latent and selection variables were chosen uniformly at random.
    \begin{figure*}[t] 
        \centering
        \begin{subfigure}[b]{1\textwidth}
            \centering
            \includegraphics[width=0.23\textwidth]{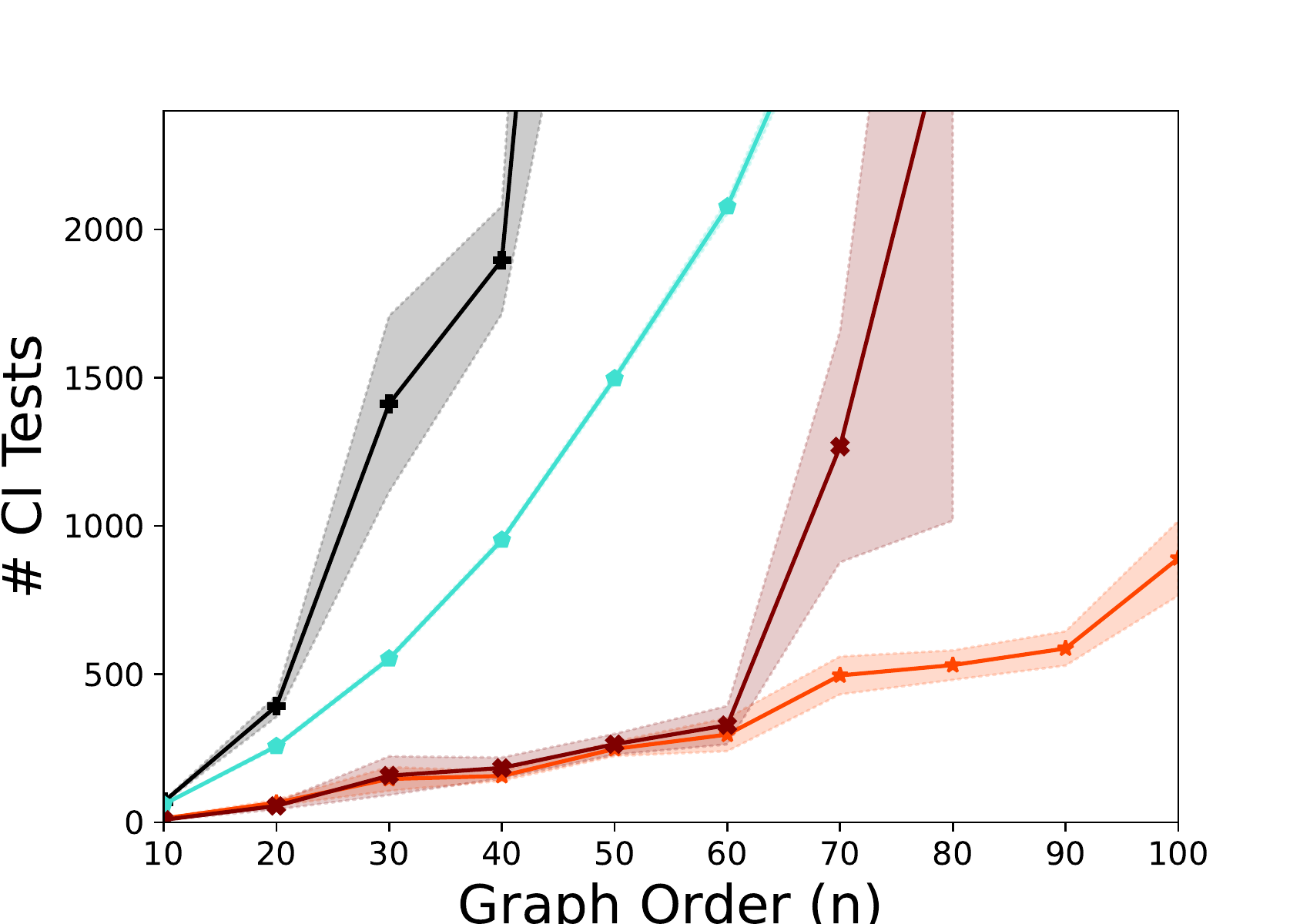}
            \hfill
            \includegraphics[width=0.23\textwidth]{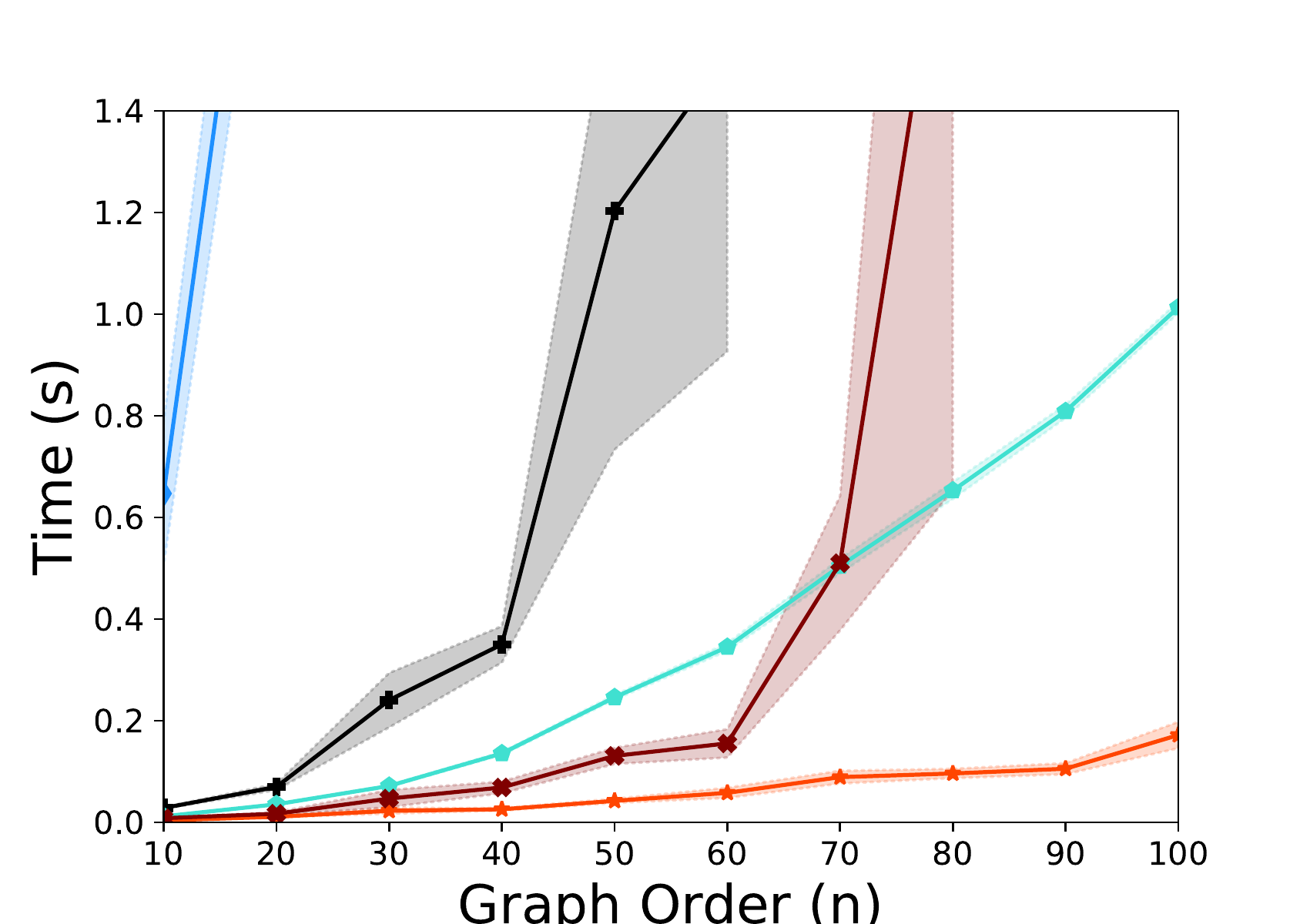}
            \hfill
            \includegraphics[width=0.23\textwidth]{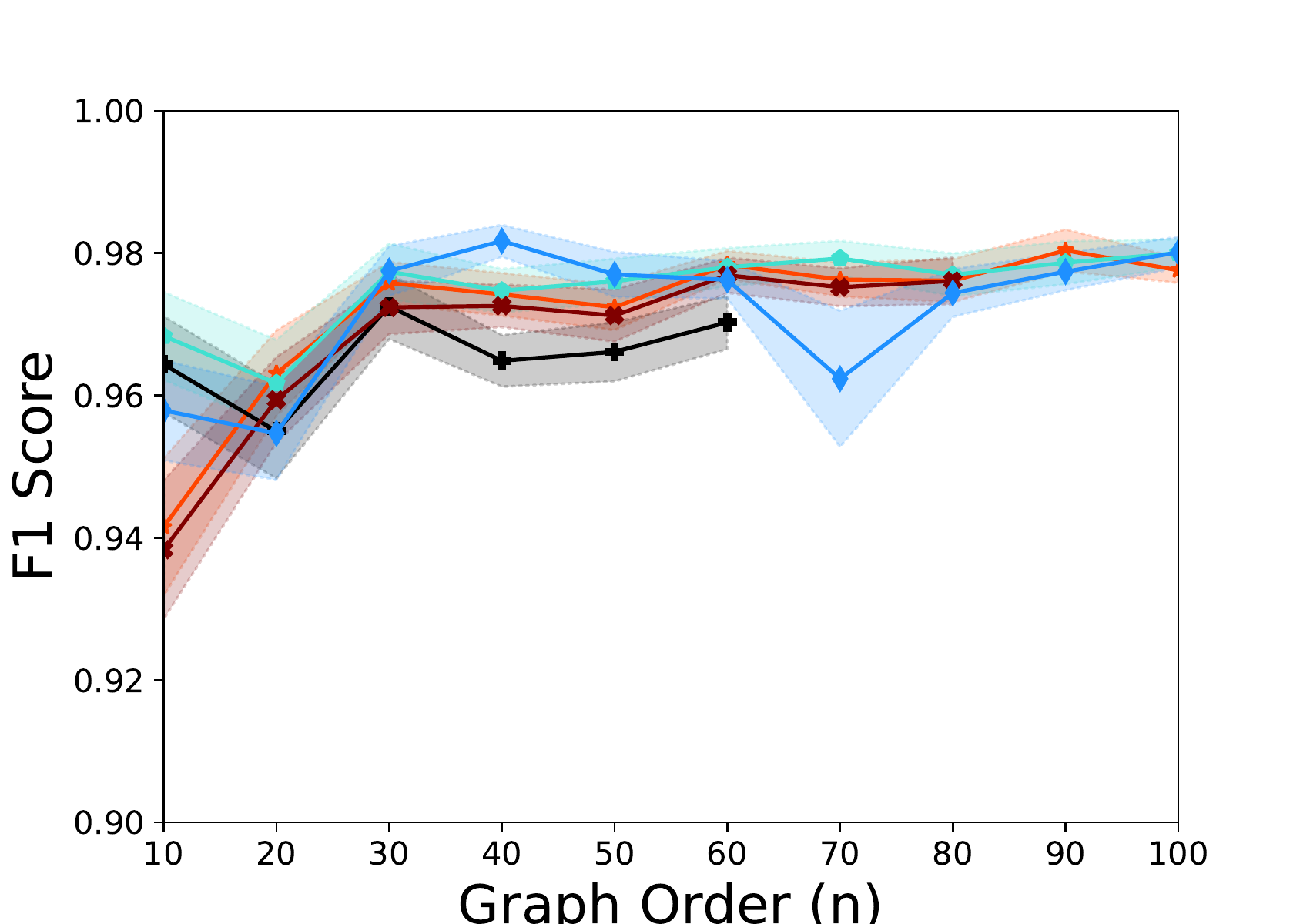}\hfill
            \includegraphics[width=0.23\textwidth]{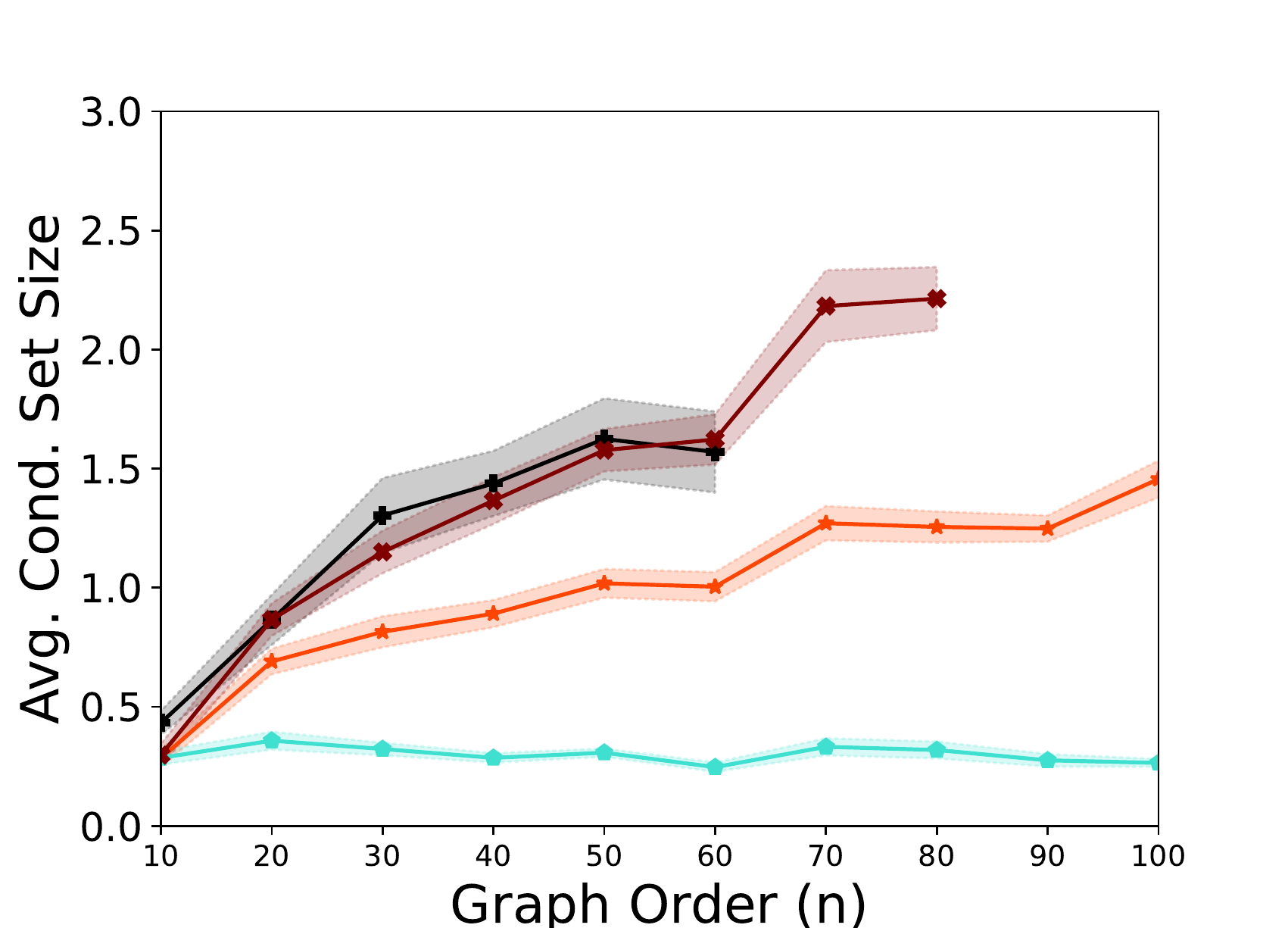}
            \caption{MAGs corresponding to $G(\Tilde{n},\frac{1}{\Tilde{n}^{0.9}})$ and latent rate = $10\%$.}
            \label{fig: ER 0.9 L=0.1}
        \end{subfigure}
        \hfill
        \begin{subfigure}[b]{1\textwidth}
            \centering
            \includegraphics[width=0.23\textwidth]{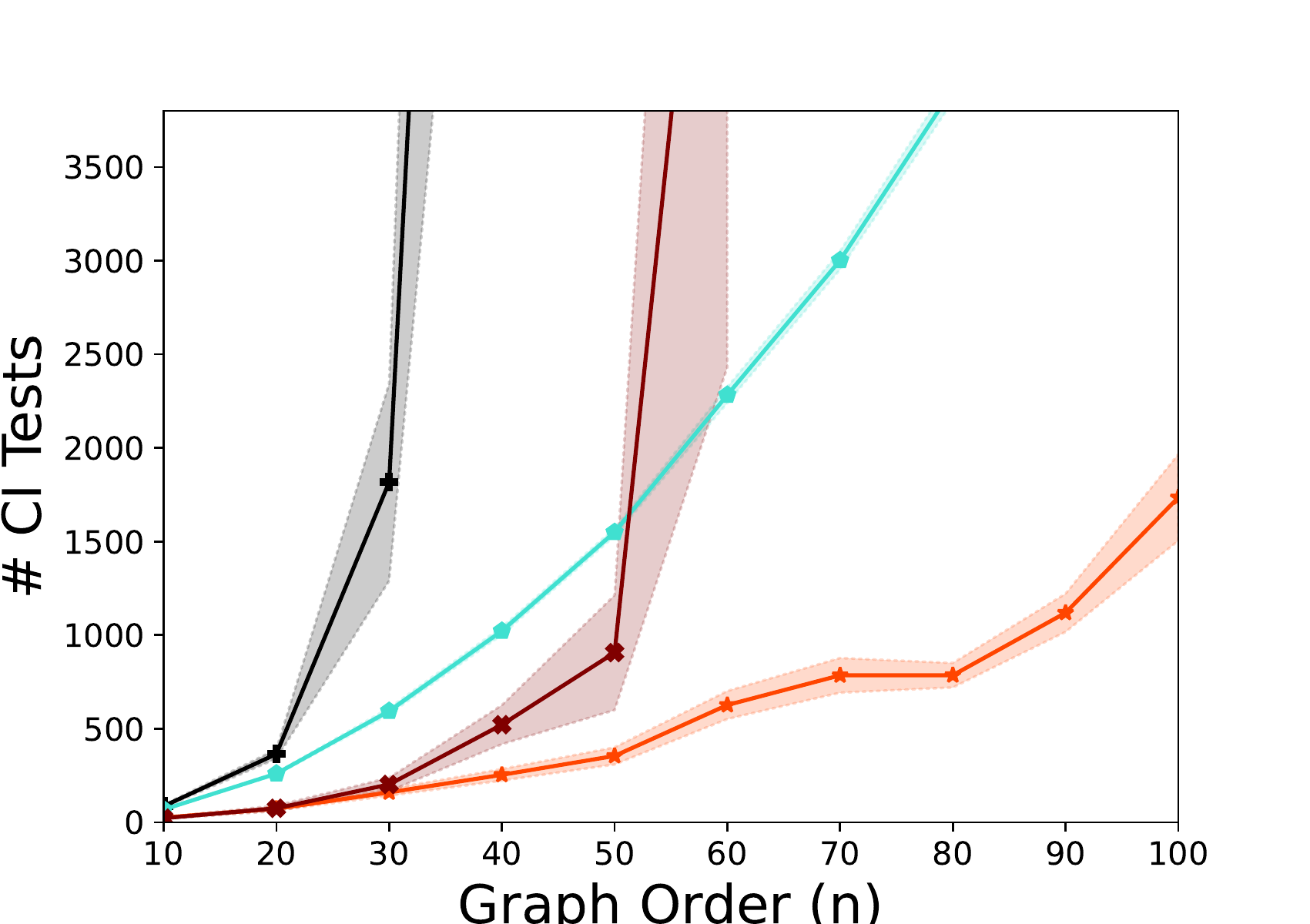}
            \hfill
            \includegraphics[width=0.23\textwidth]{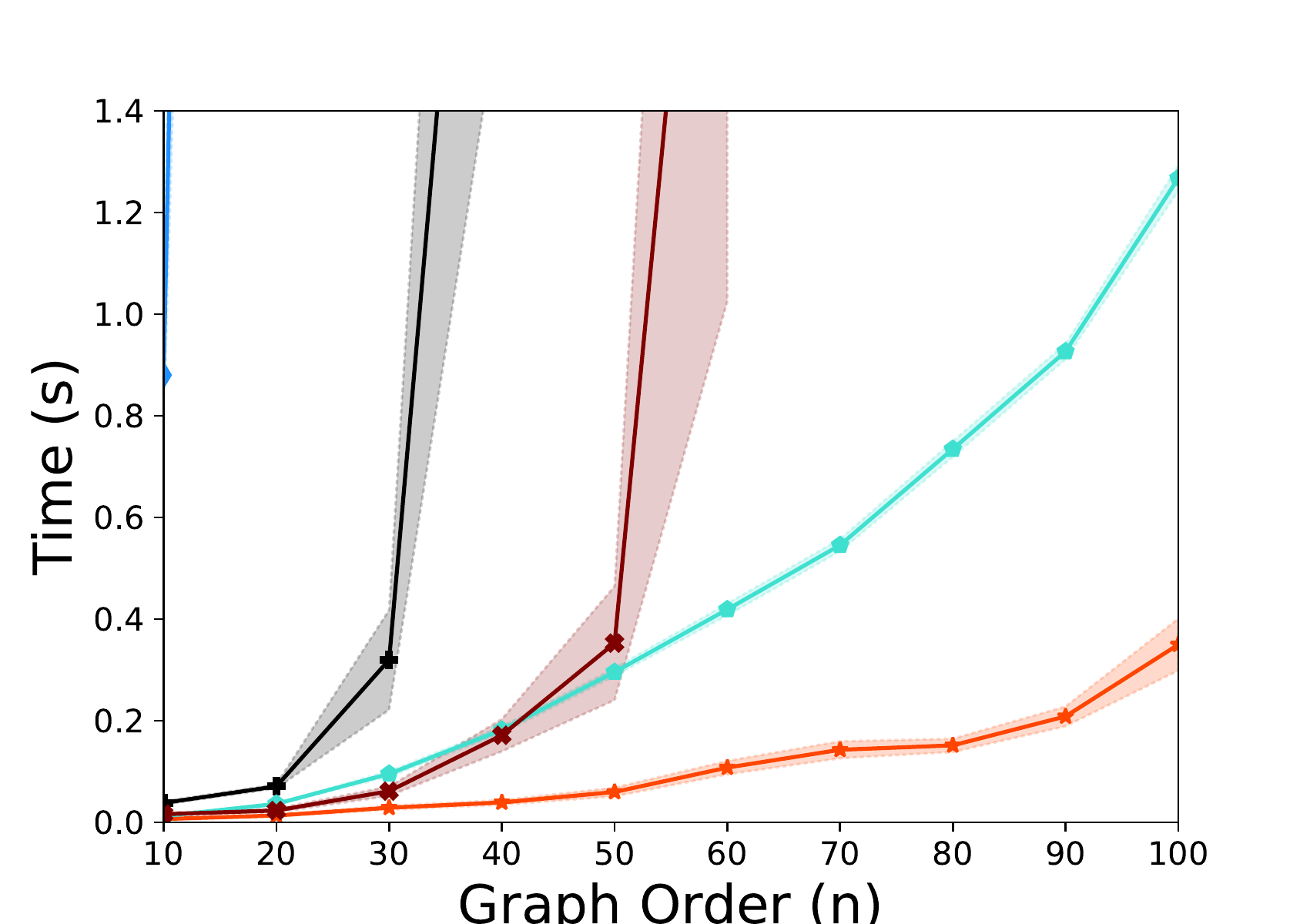}
            \hfill
            \includegraphics[width=0.23\textwidth]{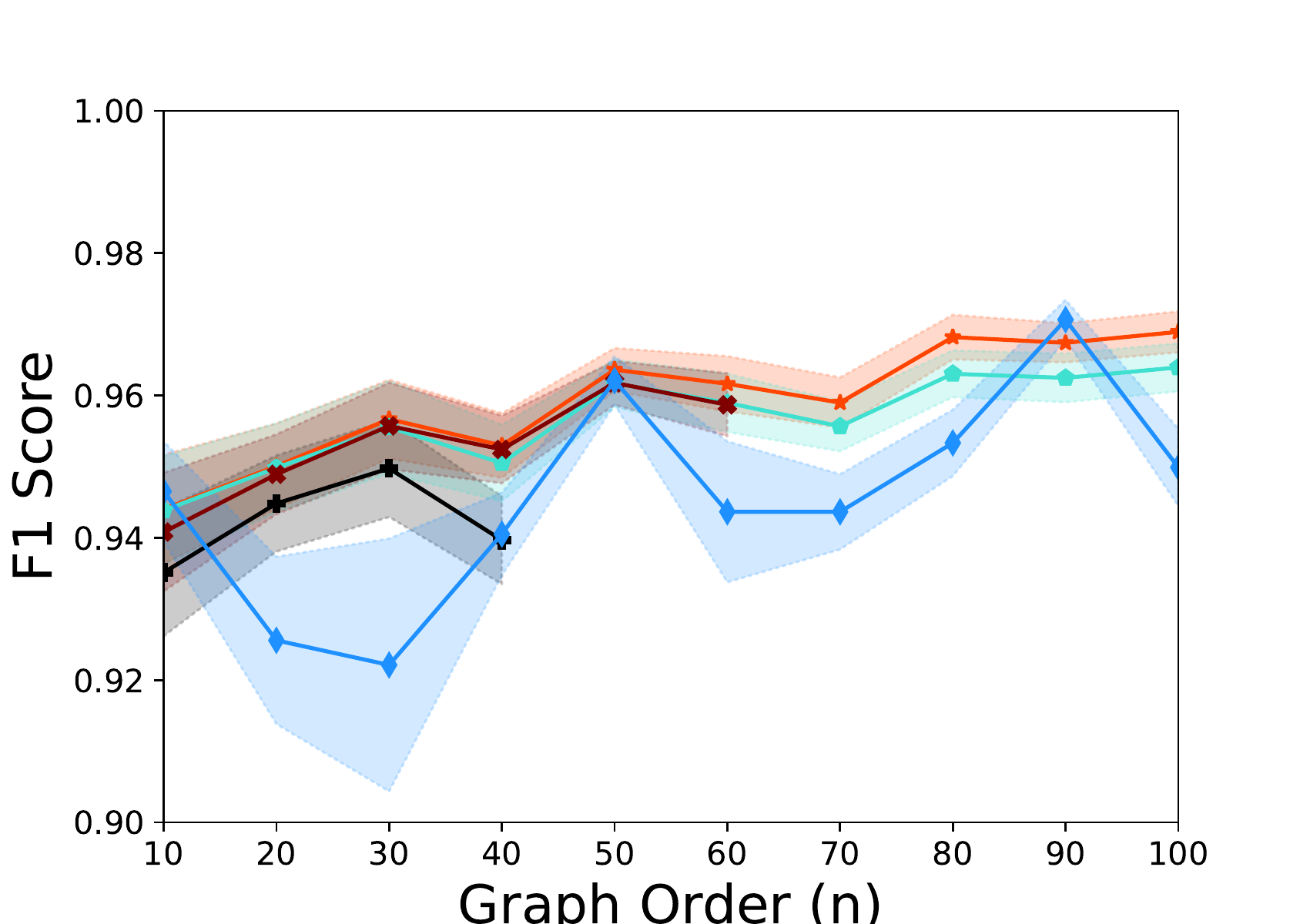}\hfill
            \includegraphics[width=0.23\textwidth]{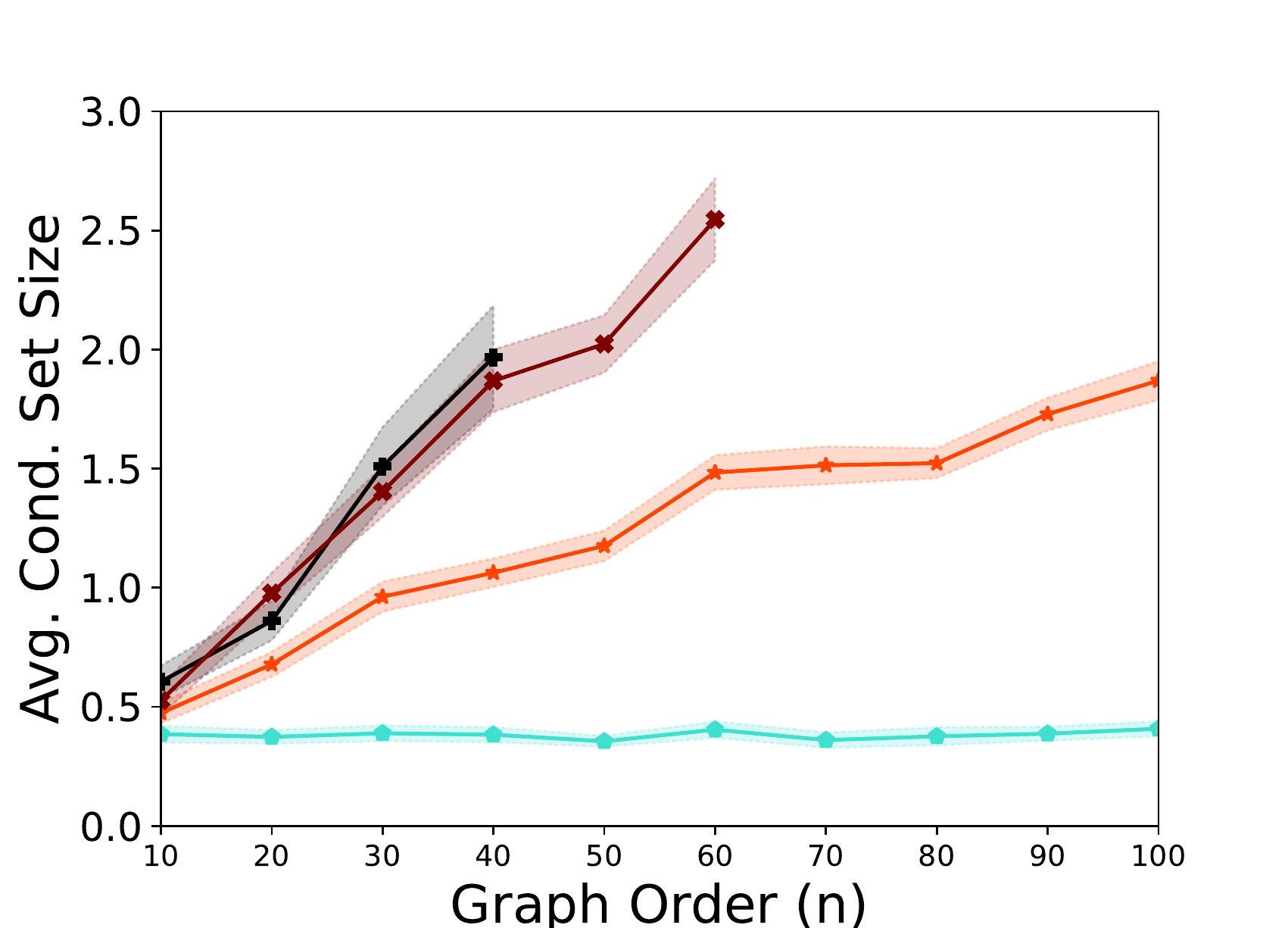}
            \caption{MAGs corresponding to $G(\Tilde{n},\frac{1}{\Tilde{n}^{0.9}})$ and latent rate = $20\%$.}
            \label{fig: ER 0.9 L=0.2}
        \end{subfigure}\hfill
        \begin{subfigure}[b]{1\textwidth}
            \centering
            \includegraphics[width=0.23\textwidth]{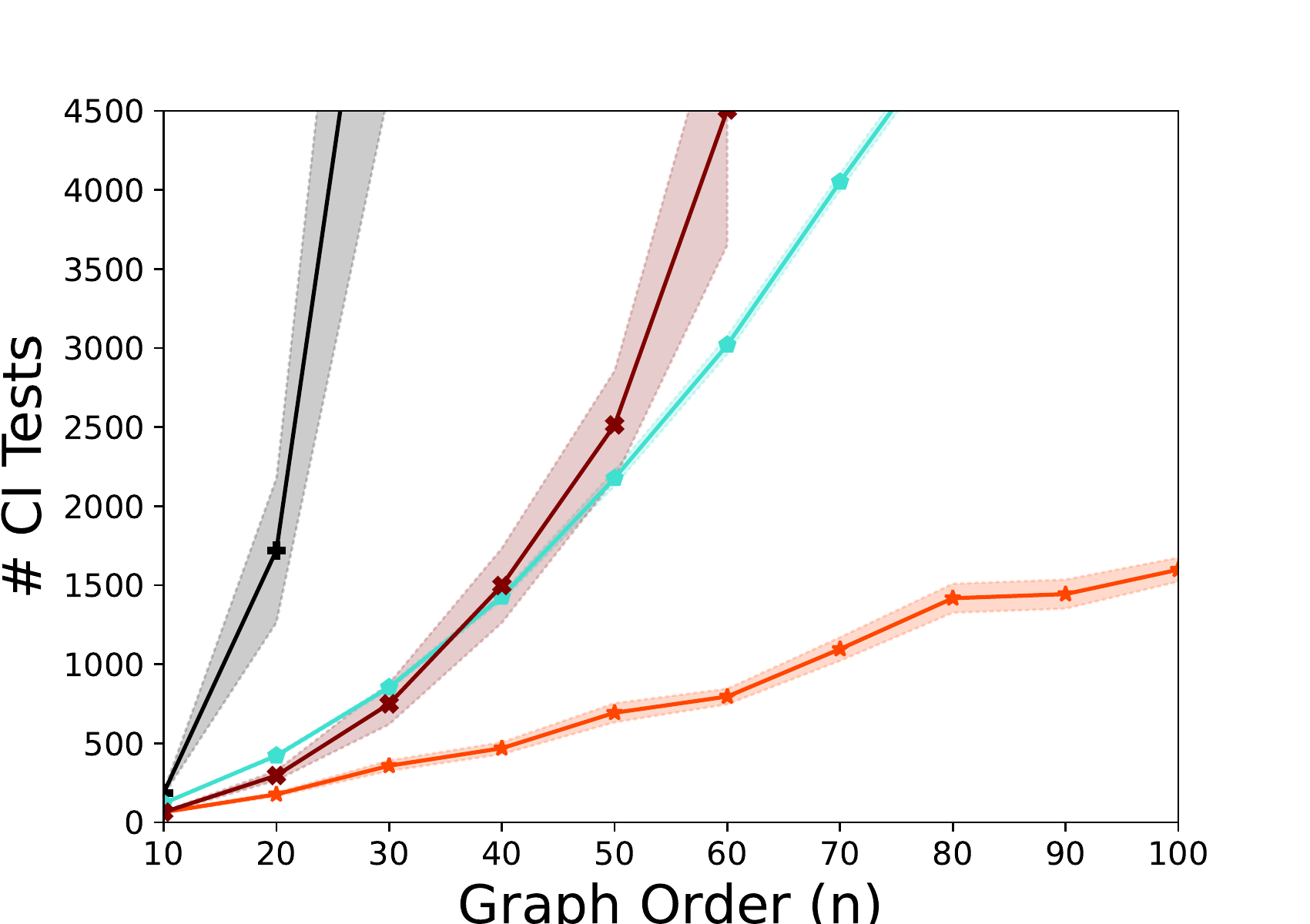}
            \hfill
            \includegraphics[width=0.23\textwidth]{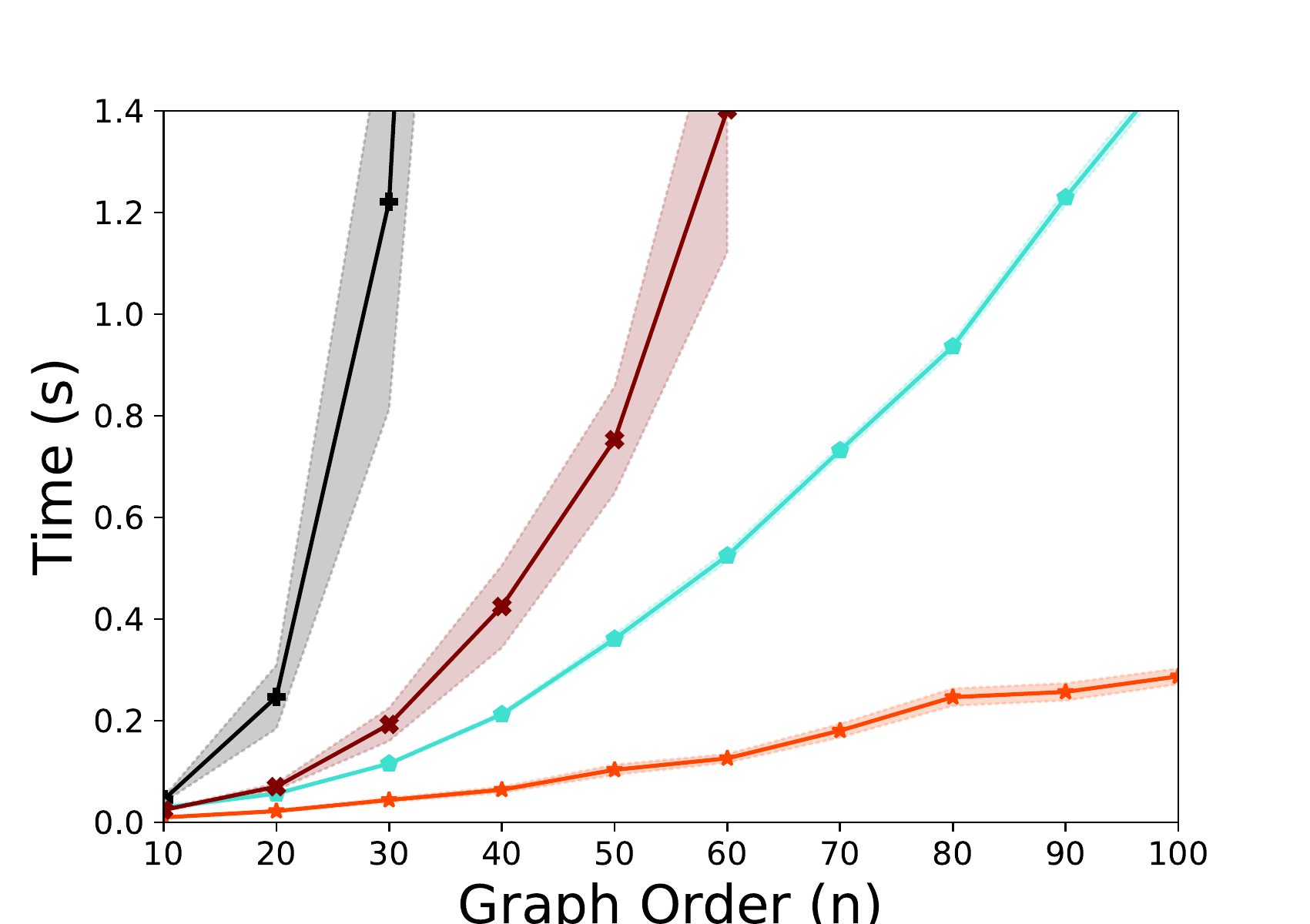}
            \hfill
            \includegraphics[width=0.23\textwidth]{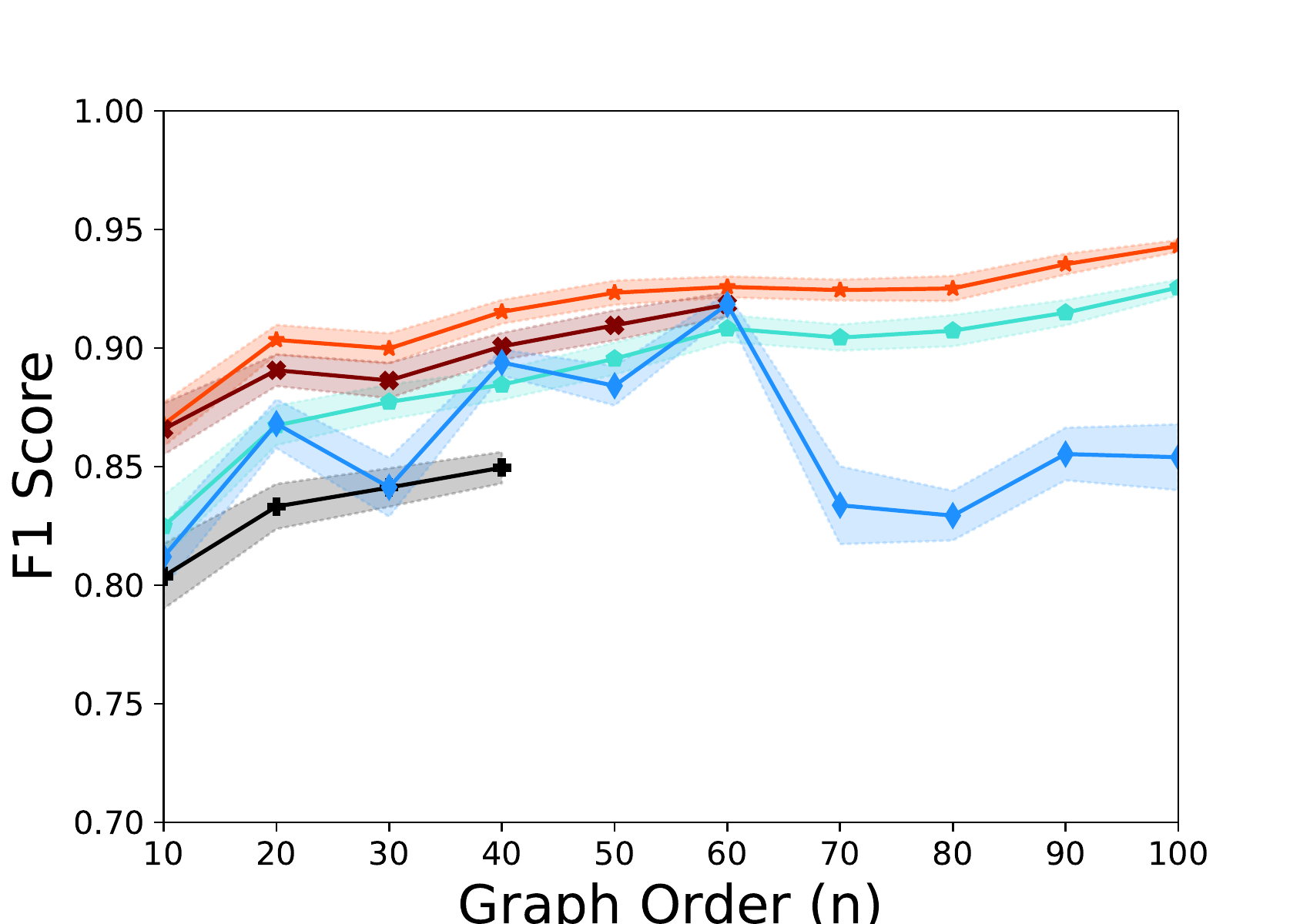}\hfill
            \includegraphics[width=0.23\textwidth]{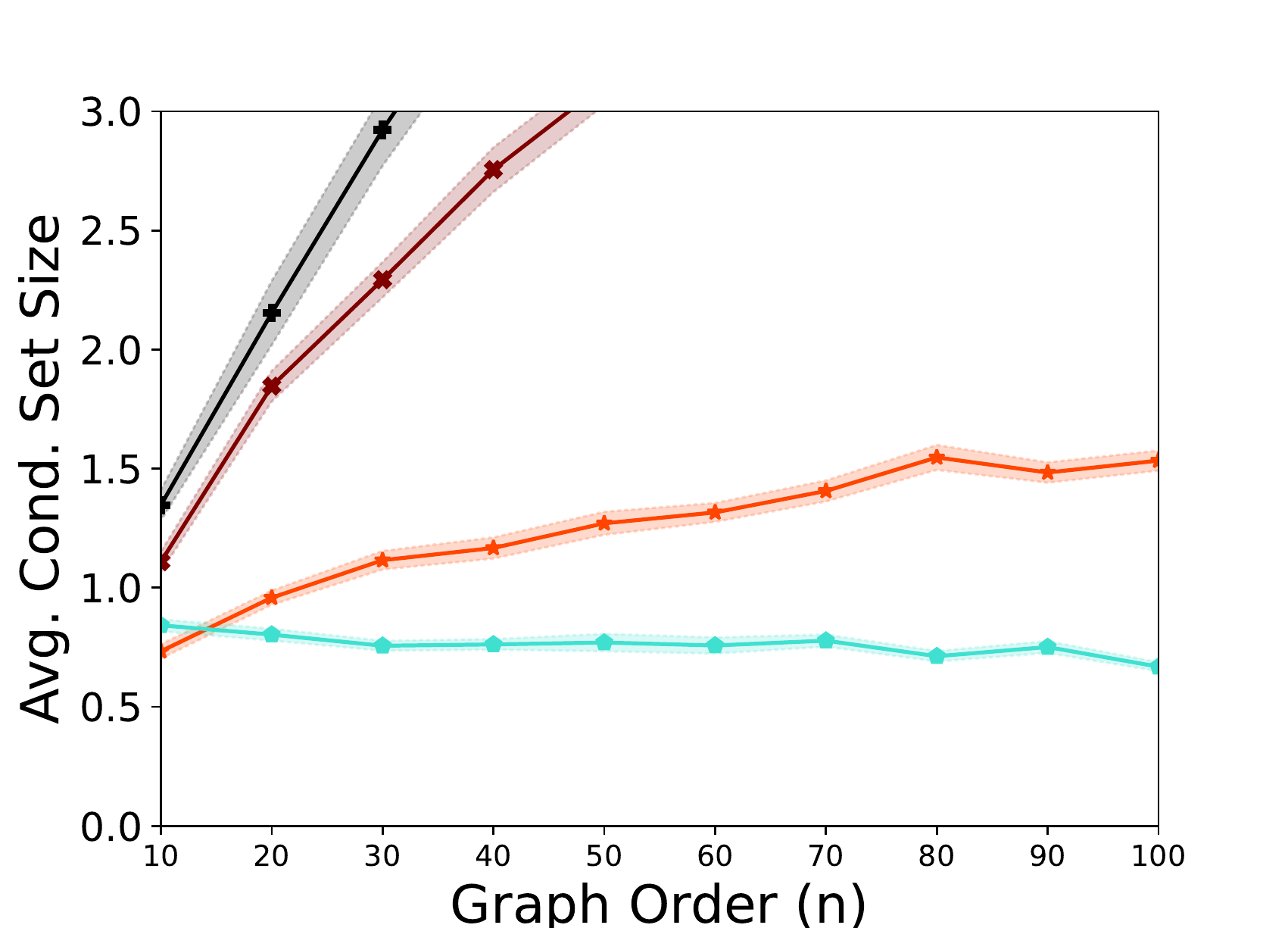}
            \caption{MAGs corresponding to random DAGs with a maximum of 3 parents for each variable, latent rate = $10\%$.}
            \label{fig: deltain=3 L=0.1}
        \end{subfigure}
        \hfill
        \begin{subfigure}[b]{1\textwidth}
            \centering
            \includegraphics[width=0.23\textwidth]{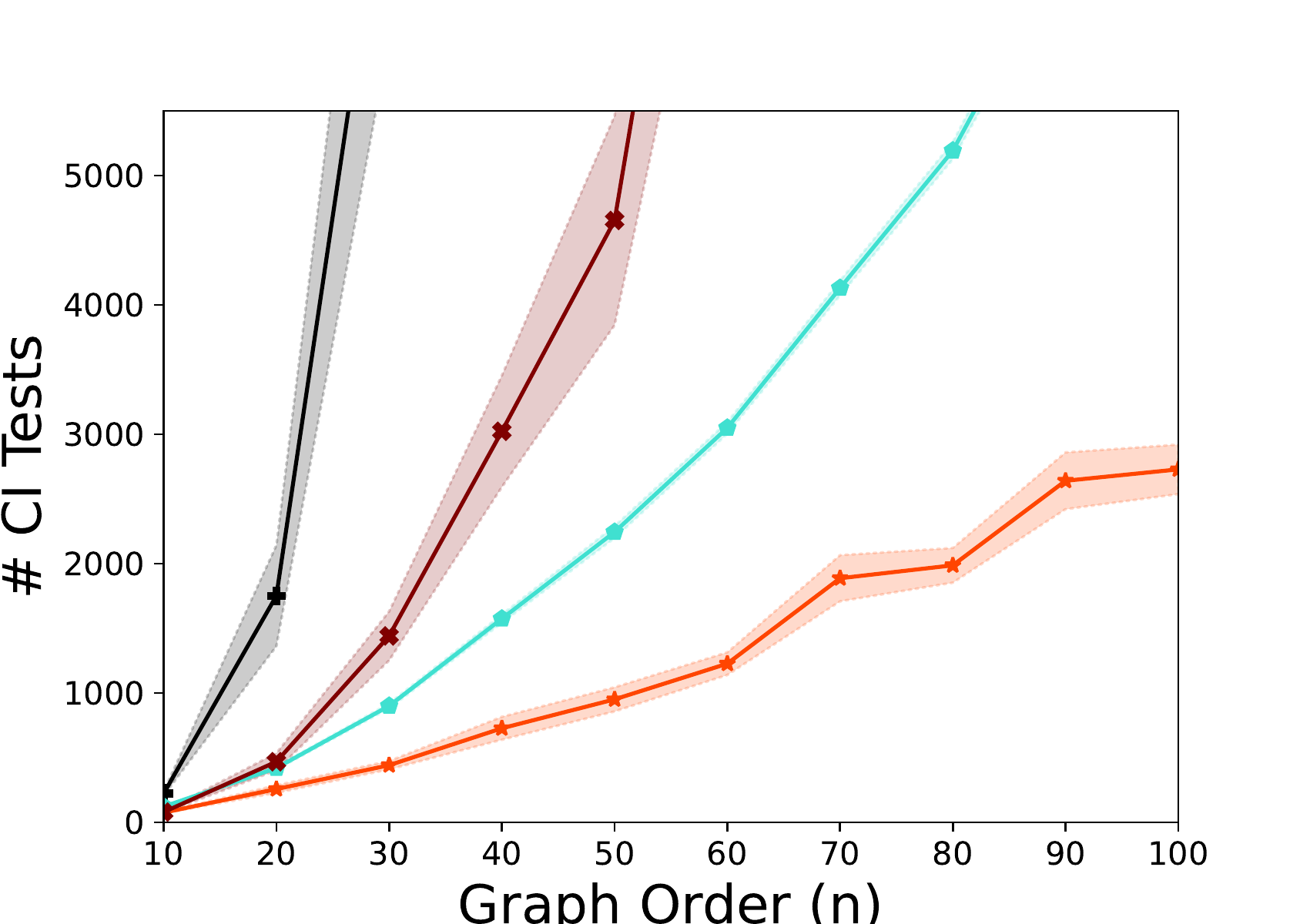}
            \hfill
            \includegraphics[width=0.23\textwidth]{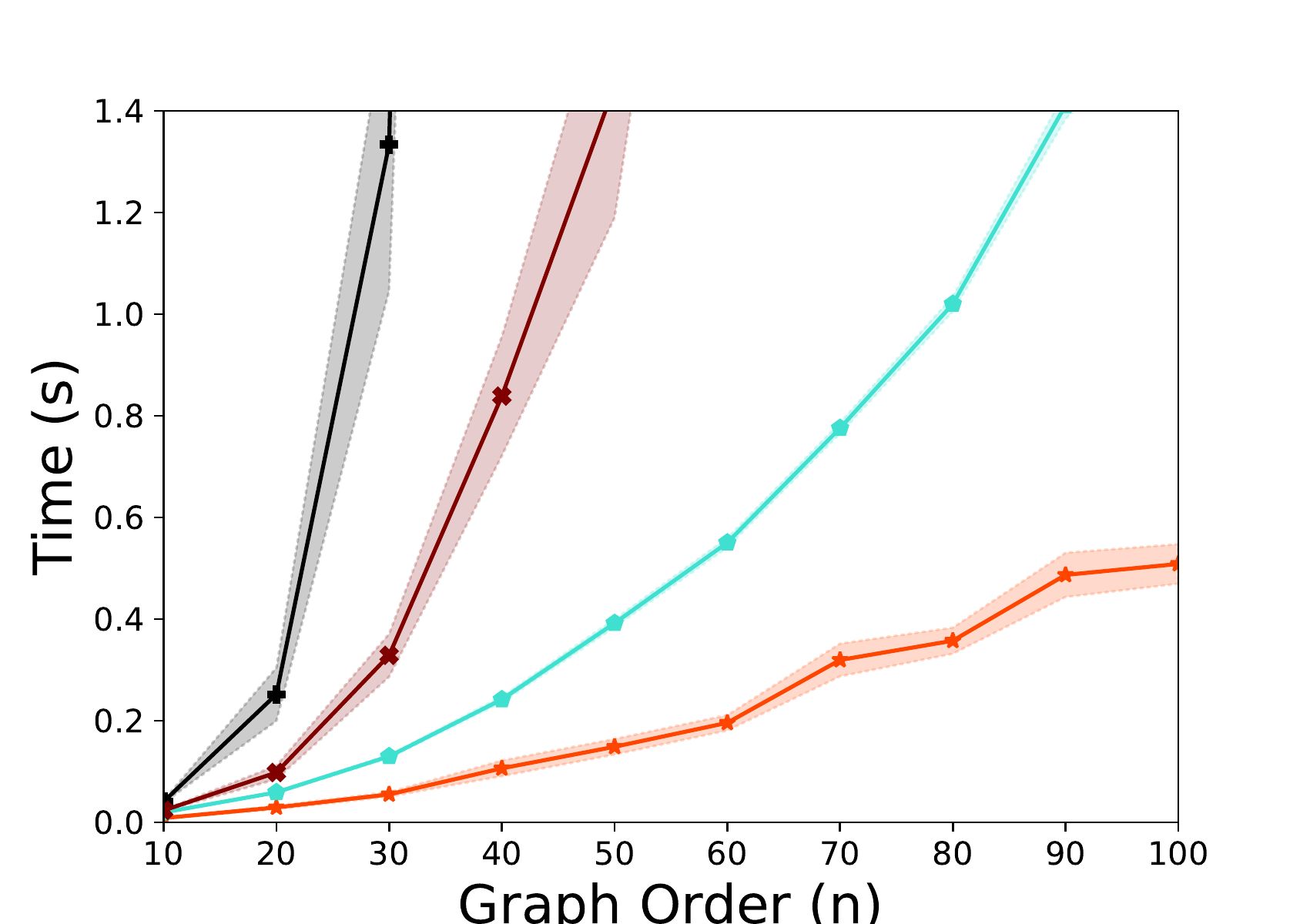}
            \hfill
            \includegraphics[width=0.23\textwidth]{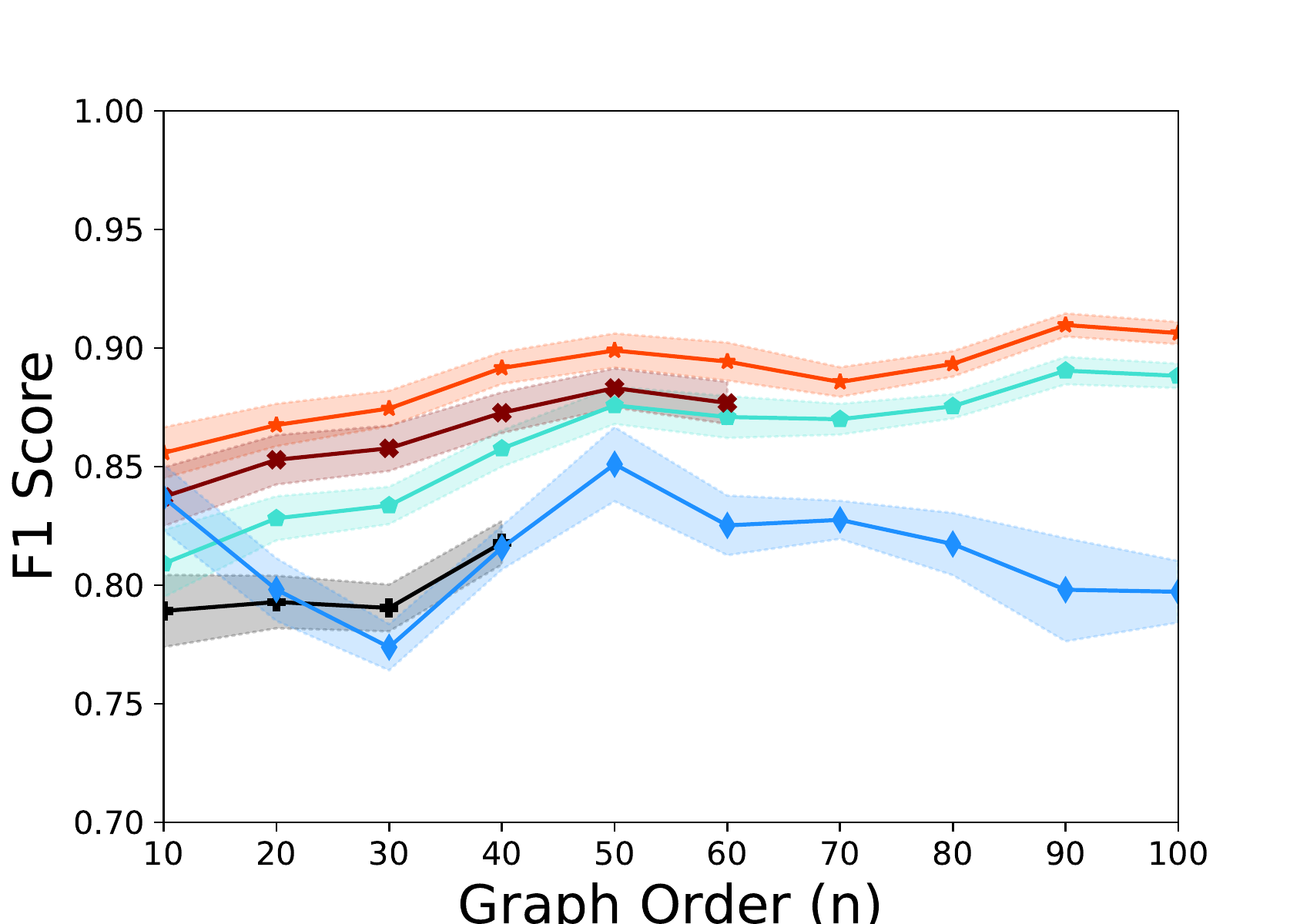}\hfill
            \includegraphics[width=0.23\textwidth]{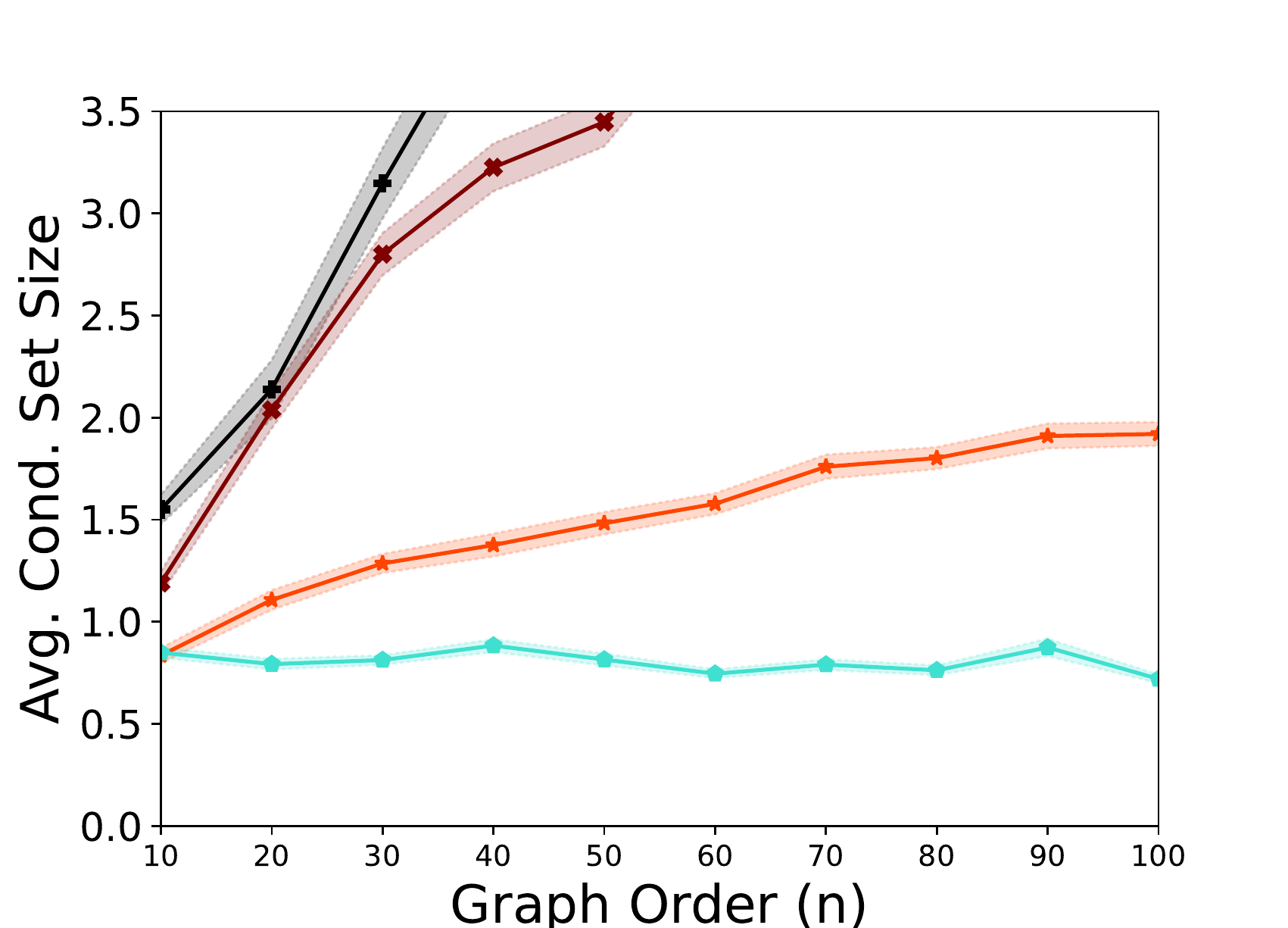}
            \caption{MAGs corresponding to random DAGs with a maximum of 3 parents for each variable, latent rate = $20\%$.}
            \label{fig: deltain=3 L=0.2}
        \end{subfigure}\hfill
        \begin{subfigure}[b]{1\textwidth}
            \centering
            \includegraphics[width=0.23\textwidth]{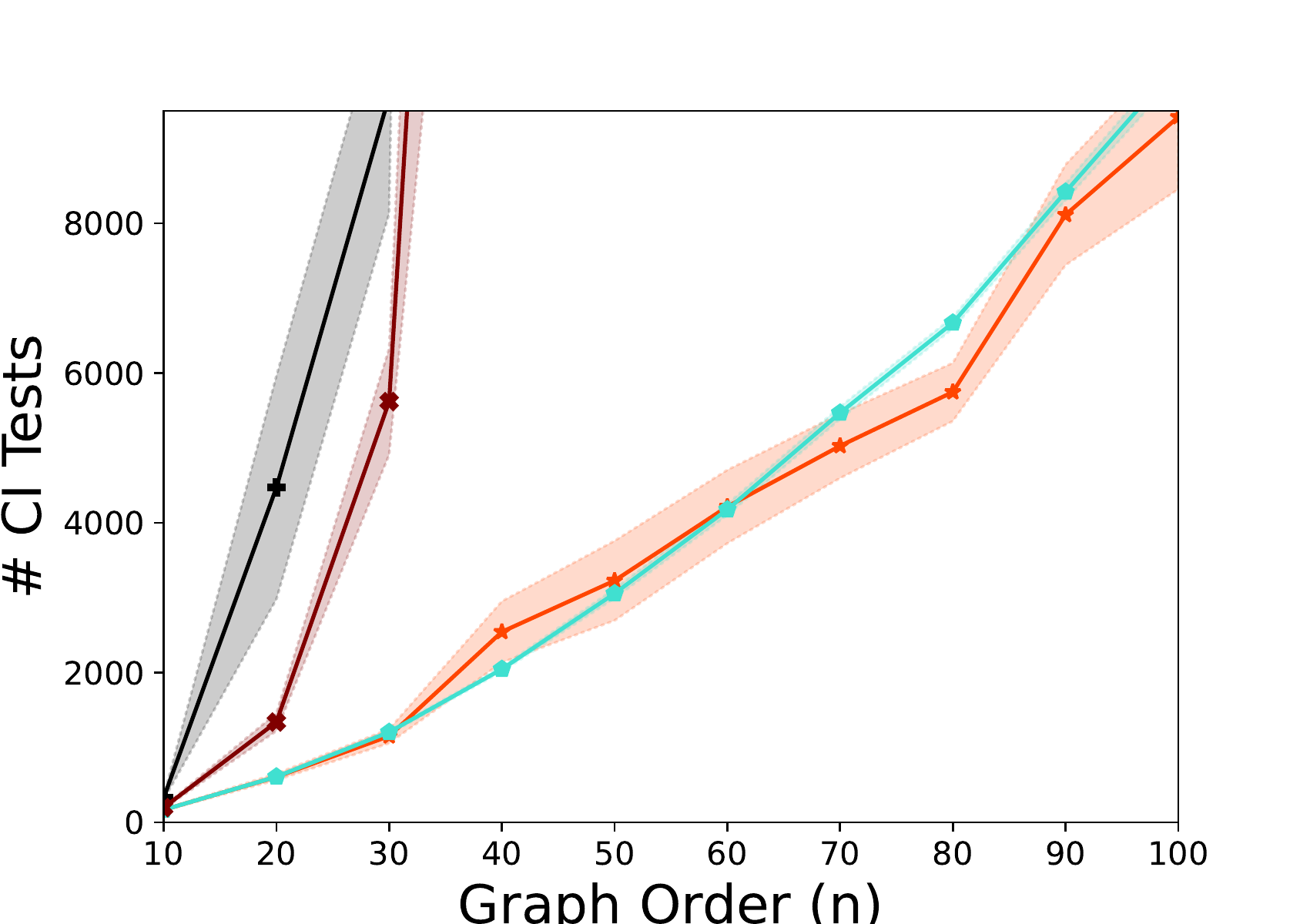}
            \hfill
            \includegraphics[width=0.23\textwidth]{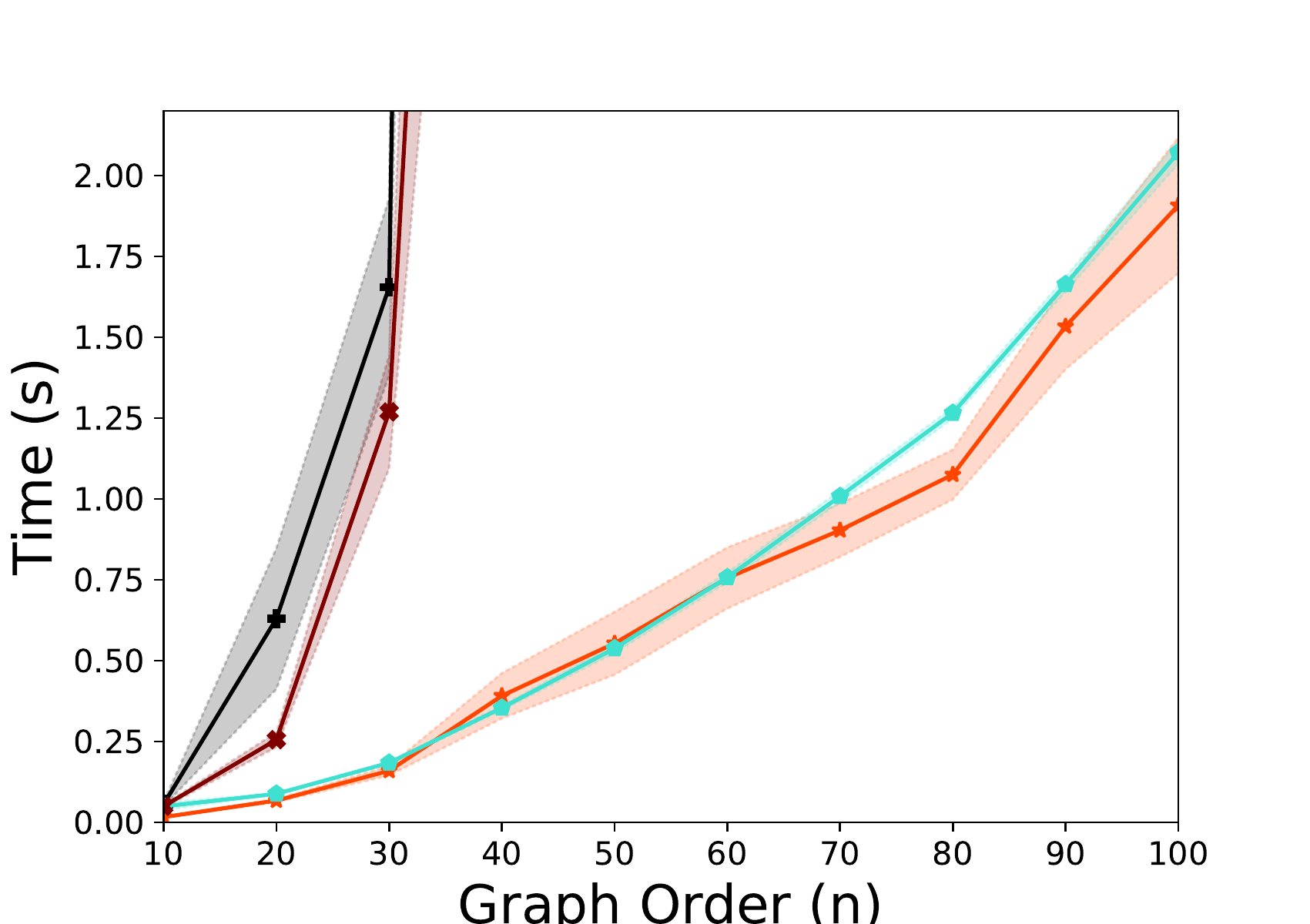}
            \hfill
            \includegraphics[width=0.23\textwidth]{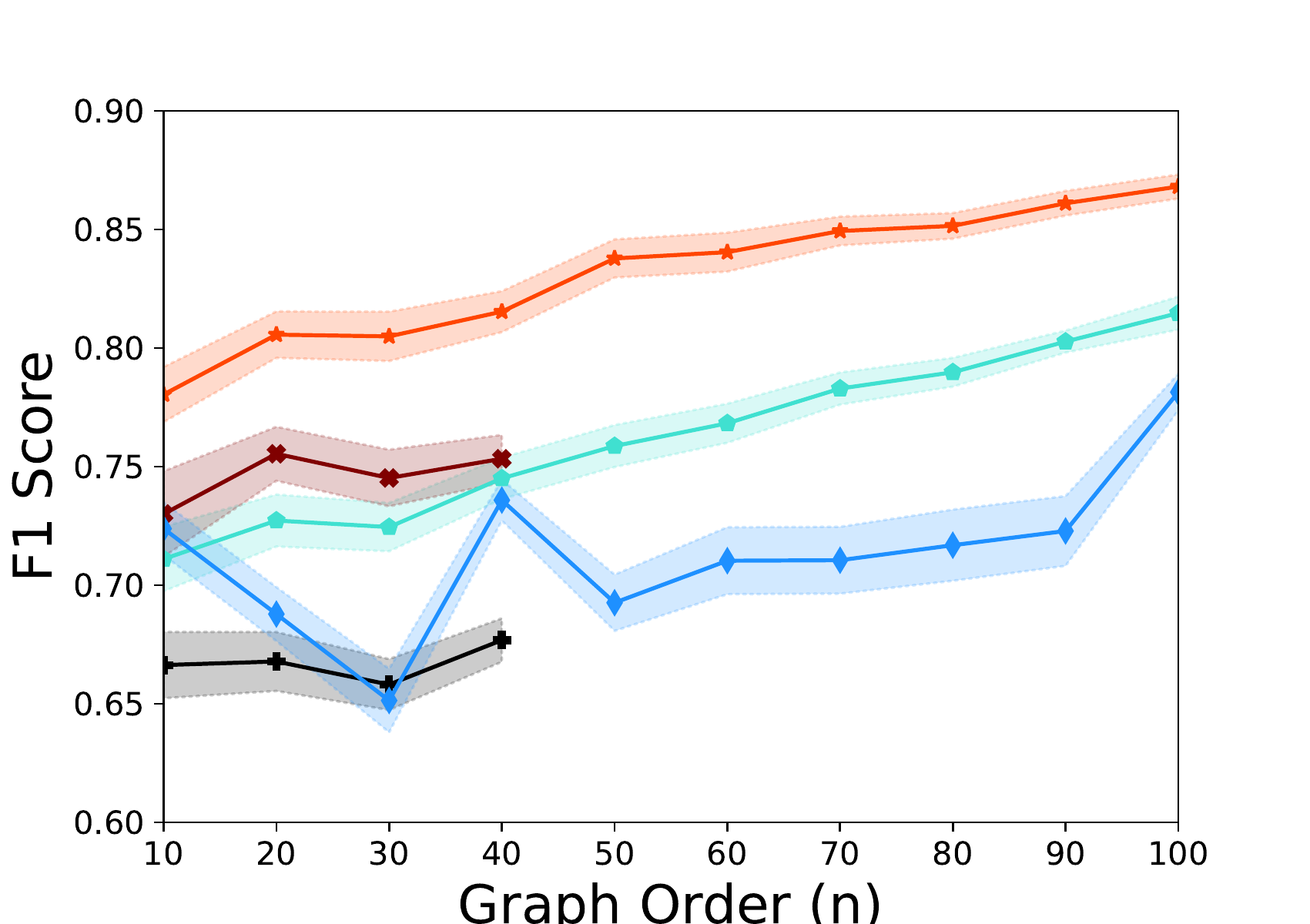}\hfill
            \includegraphics[width=0.23\textwidth]{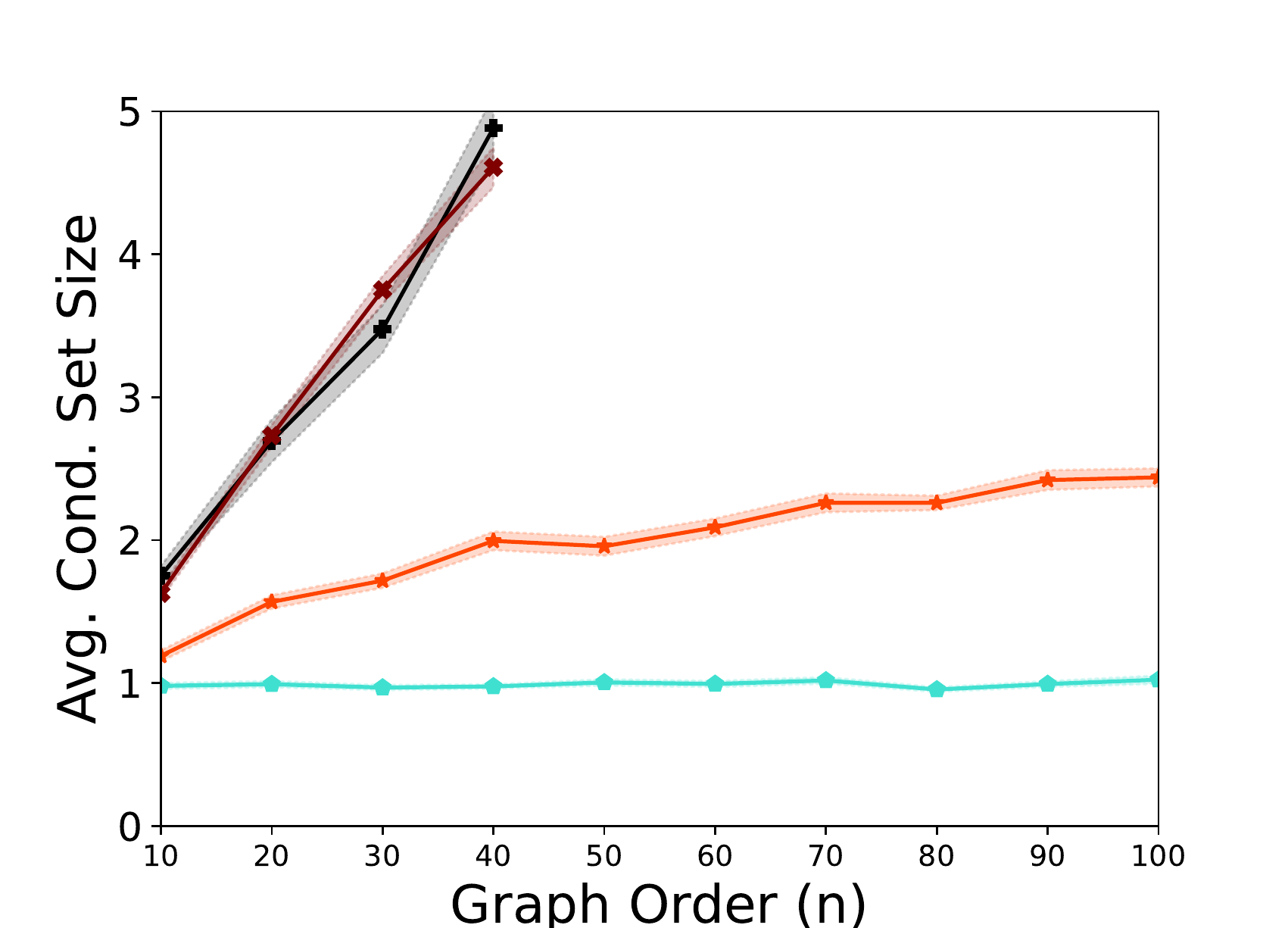}
            \caption{MAGs corresponding to random DAGs with a maximum of 4 parents for each variable, latent rate = $10\%$.}
            \label{fig: deltain=4 L=0.1}
        \end{subfigure}\hfill
        \begin{subfigure}[b]{0.8\textwidth}
            \centering
             \includegraphics[width=\textwidth]{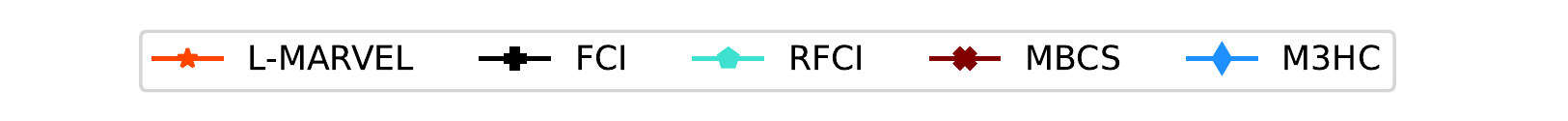}
        \end{subfigure}
        \caption{Performance of various algorithms on random graphs with significance level $\alpha=0.01$ and $50|\mathbfcal{O}|$ samples are available. Figures (a) and (b) demonstrate the evaluation over MAGs corresponding to Erdos-Renyi graphs, while (c), (d) and (e) represent the MAGs corresponding to random DAGs with bounded number of parents for each variable (sample size $=50|\mathbfcal{O}|$).}
        \label{fig: random graphs}
    \end{figure*}

    \textbf{Random Structures:}
    We used two different generating processes to obtain our random graphs. \begin{enumerate*}
        \item MAGs corresponding to DAGs generated by Erdos-Renyi model $G(\Tilde{n},p)$ \cite{erdHos1960evolution}, where $\Tilde{n}$ denotes the total number of the variables, and
        \item MAGs corresponding to random DAGs where each vertex has a maximum of 3 or 4 parents, similar to the setting in \cite{tsirlis2018scoring,mokhtarian2020recursive,chobtham2020bayesian}.
    \end{enumerate*}
    Figures \ref{fig: ER 0.9 L=0.1} and \ref{fig: ER 0.9 L=0.2} illustrate the performance of the algorithms on Erdos Renyi graphs, whereas Figures \ref{fig: deltain=3 L=0.1}, \ref{fig: deltain=3 L=0.2} and \ref{fig: deltain=4 L=0.1} represent the performance of these methods on the latter generative model.
    The coefficients of the linear model and the standard deviation of the exogenous noises are chosen uniformly at random from $\pm(0.5,2)$ and $(1,\sqrt{3})$, respectively. 
    We did not continue running algorithms that were not capable of keeping up with the cohort as the order of the graphs grew.
    Moreover, the runtime of M3HC is not reported in the plots as it does not fit into the scale of the plots. As seen in the plots, L-MARVEL demonstrates substantially lower computational complexity in terms of number of the CI tests and runtime compared to the other algorithms, while  maintaining high accuracy (the highest among the cohort in most of the cases). We also observed a low size for the conditioning sets in our CI tests for L-MARVEL. Only RFCI performs better than L-MARVEL in this metric\footnote{Note that RFCI avoids performing too many CI tests but with the caveat of lacking completeness.}.
    
    \textbf{Benchmark Structures:} Algorithms are evaluated on benchmark structures, where $5\%$ to $10\%$ of the variables are assumed to be latent, and $\sim 5\%$ of them are selection variables. 
    Latent and selection variables are chosen uniformly at random for each dataset.
    The coefficients of the linear model and the standard deviation of the noises are chosen uniformly at random from $\pm(0.5,1)$ and $(\sqrt{0.5},1)$, respectively. Our experiments, summarized in Table \ref{table: exp1}, demonstrate that L-MARVEL outperforms the other algorithms both in terms of computational complexity and the accuracy of the learned structure. NA entries for FCI demonstrate that the runtime exceeds a certain threshold.
    \begin{table*}[ht]
	    \caption{Performance of various algorithms on the benchmark structures, when $5\%$ to $10\%$ of the variables are latent and $\sim 5\%$ of them are selection variables (sample size $=50|\mathbfcal{O}|$).}
	    \fontsize{9}{10.5}\selectfont
	    \centering
	    \begin{tabular}{N M{0.3cm}|M{1.2cm}|M{1.2cm} M{1.1cm} M{1.1cm} M{1.1cm} M{1.1cm} M{1.1cm} M{1.1cm}}
    		\toprule
    		&\multicolumn{2}{c|}{Structure}
 			& Insurance
 			& Alarm
 			& Ecoli70
 			& Barley
 			& Hailfinder
 			& Carpo
 			& Arth150
			\\
			&\multicolumn{2}{c|}{$(|\mathbfcal{O}|, |\mathbfcal{L}|, |\mathbfcal{S}|)$}
 			& (22,3,2)
 			& (31,4,2)
 			& (40,3,3)
 			& (40,5,3)
 			& (50,3,3)
 			& (53,4,4)
 			& (95,6,6)
			\\
			\hline
			& \multirow{5}{*}{\rotatebox[origin=c]{90}{L-MARVEL}}
			& \#CI tests
			& \textbf{272} & \textbf{235} & \textbf{227} & \textbf{894} & \textbf{333} & \textbf{569} & \textbf{1185}
			\\
			& 
			& Runtime
			& \textbf{0.03} & \textbf{0.04} & \textbf{0.05} & \textbf{0.16} & \textbf{0.07} & \textbf{0.12} & \textbf{0.36}
			\\
			& 
			& F1-score
			 & \textbf{0.85} & \textbf{0.92} & \textbf{0.88} & \textbf{0.82} & \textbf{0.92} & \textbf{0.97} & \textbf{0.89}
			\\
			&
			& Precision
			 & 0.97 & 0.98 & 0.97 & 0.98 & 0.98 & 0.99 & 0.99
			\\
			&
			& Recall
			 & \textbf{0.76} & \textbf{0.87} & \textbf{0.81} & \textbf{0.72} & \textbf{0.87} & \textbf{0.96} & \textbf{0.82}
			\\
			\hline
			& \multirow{5}{*}{\rotatebox[origin=c]{90}{RFCI}}
			& \#CI tests
			 & 947 & 981 & 4314 & 2158 & 256754 & 11670 & 2644794
			\\
			& 
			& Runtime
			 & 0.14 & 0.20 & 0.86 & 0.44 & 62.22 & 2.59 & 1047.44
			\\
			& 
			& F1-score
			 & 0.76 & 0.89 & 0.85 & 0.73 & 0.88 & 0.94 & 0.87
			\\
			&
			& Precision
			& \textbf{0.99} & \textbf{1.00} & \textbf{1.00} & \textbf{1.00} & \textbf{1.00} & \textbf{1.00} & \textbf{1.00}
			\\
			&
			& Recall
			& 0.63 & 0.81 & 0.74 & 0.58 & 0.79 & 0.89 & 0.77
			\\
			\hline
			& \multirow{5}{*}{\rotatebox[origin=c]{90}{FCI}}
			& \#CI tests
			& 7117 & 6899 & 56781 & 117566 & NA & 123198 & NA
			\\
			& 
			& Runtime
			 & 1.13 & 1.25 & 13.22 & 25.78 & NA & 31.41 & NA
			\\
			& 
			& F1-score
			 & 0.75 & 0.88 & 0.83 & 0.70 & NA & 0.45 & NA
			\\
			&
			& Precision
			 & \textbf{0.99} & \textbf{1.00} & \textbf{1.00} & \textbf{1.00} & NA & 0.48 & NA
			\\
			&
			& Recall
			& 0.61 & 0.80 & 0.72 & 0.54 & NA & 0.42 & NA
			\\
			\hline
			& \multirow{5}{*}{\rotatebox[origin=c]{90}{MBCS*}}
			& \#CI tests
			& 640 & 335 & 499 & 2649 & 502 & 1221 & 3225
			\\
			& 
			& Runtime
			& 0.12 & 0.11 & 0.17 & 0.77 & 0.19 & 0.46 & 1.94
			\\
			& 
			& F1-score
			 & 0.80 & 0.90 & 0.86 & 0.76 & 0.89 & 0.96 & 0.86
			\\
			&
			& Precision
			& 0.98 & 0.98 & 0.98 & 0.99 & 0.99 & 0.99 & 0.99
			\\
			&
			& Recall
			& 0.68 & 0.84 & 0.77 & 0.62 & 0.82 & 0.94 & 0.76
			\\
			\hline
			& \multirow{5}{*}{\rotatebox[origin=c]{90}{M3HC}}
			& \#CI tests
			& 896 & 674 & 3033 & 1731 & 139788 & 8354 & 793754
			\\
			&
			& Runtime
			& 13.66 & 4.19 & 6.64 & 12.53 & 47.72 & 7.42 & 322.33
			\\
			& 
			& F1-score
			& 0.75 & 0.87 & 0.84 & 0.71 & 0.86 & 0.92 & 0.84
			\\
			&
			& Precision
			& \textbf{0.99} & \textbf{1.00} & \textbf{1.00} & 0.99 & \textbf{1.00} & \textbf{1.00} & 0.99
			\\
			&
			& Recall
			& 0.62 & 0.78 & 0.73 & 0.56 & 0.77 & 0.85 & 0.74
			\\
			\bottomrule
			\end{tabular}
	    \label{table: exp1}
    \end{table*}

    More comprehensive experimental results including the effect of the sample size, wider range of latent and selection rates, and assessments on different settings of parameters and structures, along with alternative metrics are reported in Appendix \ref{sec: apd experiments}.

\section{Concluding Remarks}
We proposed a recursive structure learning approach capable of handling latent and selection variables.
The recursive technique significantly reduced the number of required CI tests (and hence the time complexity). Also, since the order of the graph becomes smaller over the iterations, the recursive approach reduces the size of the conditioning sets in each CI test, which leads to an improved performance of the tests. We provided an upper bound on the complexity of the proposed method as well as a lower bound for any constraint-based method. The upper bound of our proposed approach and the lower bound at most differ by a factor equal to the number of variables in the worst case, which demonstrates the efficiency of the proposal.
We compared the performance of the proposed method with several state-of-the-art approaches on both synthetic and real-world structures. The results showed improvement in both performance and complexity on almost all the setups. We note that the performance of the proposed method is reliant on the accuracy of the Markov boundary information that is used in the algorithm.
Devising efficient and high accuracy approaches for learning the Markov boundary of the variables is left as an important direction for future work.

\ack{The work presented in this paper was in part supported by Office of Naval Research (ONR) under grant number W911NF-15-1-0479.}
\bibliographystyle{plainnat}
\bibliography{bibliography}

% %%%%%%%%%%%%%%%%%%%%%%%%%%%%%%%%%%%%%%%%%%%%%%%%%%%%%%%%%%%%
\clearpage
\appendix
{\Large \textbf{Appendix}}
\section{Removable Variables}\label{sec: apd graphical of removable}

In this section, we first prove the proposed graphical representation for a removable variable in a MAG $\mathcal{M}$ (Theorem \ref{thm: graph-rep}).
Then, we discuss how this representation reduces to Theorem 5 of \cite{mokhtarian2020recursive} in the case of DAGs.

Throughout our proofs, we say a path between $X$ and $Y$ is \emph{blocked} by a set $\mathbf{W}$ if it is not m-connecting relative to $\mathbf{W}$. 
In this case, there exists a non-collider $W$ on the path which is a member of $\mathbf{W}$, or there exists a collider $W$ on the path such that $W\notin \Anc{\{X,Y\}\cup\mathbf{W}}$. 
In both cases we say $W$ blocks this path with respect to $\mathbf{W}$, or $W$ blocks the path in short when $\mathbf{W}$ is clear from the context. 
We say $X$ is a descendant of $Y$ if $Y\in\Anc{X}$, and we denote by $\textit{De}_\mathcal{M}(X)$ the set of descendants of $X$ in the MAG $\mathcal{M}$, and $\De{X}$ whenever the graph is clear from the context. 
\subsection{Graphical representation}
\begin{customthm}{\ref{thm: graph-rep}}
    Vertex $X$ is removable in a MAG $\mathcal{M}$ over the variables $\mathbf{V}$, if and only if 
    \begin{enumerate}
        \item for any $Y\in \Adj{X}$ and $Z\in \Ch{X}\cup \N{X}\setminus \{Y\}$, $Y$ and $Z$ are adjacent, and 
        \item for any collider path $u=(X,V_1,...,V_m,Y)$ and $Z\in\mathbf{V}\setminusA\{X,Y,V_1,...,V_m\}$ such that $\{X,V_1,...,V_m\}\subseteq\Pa{Z}$, $Y$ and $Z$ are adjacent.
    \end{enumerate}
\end{customthm}
\begin{proof} Let $\mathcal{H}$ denote the induced subgraph of $\mathcal{M}$ over $\mathbf{V}\setminusA\{X\}$.

\textbf{only if part:}
Suppose $Y\in\Adj{X}$ and $Z\in\Ch{X}\cup \N{X}$. For any $\mathbf{W}\subseteq \mathbf{V}\setminusA\{X,Y,Z\}$, $(Z,X,Y)$ is an m-connecting path relative to $\mathbf{W}$ in $\mathcal{M}$, as $X$ is a non-collider and $X\notin\mathbf{W}$.
That is, no such $\mathbf{W}$ can m-separate $Y$ and $Z$.
Since $X$ is removable in $\mathcal{M}$, by definition of removability,
\begin{equation}\label{eq: proof only if thm1}\msep{Y}{Z}{\mathbf{W}}{\mathcal{M}}
\iff
\msep{Y}{Z}{\mathbf{W}}{\mathcal{H}}.
\end{equation}
As a result, $Y$ and $Z$ have no m-separating sets in $\mathcal{H}$.
Hence, $Y$ is adjacent to $Z$ in $\mathcal{H}$, and therefore, in $\mathcal{M}$.

Now suppose $u=(X,V_1,...,V_m,Y)$ is a collider path and $\{X,V_1,...,V_m\}\subseteq\Pa{Z}$.
Again for any $\mathbf{W}\subseteq \mathbf{V}\setminusA\{X,Y,Z\}$, $(Z,X,V_1,...,V_m,Y)$ is an m-connecting path relative to $\mathbf{W}$ in $\mathcal{M}$ since I) every collider on this path is a parent (and therefore an ancestor) of $Z$, and II) $X\notin\mathbf{W}$ and $X$ is the only non-collider on this path.
That is, no such $\mathbf{W}$ can m-separate $Y$ and $Z$.
Since $X$ is removable in $\mathcal{M}$, Equation \ref{eq: proof only if thm1} implies that $Y$ and $Z$ have no m-separating sets in $\mathcal{H}$.
Hence, $Y$ is adjacent to $Z$ in $\mathcal{H}$, and therefore, in $\mathcal{M}$.

\textbf{if part:} We need to prove that for any $Y,Z\in \mathbf{V}\setminus \{X\}$ and any $\mathbf{W}\subseteq \mathbf{V}\setminus \{X,Y,Z\}$,
\[ 
    \msep{Y}{Z}{\mathbf{W}}{\mathcal{M}}
    \iff
    \msep{Y}{Z}{\mathbf{W}}{\mathcal{H}}.
\]
$\Rightarrow:$ 
Suppose $\msep{Y}{Z}{\mathbf{W}}{\mathcal{M}}$ and let $u$ be an arbitrary path in $\mathcal{H}$ between $Y$ and $Z$.
Since $\mathcal{H}$ is a subgraph of $\mathcal{M}$, $u$ is also a path in $\mathcal{M}$.
As $\msep{Y}{Z}{\mathbf{W}}{\mathcal{M}}$, $u$ is not m-connecting relative to $\mathbf{W}$ in $\mathcal{M}$, Lemma \ref{lem: path block} implies that $u$ is not m-connecting relative to $\mathbf{W}$ in $\mathcal{H}$.

$\Leftarrow:$ 
Suppose $\msep{Y}{Z}{\mathbf{W}}{\mathcal{H}}$, i.e., there is no m-connecting path between $Y$ and $Z$ in $\mathcal{H}$. 
It suffices to show that none of the paths between $Y$ and $Z$ in $\mathcal{M}$ are m-connecting. 
Take an arbitrary path $u=(Y,V_1,...,V_m,Z)$ in $\mathcal{M}$. 
We will show that $u$ is not m-connecting relative to $\mathbf{W}$ in $\mathcal{M}$. 
We consider the following cases separately.
\begin{enumerate}[leftmargin=*]
    \item $X\notin u$:
        In this case, $u$ is also a path in $\mathcal{H}$. 
        Since $u$ is not m-connecting relative to $\mathbf{W}$ in $\mathcal{H}$, Lemma \ref{lem: path block} implies that $u$ is not m-connecting relative to $\mathbf{W}$ in $\mathcal{M}$.
    \item $X$ is a non-collider on $u$: 
        Suppose $u = (Y,V_1,\dots,V_{i-1},V_i=X,V_{i+1},\dots, V_m,Z)$.
        We claim that a vertex other than $X$ blocks $u$ in $\mathcal{M}$. 
        Suppose not. 
        Since $X$ is a non-collider, at least one of $V_{i-1}$ and $V_{i+1}$ is a child or neighbor of $X$. From the assumption of the theorem, $V_{i-1}\in\Adj{V_{i+1}}$.
        Now consider the path $u'=(Y,V1,...,V_{i-1},V_{i+1}, ..., V_m,Z)$, which is a path in $\mathcal{H}$ and must not be m-connecting relative to $\mathbf{W}$ in $\mathcal{H}$. 
        Hence, Lemma \ref{lem: path block} implies that $u'$ is not m-connecting relative to $\mathbf{W}$ in $\mathcal{M}$. 
        If a vertex other than $\{V_{i-1},V_{i+1}\}$ blocks $u'$ in $\mathcal{M}$, the same vertex blocks $u$, which is a contradiction. 
        Suppose without loss of generality that $V_{i-1}$ blocks $u'$ in $\mathcal{M}$.
        If $V_{i-1}$ is a collider on both $u$ and $u'$ or a non-collider on both of them, $V_{i-1}$ blocks $u$ in $\mathcal{M}$ which is a contradiction.
        So suppose $V_{i-1}$ is a collider on one of $u$ and $u'$, and a non-collider on the other one.
        From Lemma \ref{lem: X non-collider V parent}, $V_{i-1},X\in\Pa{V_{i+1}}$.
        Also, $V_{i-1}$ is a collider on $u$ in this case, that is, $(V_{i-2},V_{i-1},X)$ is a collider path.
        From the assumption of the theorem, $V_{i-2}\in\Adj{V_{i+1}}$.
        The edge between $V_{i-2}$ and $V_{i+1}$ has an arrowhead at $V_{i+1}$, as otherwise an (almost) directed cycle is formed over $V_{i-2},V_{i-1},V_{i+1}$.
        Now define the path $u''$ as $u''=(Y,V1,...,V_{i-2},V_{i+1}, ..., V_m,Z)$.
        This path also exists in $\mathcal{H}$, and therefore, $u''$ is not m-connecting relative to $\mathbf{W}$ in $\mathcal{H}$.
        Hence, $u''$ is not m-connecting relative to $\mathbf{W}$ in $\mathcal{M}$.
        If a vertex other than $V_{i-2}$ blocks $u''$ in $\mathcal{M}$, it also blocks $u'$ in $\mathcal{M}$, which is a contradiction, since we assumed that only $V_{i-1}$ blocks this path.
        If $V_{i-2}$ is a collider on both $u'$ and $u''$, or a non-collider on both of them, $V_{i-2}$ blocks $u'$ in $\mathcal{M}$, which is a contradiction.
        Now applying Lemma \ref{lem: X non-collider V parent} implies that $V_{i-2}\in\Pa{V_{i+1}}$ and $(V_{i-3},V_{i-2},V_{i-1},X)$ is collider path.
        Continuing in this manner finally implies that $Y\in\Adj{V_{i+1}}$ and the edge between $Y$ and $V_{i+1}$ has an arrowhead at $V_{i+1}$.
        Now since the path $(Y,V_{i+1},...,V_m,Z)$ is not m-connecting relative to $\mathbf{W}$, there exists a vertex $T$ that blocks it in $\mathcal{M}$. 
        The same vertex must block $(Y,V_1,V_{i+1},...,V_m,Z)$, which is a contradiction.
        Note that now $T$ is either a collider on both of these paths, or a non-collider on both of them.
        Also note that the assumption that $V_{i-1}$ blocks $u'$ in $\mathcal{M}$ does not violate the generality of the proof as if we assumed that $V_{i+1}$ blocks $u'$, that would imply the same arguments for the paths $(Y,V_1,...,V_{i-1},V_j,V_{j+1},...,V_m,Z)$, with the only difference that $Y$ and $Z$ would be interchanged throughout the proof.
    \item $X$ is a collider on $u$: 
        Suppose $u = (Y,V_1,\dots,V_{i-1},V_i=X,V_{i+1},\dots, V_m,Z)$. 
        If a vertex other than $X$ blocks $u$ in $\mathcal{M}$, we are done. Otherwise, we claim that $X$ blocks $u$ in $\mathcal{M}$. 
        Since $X\notin\mathbf{W}$, it suffices to show that $\textit{De}_\mathcal{M}(X) \cap \big( \{Y,Z\} \cup \mathbf{W} \big) = \varnothing$. 
        Assume by contradiction that there exists a directed path from $X$ to a vertex in $\{Y,Z\}\cup\mathbf{W}$, and let $T\in\Ch{X}$ denote the first vertex next to $X$ on this path. 
        Note that $T\notin \{V_{i-1},V_{i+1}\}$.
        Since $(V_{i-1},X)$ and $(V_{i+1},X)$ are collider paths and $X\in\Pa{T}$, $V_{i-1},V_{i+1}\in\Adj{T}$ from the assumption. 
        Both of these edges must have arrows on the side of $T$, as otherwise, an (almost) directed cycle would occur. 
        Therefore, $T$ is a collider on $(V_{i-1},T,V_{i+1})$. 
        Now, consider the path $u'=(Y,V1,...,V_{i-1},T,V_{i+1}, ..., V_m,Z)$, which is a path in $\mathcal{H}$ and must not be m-connecting relative to $\mathbf{W}$ in $\mathcal{H}$. 
        Hence, Lemma \ref{lem: path block} implies that $u'$ is not m-connecting relative to $\mathbf{W}$ in $\mathcal{M}$. 
        If a vertex other than $\{V_{i-1},T,V_{i+1}\}$ blocks $u'$ in $\mathcal{M}$, the same vertex blocks $u$, which is a contradiction. 
        $T$ cannot block $u'$ in $\mathcal{M}$ as it is a collider on $u'$ and it has a descendant in $\{Y,Z\}\cup\mathbf{W}$.
        Thus, suppose without loss of generality that $V_{i-1}$ blocks $u'$ in $\mathcal{M}$. 
        If $V_{i-1}$ is a collider on both $u$ and $u'$ or a non-collider on both of them, $V_{i-1}$ blocks $u$ in $\mathcal{M}$ which is a contradiction.
        So suppose $V_{i-1}$ is a non-collider on $u'$ and a collider on $u$.
        Note that the other case is not possible because an (almost) directed cycle would occur over the vertices $V_{i-1},X,T$.
        As a result, $V_{i-1}\in\Pa{T}$.
        Now, consider the collider path $(V_{i-2},V_{i-1},X)$ in which $V_{i-1},X\in\Pa{T}$. 
        Therefore, $V_{i-2}\in\Adj{T}$.
        Again, this edge must have an arrowhead on the side of $T$, as otherwise an (almost) directed cycle is formed over $(V_{i-2},V_{i-1},T$.
        Now, consider the path $u''=(Y,V1,\dots,V_{i-2},T,V_{i+1},\dots, V_m,Z)$, which is a path in $\mathcal{H}$, and therefore, is not m-connecting relative to $\mathbf{W}$ in $\mathcal{H}$. 
        In this case, Lemma \ref{lem: path block} implies that $u''$ is not m-connecting relative to $\mathbf{W}$ in $\mathcal{M}$. 
        We can repeat the arguments above for this path, implying that either there exists a vertex that blocks $u$ in $\mathcal{M}$, or $V_{i-2}\in\Pa{T}$, and therefore, $V_{i-3}\in\Adj{T}$ (or alternatively, $V_{i+1}\in\Pa{T}$, and therefore, $V_{i+2}\in\Adj{T}$, which does not alter the proof.)
        Continuing in the same manner, either there exists a vertex that blocks $u$ in $\mathcal{M}$ which is a contradiction, or $Y,Z\in\Adj{T}$, where $T$ is a collider on $(Z,T,Y)$.
        Finally, $(Z,T,Y)$ is a path in $\mathcal{H}$ and must not be m-connecting relative to $\mathbf{W}$, but this is not possible because $\text{De}_\mathcal{M}(Y)\cap\Anc{\{Y,Z\}\cup\mathbf{W}}\neq\varnothing$. 
        This contradiction proves that $X$ cannot have a descendant in $\{Y,Z\}\cup\mathbf{W}$, which implies that $X$ blocks $u$ in $\mathcal{M}$.
\end{enumerate} 
In all of the cases, $u$ is not m-connecting relative to $\mathbf{W}$, which completes the proof.
\end{proof}
\subsection{Reduction to DAGs}
The notion of removability is first discussed in \cite{mokhtarian2020recursive} for the case of DAGs. 
Herein, we discuss how our definition of removability for MAGs (Definition \ref{def: removable}) and the provided graphical representation (Theorem \ref{thm: graph-rep}) can be reduced to their results when we restrict ourselves to the space of DAGs.
Note that our removability tests in Theorem \ref{thm: test removability} do not reduce to what they proposed for DAGs.
For instance, we directly test the removability of a vertex without identifying its so-called co-parents.
\begin{itemize}[leftmargin=*]
    \item Definition \ref{def: removable}: 
        In the case of DAGs, m-separation reduces to d-separation. Hence, Definition \ref{def: removable} is reduced to what \cite{mokhtarian2020recursive} proposed in the case of DAGs.
    \item Graphical representation:
        Suppose the ground-truth graph is a DAG.
        Note that collider paths in DAGs can be of length at most two and the vertices have no neighbors.
        In this case, our graphical representation of a removable variable in Theorem \ref{thm: graph-rep} is reduced to what is proposed in Theorem 5 of \cite{mokhtarian2020recursive}.
\end{itemize}
The removability test provided in \cite{mokhtarian2020recursive} fails in the case that causal sufficiency is violated. 
Consider for example the vertex $X$ in Figure \ref{fig: graph-rep not rem}. 
If the proposed tests of \cite{mokhtarian2020recursive} are performed for $X$, then $Z$ and $V_1$ are identified to be adjacent to $X$, and then the collider paths $X\to Z\gets V_1$, $X\to Z\gets V_2$, and $X\to V_1\gets V_2$ are identified. 
Then due to their removability tests, $X$ is decided to be removable since the pairs $(Z,V_1)$, $(Z,V_2)$ and $(V_1,V_2)$ cannot be m-separated. 
However, we know from Theorem \ref{thm: graph-rep} that $X$ is not removable in this MAG.
\section{Proofs} \label{sec: apd proof}
In this section, we first present fundamental lemmas used throughout our proofs. The proofs for the results of the main text is provided in Appendix \ref{sec: apd main results}.
\subsection{Preliminary lemmas}
\begin{lemma}\label{lem: preserveDescendants}
    Suppose $X$ is a vertex in a MAG $\mathcal{M}$ with vertex set $\mathbf{V}$ such that if $Y\in \Pa{X}$ and $Z\in \Ch{X}$, then $Y\in \Pa{Z}$.
    Let $\mathcal{H}$ be the induced subgraph of $\mathcal{M}$ over $\mathbf{V}\setminus \{X\}$.
    Note that $\mathcal{H}$ is also a MAG.
    In this case, for any $Y\in \mathbf{V}\setminus \{X\}$, 
    \[\text{De}_\mathcal{M}(Y)\setminusA\{X\}=\text{De}_\mathcal{H}(Y).\]
\end{lemma}
\begin{proof}
    Suppose $Z\in\text{De}_\mathcal{M}(Y) \setminus \{X\}$, i.e., there exists a directed path from $Y$ to $Z\neq X$ in $\mathcal{M}$. 
    If this path does not pass through $X$, the same path exists in $\mathcal{H}$, and $Z\in\text{De}_\mathcal{H}(Y)$. 
    Otherwise, suppose this path is $(Y,U_1,\dots,U_i,X,U_{i+1},\dots,Z)$. 
    Since $U_i\in\Pa{X}$ and $U_{i+1}\in\Ch{X}$, $U_i\in \Pa{U_{i+1}}$. 
    Hence, $(Y,U_1,\dots,U_i,U_{i+1},\dots,Z)$ is a directed path in $\mathcal{H}$, and $Z\in\text{De}_\mathcal{H}(Y)$. 
    This implies that
    \[\text{De}_\mathcal{M}(Y)\setminusA\{X\}\subseteq\text{De}_\mathcal{H}(Y).\]
    Furthermore, if there exists a directed path from $Y$ to $Z$ in $\mathcal{H}$, the same path exists in $\mathcal{M}$, which implies that 
    \[\text{De}_\mathcal{H}(Y)\subseteq\text{De}_\mathcal{M}(Y)\setminusA\{X\}.\]
    This completes the proof.
\end{proof}
\begin{lemma}\label{lem: parent-plus}
    Let $X$ and $Y$ be two non-adjacent vertices in a MAG $\mathcal{M}$, where $X\notin\Anc{Y}$. 
    Then 
    \begin{equation}
        \msep{X}{Y}{\mathbf{W}\setminusA\{X,Y\}}{\mathcal{M}}, \hspace{5pt} \text{where } \mathbf{W}= N(X) \cup \big(\PaP{X}\cap\Anc{\{X,Y\}}\big).
    \end{equation}
\end{lemma}
\begin{proof}
    Let $u = (X=V_0,V_1,\dots,V_m,Y=V_{m+1})$ be an arbitrary path between $X$ and $Y$.
    It suffices to show that $\mathbf{W}\setminusA\{X,Y\}$ blocks $u$. 
    Let $i$ be the largest index such that all the edges on $(V_0,V_1,...,V_i)$ are bidirectional. We consider the following cases separately.
    \begin{enumerate}[leftmargin=*]
        \item $i\geq m$: In this case, all the vertices $V_1,...,V_m$ on the path are colliders that belong to $\PaP{X}$. 
        Since $X$ and $Y$ are non-adjacent, $u$ is not an inducing path. 
        Hence, there exists $j$ such that $V_j\notin\Anc{\{X,Y\}}$ and therefore, $V_j\notin\Anc{\mathbf{W}\cup\{X,Y\}}$. 
        Hence, $V_j$ blocks $u$.
        \item $i=0$: If $V_1\in\Pa{X}\cup N(X)$, then $V_1\in\mathbf{W}\setminusA\{X,Y\}$ is a non-collider on $u$ that blocks $u$. 
        Otherwise, $V_1\in\Ch{X}$. 
        Continuing the path $u$ from $V_1$, let $V_j$ be the first collider on $u$. 
        Note that such a collider exists as $X\notin\Anc{Y}$ and therefore, $u$ is not a directed path. 
        $V_j$ is a descendant of $X$ and therefore, $V_j\notin\Anc{X,Y}$.
        Hence, $V_j\notin\Anc{\mathbf{W}\cup\{X,Y\}}$ blocks $u$ as a collider.
        \item $1\leq i<m$: The edge between $V_i$ and $V_{i+1}$ is either $V_i\to V_{i+1}$, or $V_i\gets V_{i+1}$ (it cannot be undirected by definition of MAGs.) 
        Let $Z$ be the parent among these two vertices, and $T$ be the child, i.e., if $V_i\to V_{i+1}$, then $Z$ and $T$ denote $V_i$ and $V_{i+1}$, respectively. 
        Note that $Z\in\PaP{X}$. 
        If $Z\in\Anc{\{X,Y\}}$, then $Z\in\mathbf{W} \setminusA\{X,Y\}$ blocks $u$ as a non-collider. 
        Suppose otherwise that $Z\notin\Anc{\{X,Y\}}$. 
        Continuing the path $u$ from $Z$ towards the side of $T$, let $V_j$ be the first collider. 
        Such a collider exists as $Z\notin\Anc{\{X,Y\}}$. 
        $V_j$ is a descendant of $Z$, and therefore $V_j\notin\Anc{\{X,Y\}}$. Hence, $V_j\notin\Anc{\mathbf{W}\cup\{X,Y\}}$ blocks $u$ as a collider.
    \end{enumerate}
    In all of the above cases, $\mathbf{W}\setminusA\{X,Y\}$ blocks $u$, which completes the proof.
\end{proof}

\begin{lemma}\label{lem: inclusion}
    If $X\in\mathbf{V}$ is a removable vertex, then for any $Y,Z\in\Mb{X}{}$,
    \[Z\in\Mb{Y}{}\text{ and }Y\in\Mb{Z}{}.\]
    Moreover, there exists at least one collider path between $Y$ and $Z$ that passes through only the vertices in $\Mb{X}{}\cup\{X\}$.
\end{lemma}
\begin{proof}
        Take two arbitrary vertices $Y,Z$ in $\Mb{X}{}$. 
        We will show that there exists a collider path between $Y$ and $Z$ that passes through only the vertices in $\Mb{X}{}\cup\{X\}$.
         
        Since $Y,Z\in\Mb{X}{}$, there exist collider paths $(Y,V_1,\dots,V_i,X)$ and $(X,W_1,\dots,W_j,Z)$, where $V_1,...,V_i,W_1,...,W_j\in\Mb{X}{}$. 
        Consider the path $(Y,V_1,\dots,V_i,X,W_1,\dots,W_j,Z)$. 
        If $X$ is a collider on this path, we are done. 
        Otherwise, without loss of generality, assume $W_1\in\Ch{X}\cup\N{X}$. Since $X$ is removable, $V_i\in \Adj{W_1}$. 
        We now consider the following two cases separately.
        \begin{enumerate}[leftmargin=*]
        \item $W_1 \in \Ch{X}$: If the edge between $V_i$ and $W_1$ is bidirected, then the path $(Y,V_1,\dots,V_i,W_1,\dots,W_j,Z)$ is a collider path. 
        Otherwise, again without loss of generality assume $V_i$ is a parent of $W_1$. 
        Note that a child of $X$ cannot be a parent of its spouse since this would create an almost directed cycle. 
        Now, since $X$ is removable, $V_{i-1}$ and $W_1$ are adjacent. 
        If the edge is bidirected, then $(Y,V_1,\dots,V_{i-1},W_1,\dots,W_j,Z)$ is a collider path. 
        Otherwise, we can continue the same argument as before by induction on $i$ and conclude that $Y$ is adjacent to $W_1$. 
        Since the structure is a MAG, $W_1 \notin \Pa{Y}$ and $W_1$ is a collider on $(Y,W_1,\dots,W_j,Z)$.
        Therefore, a collider path exists between $Y$ and $Z$ using only the vertices in $\Mb{X}{}\cup\{X\}$.
        
        \item $W_1\in\N{X}$: In this case, $W_1=Z$, since $W_1$ is not a collider. 
        Also, since $X$ has a neighbor, it cannot have a parent or a spouse. 
        As a result, $V_i\in\Ch{X}\cup\N{X}$. 
        If $V_i\in\N{X}$, then by the same argument, $V_i=Y$ and we already know that $Y$ and $Z$ are adjacent, which is the desired collider path. 
        Otherwise, $Z\in\Pa{V_i}$. 
        Now, the path $(Y,V_1,\dots,V_i,Z)$ is the desired path, which completes the proof.
        \end{enumerate}
    \end{proof}
\begin{lemma}\label{lem: removable to cond1,2}
    Suppose $\mathbf{V}\subseteq \mathbfcal{O}$ and let $\mathcal{G} = \GV[\mathbf{V}|\mathbfcal{S}]$.
    If $X\in\mathbf{V}$ is removable in $\mathcal{G}$, then for any $Y,Z\in\mathbf{V}\setminusA\{X\}$ and $\mathbf{W}\subseteq\mathbf{V}\setminusA\{X,Y,Z\}$,
    \[
        \msep{Y}{Z}{\mathbf{W}\cup \{X\}}{\mathcal{G}}
        \Longrightarrow
        \msep{Y}{Z}{\mathbf{W}}{\mathcal{G}}.
    \]
\end{lemma}
\begin{proof}
    Suppose $\msep{Y}{Z}{\mathbf{W}\cup \{X\}}{\mathcal{G}}$. 
    We need to show that $\msep{Y}{Z}{\mathbf{W}}{\mathcal{G}}$.
    To this end, we first show that $\msep{Y}{Z}{\mathbf{W}}{\mathcal{H}}$, where $\mathcal{H}$ is the induced subgraph of $\mathcal{G}$ over $\mathbf{V}\setminusA\{X\}$.
    
    Note that all the paths between $Y$ and $Z$ are blocked by $\mathbf{W}\cup\{X\}$ in $\mathcal{G}$.
    Now, take an arbitrary path $u$ between $Y$ and $Z$ in $\mathcal{H}$.
    This path also exists in $\mathcal{G}$, and $X$ is not on the path. 
    We claim $\mathbf{W}$ blocks it in $\mathcal{H}$.
    Suppose $u$ is blocked by a vertex $T$ in $\mathcal{G}$ (note that $T\neq X$.) 
    If $T$ is a non-collider on $u$, then it also blocks $u$ in $\mathcal{H}$. 
    If it is a collider with no descendants in $\mathbf{W}\cup\{X\}$, then lemma \ref{lem: preserveDescendants} implies that $\text{De}_{\mathcal{H}}(T) \cap \mathbf{W}=\varnothing$, and $T$ blocks $u$ in $\mathcal{H}$. 
    Therefore, $\msep{Y}{Z}{\mathbf{W}}{\mathcal{H}}$.
    
    Finally, since $X$ is removable in $\mathcal{G}$ and $\msep{Y}{Z}{\mathbf{W}}{\mathcal{H}}$, Definition \ref{def: removable} implies that $\msep{Y}{Z}{\mathbf{W}}{\mathcal{G}}$.
\end{proof}
\begin{lemma}\label{lem: cond1 not hold}
    Suppose $(X,V_1,...,V_m,Y)$ is a collider path where $\{X,V_1,...,V_m\}\in\Pa{Z}$ for a vertex $Z$. If $\msep{Y}{Z}{\mathbf{W}}{}$ for a set $\mathbf{W}$, then $X\in\mathbf{W}$.
\end{lemma}
\begin{proof}
     Since $Y$ and $Z$ are m-separated by $\mathbf{W}$, $\mathbf{W}$ blocks all the paths between $Y$ and $Z$. Now consider the path $u=(Z,X,V_1,...,V_m,Y)$ which must be blocked by $\mathbf{W}$. 
     $\{V_1,...,V_m\}\subseteq\Anc{Z}$ are colliders on $u$. 
     As a result, if $X\notin\mathbf{W}$, then $u$ is m-connecting relative to $\mathbf{W}$, which is a contradiction. 
\end{proof}

\begin{lemma} \label{lem: path block}
    Suppose $\mathcal{G}$ is a MAG with the vertex set $\mathbf{V}$, and $X\in\mathbf{V}$ is removable in $\mathcal{G}$. 
    Let $\mathcal{H}$ denote the induced subgraph of $\mathcal{G}$ over $\mathbf{V}\setminusA\{X\}$. 
    For a path $u$ in $\mathcal{H}$ and a set $\mathbf{W}\subseteq\mathbf{V}\setminusA\{X\}$, 
    \begin{equation} \label{eq: msep iff}
        u \text{ is m-connecting w.r.t. } \mathbf{W} \text{ in } \mathcal{M} 
        \iff
        u \text{ is m-connecting w.r.t. } \mathbf{W} \text{ in } \mathcal{H}.
    \end{equation}
\end{lemma}
\begin{proof}
    The proof of both sides of Equation \eqref{eq: msep iff} are the same. 
    Let $\mathcal{G}_1$ be $\mathcal{M}$ or $\mathcal{H}$, and $\mathcal{G}_2$ be the other one.
    Suppose $\mathbf{W}\subseteq\mathbf{V}\setminusA\{X\}$ and let $u=(Y,V_1,\dots,V_m,Z)$ be a path in $\mathcal{H}$ such that $u$ is m-connecting relative to $\mathbf{W}$ in $\mathcal{G}_1$.
    We need to show that $u$ is m-connecting relative to $\mathbf{W}$ in $\mathcal{G}_2$.
    Let $T$ be an arbitrary non-endpoint vertex on $u$.
    We need to show that $T$ does not block $u$ in $\mathcal{G}_2$.
    There are two possibilities. 
    \begin{enumerate}[leftmargin=*]
        \item $T$ is non-collider in $u$: 
            Since $T$ does not block $u$ in $\mathcal{G}_1$, $T\notin \mathbf{W}$.
            Hence, $T$ does not block $u$ in $\mathcal{G}_2$.
        \item $T$ is a collider on $u$:
            Since $T$ does not block $u$ in $\mathcal{G}_1$, $\textit{De}_{\mathcal{G}_1}(T) \cap \big(\mathbf{W}\cup \{Y,Z\} \big) \neq \varnothing$. 
            Hence, Lemma \ref{lem: preserveDescendants} implies that $\textit{De}_{\mathcal{G}_2}(T) \cap \big(\mathbf{W}\cup \{Y,Z\} \big) \neq \varnothing$ and $T$ does not block $u$ in $\mathcal{G}_2$.
    \end{enumerate}
    In both cases $T$ does not block $u$ in $\mathcal{G}_2$ and therefore, $u$ is m-connecting relative to $\mathbf{W}$ in $\mathcal{G}_2$.
\end{proof}
\begin{lemma}\label{lem: X non-collider V parent}
    Suppose $\mathcal{G}$ is a MAG and $u=(Y,...,V_0,V_1,X,V_2,...,Z)$ is a path in $\mathcal{G}$, where $X$ is a non-collider on $u$ and $V_1\in\Adj{V_2}$. 
    Define $\Tilde{u}=(Y,...,V_1,V_2,...,Z)$, which is a path in $\mathcal{G}$. If $V_1$ is a collider on $u$ and a non-collider on $\Tilde{u}$, or a non-collider on $u$ and a collider on $\Tilde{u}$, then $X,V_1\in\Pa{V_2}$.
\end{lemma}
\begin{proof}
    First note that the edge between $V_0$ and $V_1$ must have an arrowhead at $V_1$, since otherwise $V_1$ cannot be a collider on any of the paths.
    Now, two possibilities may occur.
    \begin{itemize}[leftmargin=*]
        \item The edge between $V_1$ and $X$ has a tail at $V_1$:
            Since $V_1$ has an arrowhead, it does not have any neighbors, i.e., $X\notin N(V_1)$. 
            Hence, $V_1\in\Pa{X}$. 
            As $X$ is not a collider on $u$, $X\in\Pa{V_2}$, i.e., $V_1\to X\to V_2$.
            Now, the edge between $V_1$ and $V_2$ can only be $V_1\to V_2$, as otherwise, an (almost) directed cycle is formed on $V_1,X,V_2$.
        \item The edge between $V_1$ and $X$ has an arrowhead at $V_1$:
            Since $V_1$ is a collider on $u$, it is a non-collider on $\Tilde{u}$.
            Also, $V_1$ does not have any neighbors by definition of MAGs, which implies that $V_1\in\Pa{V_2}$.
            Consider the edge between $X$ and $V_2$. 
            If this edge has an arrowhead at $X$, then $X\in\Pa{V_1}$ as $X$ is a non-collider on $u$. 
            Now, the triple $X,V_1,V_2$ forms an (almost) directed cycle, which is a contradiction.
            As a result, the edge between $X$ and $V_2$ has a tail at $X$.
            Note that $V_2$ has no neighbors because $V_1\to V_2$.
            This implies that $X\in\Pa{V_2}$, which competes the proof.
    \end{itemize}
\end{proof}
\subsection{Main Results}\label{sec: apd main results}
\begin{customprp} {\ref{prp: removable-subgraph}}
Suppose $\mathbf{V}\subseteq \mathbfcal{O}$ and $X\in \mathbf{V}$. $\GV[\mathbf{V}\setminusA \{X\} \vert \mathbfcal{S}]$ is equal to the induced subgraph of $\GV[\mathbf{V} \vert \mathbfcal{S}]$ over $\mathbf{V}\setminusA \{X\}$ if and only if $X$ is removable in $\GV[\mathbf{V} \vert \mathbfcal{S}]$.
\end{customprp}
\begin{proof}
    Denote $\GV[\mathbf{V} \vert \mathbfcal{S}]$, $\GV[\mathbf{V}\setminusA \{X\} \vert \mathbfcal{S}]$ and the induced subgraph of $\GV[\mathbf{V} \vert \mathbfcal{S}]$ over $\mathbf{V}\setminusA \{X\}$ by $\mathcal{G}$, $\mathcal{M}$ and $\mathcal{H}$, respectively.
    
    \textbf{only if:} Suppose $\mathcal{M}$ is equal to $\mathcal{H}$. Let $Y$ and $W$ be arbitrary vertices in $\mathbf{V}\setminusA\{X\}$ and $\mathbf{Z}$ be an arbitrary subset of $\mathbf{V}\setminusA\{X\}$. It suffices to show that Equation \eqref{eq: d-sepEquivalence} holds. 
    Since m-separation and conditional independence are equivalent in latent projections $\mathcal{G}$ and $\mathcal{M}$, 
    \[\msep{Y}{W}{\mathbf{Z}}{\mathcal{G}}\Leftrightarrow\CI{Y}{W}{\mathbf{Z}}{}\Leftrightarrow\msep{Y}{W}{\mathbf{Z}}{\mathcal{M}}\Leftrightarrow\msep{Y}{W}{\mathbf{Z}}{\mathcal{H}},\]
    where the last equivalence is due to the fact that $\mathcal{M}$ and $\mathcal{H}$ are equal.
    
    \textbf{if:} Suppose $X$ is removable. We first prove that the skeleton of $\mathcal{M}$ and $\mathcal{H}$ are equal. With similar arguments to the above case, CI relations and m-separation in $\mathcal{G}$ and $\mathcal{M}$ are equivalent. Therefore, 
    \[\msep{Y}{W}{\mathbf{Z}}{\mathcal{M}}\Leftrightarrow\CI{Y}{W}{\mathbf{Z}}{}\Leftrightarrow\msep{Y}{W}{\mathbf{Z}}{\mathcal{G}}\Leftrightarrow\msep{Y}{W}{\mathbf{Z}}{\mathcal{H}},\]
    where the last equivalence follows from Equation \eqref{eq: d-sepEquivalence}. Since $\mathcal{M}$ and $\mathcal{H}$ impose the same set of m-separations, that is they are Markov equivalent, they must have the same skeleton. Now for the edge marks, note that the edge marks of $\mathcal{H}$ are those of $\mathcal{G}$, as $\mathcal{H}$ is an induced subgraph of $\mathcal{G}$. Furthermore, edges in $\mathcal{G}$ and $\mathcal{M}$ are oriented by the same rules of Definition \ref{def: DAG to MAG} as they are the projections of the same DAG $\GV$. Therefore, both the skeleton and the edge marks of $\mathcal{M}$ and $\mathcal{H}$ are identical, which completes the proof.
\end{proof}

\begin{customthm}{\ref{thm: test removability}}
    Suppose the edge-induced subgraph of $\mathcal{M}$ over the undirected edges (i.e., the edges due to selection bias) is chordal. 
    Let $\mathcal{G} = \GV[\mathbf{V}|\mathbfcal{S}]$ for some $\mathbf{V}\subseteq \mathbfcal{O}$.
    $X\in\mathbf{V}$ is removable in $\mathcal{G}$ if and only if for every $Y\in \Adj{X}$ and $Z\in \Mb{X}{\mathbf{V}}$, at least one of the following holds.
    \begin{enumerate}
        \item[] \textbf{Condition 1:\:}
            $\exists \mathbf{W}\subseteq \Mb{X}{\mathbf{V}} \setminusA \{Y,Z\}\!:\: Y\independent Z\vert{\mathbf{W}}.$
        \item[] \textbf{Condition 2:\:}
            $\forall \mathbf{W}\subseteq \Mb{X}{\mathbf{V}} \setminusA \{Y,Z\}\!:\: Y\notindependent Z\vert{\mathbf{W}\cup \{X\}}.$
    \end{enumerate}
    Furthermore, the set of removable vertices in $\mathcal{G}$ is non-empty.
\end{customthm}
\begin{proof}
    We first prove the equivalence of removability and the two conditions.

    \textbf{only if:} Suppose $X$ is removable. 
    It suffices to show that if Condition 2 does not hold, then condition 1 holds. 
    Let $\mathbf{W_1}\subseteq\Mb{X}{\mathbf{V}}\setminusA\{Y,Z\}$ be such that $Y\independent Z\vert{\mathbf{W_1}\cup \{X\}}$. 
    Since m-separation is equivalent to conditional independence, $\msep{Y}{Z}{\mathbf{W_1}\cup \{X\}}{\mathcal{G}}$. 
    Now from lemma \ref{lem: removable to cond1,2}, $\msep{Y}{Z}{\mathbf{W_1}}{\mathcal{G}}$, which implies $\CI{Y}{Z}{\mathbf{W_1}}{}$, that is, Condition 1 holds.

    \textbf{if:} We show that the graphical representation of Theorem \ref{thm: graph-rep} is satisfied.
    To this end, we show $Y$ and $Z$ are adjacent in all of the following cases:
    \begin{enumerate}
        \item $u=(X,V_1,...,V_m,Y)$ is a collider path such that $\{X,V_1,...,V_m\}\subseteq \Pa{Z}$: 
        By definition of $\PaP{\cdot}$, $\PaP{Z}\subseteq\Mb{X}{\mathbf{V}}\cup\{X\}$.
        Lemma \ref{lem: parent-plus} indicates that \[\mathbf{W_1}=(\PaP{Z}\cap\Anc{\{Z,Y\}}\setminusA\{Z,Y\})\subseteq\Mb{X}{\mathbf{V}}\cup\{X\}\]
        m-separates $Y$ and $Z$. 
        Note that $\N{Z}=\varnothing$ since $Z$ has at least one parent.
        Since conditional independence is equivalent to m-separation, 
        \[\CI{Y}{Z}{\mathbf{W_1}}{},\]
        that is, Condition 2 does not hold. 
        If $Y$ and $Z$ are m-separated by some set $\mathbf{W_1}$, from Lemma \ref{lem: cond1 not hold}, $X\in\mathbf{W_1}$. 
        As a result, Condition 1 cannot hold for any $\mathbf{W}\subseteq\Mb{X}{\mathbf{V}}$ as these sets do not contain $X$, which is a contradiction.
        This proves that $Y$ and $Z$ are adjacent. 
        \item
        $Y\in\Adj{X}$ and $Z\in\Ch{X}$:
        The proof in this case is exactly the same as the previous one.
        \item
        $Z\in\N{X}$ and $Y\in\Adj{X}$:
        Since $X$ has a neighbor, by definition of MAG, $Y$ is either a child or a neighbor of $X$.
        If $Y\in\Ch{X}$, this case reduces to case 2 with $Y$ and $Z$ interchanged.
        So we only consider the case where $Y\in\N{X}$.
        Considering the path $(Y,X,Z)$, no set $\mathbf{W}$ can m-separate $Y$ and $Z$ if $X\notin\mathbf{W}$, i.e., Condition 1 does not hold. 
        We claim if $Y$ and $Z$ are not adjacent, Condition 2 does not hold either, which is a contradiction.
        To prove this, take $\mathbf{W}=\{X\}\cup N(X)\setminusA\{Y,Z\}$. 
        It is enough to show that $\msep{Y}{Z}{\mathbf{W}}{\mathcal{G}}$, i.e., $\mathbf{W}$ blocks all the paths between $Y$ and $Z$.
        Let $u$ be an arbitrary path of length at least 2 between $Y$ and $Z$.
        If $u$ contains a directed or bidirected edge, it also contains a collider, since $Y$ and $Z$ do not have any incoming edges incident to them and therefore no ancestors. 
        This collider blocks the path as it does not have any descendants in $\mathbf{W}$ (note that the vertices in $\mathbf{W}$ have at least one neighbor, and therefore by definition of MAG, they do not have any ancestors.) 
        Otherwise, $u$ is a path with only undirected edges.
        If $X$ is on $u$, $X$ itself blocks this path. 
        Otherwise, consider the cycle formed by adding the path $Y-X-Z$ to $u$. 
        Since the edge-induced subgraph of $\mathcal{M}=\GV[\mathbfcal{O}\vert \mathbfcal{S}]$ over its undirected edges is chordal, if $Y$ and $Z$ are not adjacent, there exists a chord which connects $X$ to a non-endpoint vertex on $u$.
        As a result, at least one of the neighbors of $X$ appears on $u$, and therefore blocks $u$ as a non-collider, as it belongs to $\mathbf{W}$.
    \end{enumerate}
    
    For a proof of the second part of the theorem, i.e., the set of removable vertices is non-empty, we refer the reader to Lemma \ref{lem: rem exists} in Appendix \ref{sec: apd chordal}.
\end{proof}
\begin{customthm}{\ref{thm: sound and complete}}
    Suppose the distribution $\PV$ over $\V = \mathbfcal{O}\cup\mathbfcal{L}\cup\mathbfcal{S}$ is faithful to the DAG $\GV$. If the conditional independence relations among all variables in $\mathbfcal{O}$ given $\mathbfcal{S}$ is provided to L-MARVEL, the output of L-MARVEL is the PAG representing the Markov equivalence class of $\GV[\mathbfcal{O}\vert\mathbfcal{S}]$.
\end{customthm}
\begin{proof}
    In order to prove this theorem, it is enough to show that the information stored in $\A$, i.e., the set of adjacencies and the separating sets for non-adjacent variables, is correct.
    L-MARVEL identifies that two variables are not adjacent, only if it finds a separating set for them.
    In this case, L-MARVEL adds that separating set to $\A$.
    Hence, all the separating sets found in $\A$ are correct, and the non-adjacent variables in $\A$ are non-adjacent in $\mathcal{M}$. 
    Note that even in the case that two variables are excluded from each other's Markov boundary, this is due to a found separating set for these two variables.
    It is left to show that L-MARVEL correctly finds all the adjacent variables in $\mathcal{M}$.
    
    Let $\mathcal{H}_\mathbf{V}$ denote the induced subgraph of $\mathcal{M}$ over $\mathbf{V}\subseteq \mathbfcal{O}$.
    We claim every time that L-MARVEL is called over a subset $\mathbf{V}\subseteq \mathbfcal{O}$ during the execution of the algorithm, $\mathcal{H}_\mathbf{V}$ is equal to $\GV[\mathbf{V}|\mathbfcal{S}]$.
    For the first time, we call L-MARVEL over $\mathbfcal{O}$ and the claim holds.
    Now, assume $\mathcal{H}_\mathbf{V} = \GV[\mathbf{V}|\mathbfcal{S}]$ in a recursion. 
    We need to show that our claim holds for the next recursion.
    First, note that Equation \eqref{eq: msep iff dsep} implies that $\mathcal{H}_\mathbf{V}$ satisfies faithfulness with respect to $P_{\mathbf{V}|\mathbfcal{S}}$. Theorem \ref{thm: test removability} implies that when the if condition in line 9 holds for the first $i=i^*$, then $X_{i^*}$ is removable in $\mathcal{H}_\mathbf{V}$. Note that by Lemma \ref{lem: rem exists}, there always exists a variable that satisfies the if condition in line 9.
    Hence, Proposition \ref{prp: removable-subgraph} implies that in the next recursion, $\mathcal{H}_{\mathbf{V}\setminus X_{i^*}} = \GV[\mathbf{V} \setminus X_{i^*} | \mathbfcal{S}]$, which proves our claim.
    
    So far we have shown that in each recursion, $\mathcal{H}_\mathbf{V} = \GV[\mathbf{V}|\mathbfcal{S}]$ and $\mathcal{H}_\mathbf{V}$ satisfies faithfulness with respect to $P_{\mathbf{V}|\mathbfcal{S}}$.
    Hence, Function \textbf{FindAdjacent} and $\textbf{UpdateMb}$ correctly learn the adjacent variables and update the Markov boundaries, respectively.
    Hence, L-MARVEL manages to terminate after $n$ recursion and correctly add all the edges of $\mathcal{M}$ to $\A$.
\end{proof}
\begin{customprp}{\ref{prp: mb bound}}
    If $X$ is a removable variable in MAG $\mathcal{H}$ with vertices $\mathbf{V}$, then $\vert\Mb{X}{\mathbf{V}}\vert\leq\deltaplus{\mathcal{H}}$.
\end{customprp}
\begin{proof}
    Consider the set of variables $\Mb{X}{}\cup\{X\}$. Since MAGs are acyclic, there exists a vertex in this set such that it has no children in $\Mb{X}{}\cup\{X\}$. Denote this vertex by $Z$. From Lemma \ref{lem: inclusion}, every vertex in $\{X\}\cup\Mb{X}{}\setminusA\{Z\}$ has a collider path to $Z$ such that it passes through only the vertices in $\{X\}\cup\Mb{X}{}$. Since $Z$ has no child in this set, the vertex adjacent to $Z$ on these collider paths is either a parent, or a spouse, or a neighbor of $Z$. Therefore, by definition,
     \[\{X\}\cup\Mb{X}{}\setminusA\{Z\}\subseteq \PaP{Z}.\]
     As a result, 
     \[\left\vert\Mb{X}{}\right\vert=\left\vert\{X\}\cup\Mb{X}{}\setminusA\{Z\}\right\vert\leq\left\vert \PaP{Z}\right\vert\leq\deltaplus{\mathcal{H}}.\]
\end{proof}
\begin{customprp}{\ref{prp: upper-bound}}
    The number of conditional independence tests Algorithm \ref{alg: L-MARVEL} performs on a MAG $\mathcal{M}$ of order $n$, in the worst case, is upper bounded by
    \begin{equation}\label{eq: upper bound proof}
        \mathcal{O}(n^2 + n{\deltaplus{\mathcal{M}}}^2 2^{\deltaplus{\mathcal{M}}}).
    \end{equation}
\end{customprp}
\begin{proof}
    Algorithm \ref{alg: L-MARVEL} performs CI tests throughout the following subroutines:
    \begin{itemize}[leftmargin=*]
        \item ComputeMb: This is the initial Markov boundary discovery, that can be performed using any of the existing quadratic algorithms such as GS, TC, IAMB, etc. as discussed in the main text, that is, $\mathcal{O}(n)$ CI tests are required for this task.
	    \item \textbf{FindAdjacent($X$):}
	        The performed CI tests are of the type $\CI{X}{Y}{\mathbf{W}}{}$, where $Y\in\Mb{X}{\mathbf{V}}$ and $\mathbf{W}\subseteq\Mb{X}{\mathbf{V}}\setminusA\{Y\}$. There are $\left\vert\Mb{X}{\mathbf{V}}\right\vert$ choices for $Y$ and $2^{(\left\vert\Mb{X}{\mathbf{V}}\right\vert-1)}$ choices for $\mathbf{W}$, that is, $\left\vert\Mb{X}{\mathbf{V}}\right\vert2^{(\left\vert\Mb{X}{\mathbf{V}}\right\vert-1)}$ total tests.
	    \item \textbf{IsRemovable($X$):} The performed CI tests are of the type $\CI{Y}{Z}{\mathbf{W}}{}$, where $Y\in\Adj{X}\cap\mathbf{V}$, $Z\in\Mb{X}{\mathbf{V}}\setminusA\{Y\}$ and $\mathbf{W}\subseteq\{X\}\cup\Mb{X}{\mathbf{V}}\setminusA\{Y,Z\}$. There are $\left\vert N(X)\right\vert$ choices for $Y$, at most $\left\vert\Mb{X}{\mathbf{V}}\right\vert$ choices for $Z$ and $2^{(\left\vert\Mb{X}{\mathbf{V}}\right\vert-1)}$ choices for $\mathbf{W}$, that is, at most $\left\vert\Mb{X}{\mathbf{V}}\right\vert\left\vert N(X)\right\vert2^{(\left\vert\Mb{X}{\mathbf{V}}\right\vert-1)}$ total tests.
	        
	    \item \textbf{UpdateMb($X$):} L-MARVEL performs a single CI test for any pair of vertices in $\Mb{X}{\mathbf{V}}$, that is $\binom{\left\vert\Mb{X}{\mathbf{V}} \right\vert}{2}$ tests.
    \end{itemize}
    Note that due to Proposition \ref{prp: mb bound}, the for loop in line 6 of Algorithm \ref{alg: L-MARVEL} only reaches vertices with maximum Markov boundary size of $\deltaplus{\mathcal{M}}$. Therefore, the number of CI tests performed for a single vertex $X$ is upper bounded by $\mathcal{O}({\deltaplus{\mathcal{M}}}^2 2^{\deltaplus{\mathcal{M}}})$.
    We shall next discuss why we do not need to perform each of the aforementioned tests more than once, which the yields the desired upper bound.
    
    \begin{itemize}[leftmargin=*]
	    \item \textbf{FindAdjacent($X$):}
	        The set of vertices adjacent to $X$ does not change throughout the algorithm. 
	        Therefore, the first time that \textbf{FindAdjacent} is called for $X$, the variables adjacent to $X$ are identified and saved in $\mathcal{A}$, and are used in later iterations without requiring further CI tests.
	    \item \textbf{IsRemovable($X$):} It might happen that L-MARVEL performs some CI tests to identify that $X$ is not removable, and therefore, it has to call \textbf{IsRemovable} for $X$ in a later iteration (note that every variable gets removed throughout the algorithm.) 
	    This is due to the fact that the removal of other variables can render $X$ removable in a later iteration. 
	    However, we claim that no duplicate CI tests are needed in later iterations where L-MARVEL calls \textbf{IsRemovable}.
	    To show this, note that for any pair $Y,Z$ where $Y\in\Adj{X}\cap\mathbf{V}$ and $Z\in\Mb{X}{\mathbf{V}}\setminusA\{Y\}$, all of the separating sets of $Y$ and $Z$ in $\Mb{X}{\mathbf{V}}\cup\{X\}$ are saved in $\mathcal{A}$ during the first call to \textbf{IsRemovable}. 
	    Since the Markov boundary of $X$ can only be reduced throughout the algorithm, in all the succeeding iterations, it suffices for L-MARVEL to query the found separating sets.

	    \item \textbf{UpdateMb($X$):} These CI tests are performed only before $X$ is removed from the set of variables, that is, they are performed exactly once for each variable.
    \end{itemize}
\end{proof}
\begin{customthm}{\ref{thm: lwrBound}}
	The number of conditional independence tests of the form $\CI{X}{Y}{\mathbf{Z}}{}$ required by any constraint-based algorithm on a MAG ${\mathcal{M}}$ of order $n$, in the worst case, is lower bounded by
	\begin{equation} \label{eq: lwrbound proof}
	    \Omega(n^2+n{\deltaplus{\mathcal{M}}}2^{\deltaplus{\mathcal{M}}}).
	\end{equation}
\end{customthm}
\begin{proof}
    First, suppose an algorithm does not query any CI test of the form $\CI{X}{Y}{\mathbf{W}}{}$ for a pair of vertices $(X,Y)$. 
    If all the queried CI tests yield independence, this algorithm cannot tell an empty graph and a graph where only $X$ and $Y$ are adjacent apart.
    Therefore, at least one CI test is required for any pair of vertices, which yields a lower bound of $\binom{n}{2}$.
    
    Furthermore, \cite{mokhtarian2020recursive} proposed a lower bound of the form $\Omega(n{\Delta_{in}(\mathcal{M}})2^{\Delta_{in}(\mathcal{M}}))$ for the case that $\mathcal{M}$ is a DAG, where $\Delta_{in}(\mathcal{M})$ is the maximum number of parents among the variables. 
    Note that in the case of a DAG, $\deltaplus{\mathcal{M}}=\Delta_{in}(\mathcal{M})$, which proves our claim.
    However, we briefly discuss how their worst-case example can be modified in a way that it is no longer a DAG, and also $\deltaplus{\mathcal{M}}$ is strictly larger than $\Delta_{in}(\mathcal{M})$.
    The provided example is as follows.
    The vertices of the ground truth graph is partitioned into $\frac{n}{\deltaplus{\mathcal{M}}+1}$ clusters, where each cluster is a complete graph and there is no edge between the variables of different clusters.
    They show that if fewer CI tests than the claimed lower bound are performed, then a CI test of the form $\CI{X}{Y}{\mathbf{W_1\cup W_2}}{}$ is not queried, where $X,Y,\mathbf{W_1}$ belong to a cluster $\mathbf{C}$, whereas $\mathbf{W_2}$ does not contain any vertex of $\mathbf{C}$.
    Then they show that the graph where $\mathbf{W_1}$ are parents of $X$ and $Y$, and the rest of the graph is exactly the same as $\mathcal{M}$ with the exception that there is no edge between $X$ and $Y$ is consistent with the performed CI tests.
    In this example, if the rest of the edges in the cluster $\mathbf{C}$, i.e., the edges other than those between $\mathbf{W}$ and $X,Y$, as well as all the edges in the other clusters are replaced by bidirectional edges, the same proof still works.
    Note that in this example, $\deltaplus{\mathcal{M}}=|\mathbf{C}|-1$, whereas $\Delta_{in}=|\mathbf{W_1}|$.
    Hence, we achieve the lower bound of Equation \eqref{eq: lwrbound proof}.
\end{proof}

\section{Additional experiments} \label{sec: apd experiments}
In this section, we provide further experimental results to assess the performance of L-MARVEL against the state of the art.

Figure \ref{fig: ins alarm} illustrates the effect of the sample size on the performance of various algorithms. 
It is seen that L-MARVEL has the lowest run time and the fewest number of performed CI tests, while it maintains high accuracy in the wide range of the sample size. 
Also note that on these benchmark structures, L-MARVEL beats RFCI in terms of the average number of CI tests, which was the only metric in which RFCI showed advantage on random graphs. 
The experimental setting in this part is exactly that of Table \ref{table: exp1}, except for the sample size, to observe only the effect of the sample size. 
Each point of these graphs represents 50 MAGs generated by selecting the latent and selection variables uniformly at random.

Table \ref{table: exp appendix} extends our experiments to two new benchmark structures, namely mildew and water. The number of latent and selection variables varies in different columns of this table, where the latent and selection variables are chosen uniformly at random.
The coefficients of the linear SEM are chosen uniformly at random from the interval $\pm(1,1.5)$, whereas the standard deviation of the noise variables is chosen uniformly at random from the interval $(1,\sqrt{2})$ to represent a set of parameters different than that of the main text.
The entries of the table represent an average of 20 runs. 
As observed in Table \ref{table: exp1}, L-MARVEL outperforms all the other algorithms in almost every comparison metric, except for the precision, where it still is competent to the state of the art.
\begin{figure*}[t] 
    \centering
    \begin{subfigure}[b]{1\textwidth}
        \centering
        \includegraphics[width=0.3\textwidth]{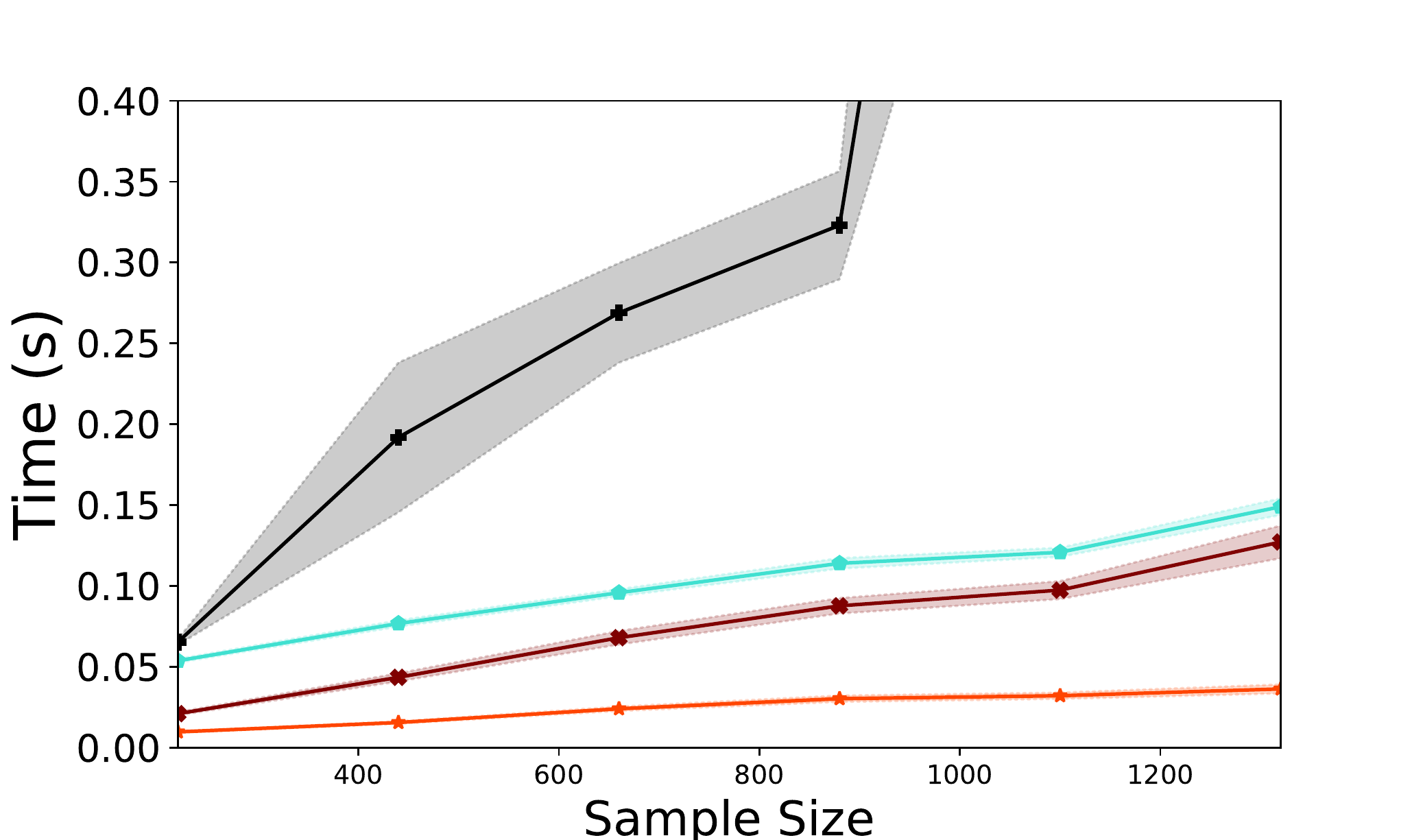}
        \hfill
        \includegraphics[width=0.3\textwidth]{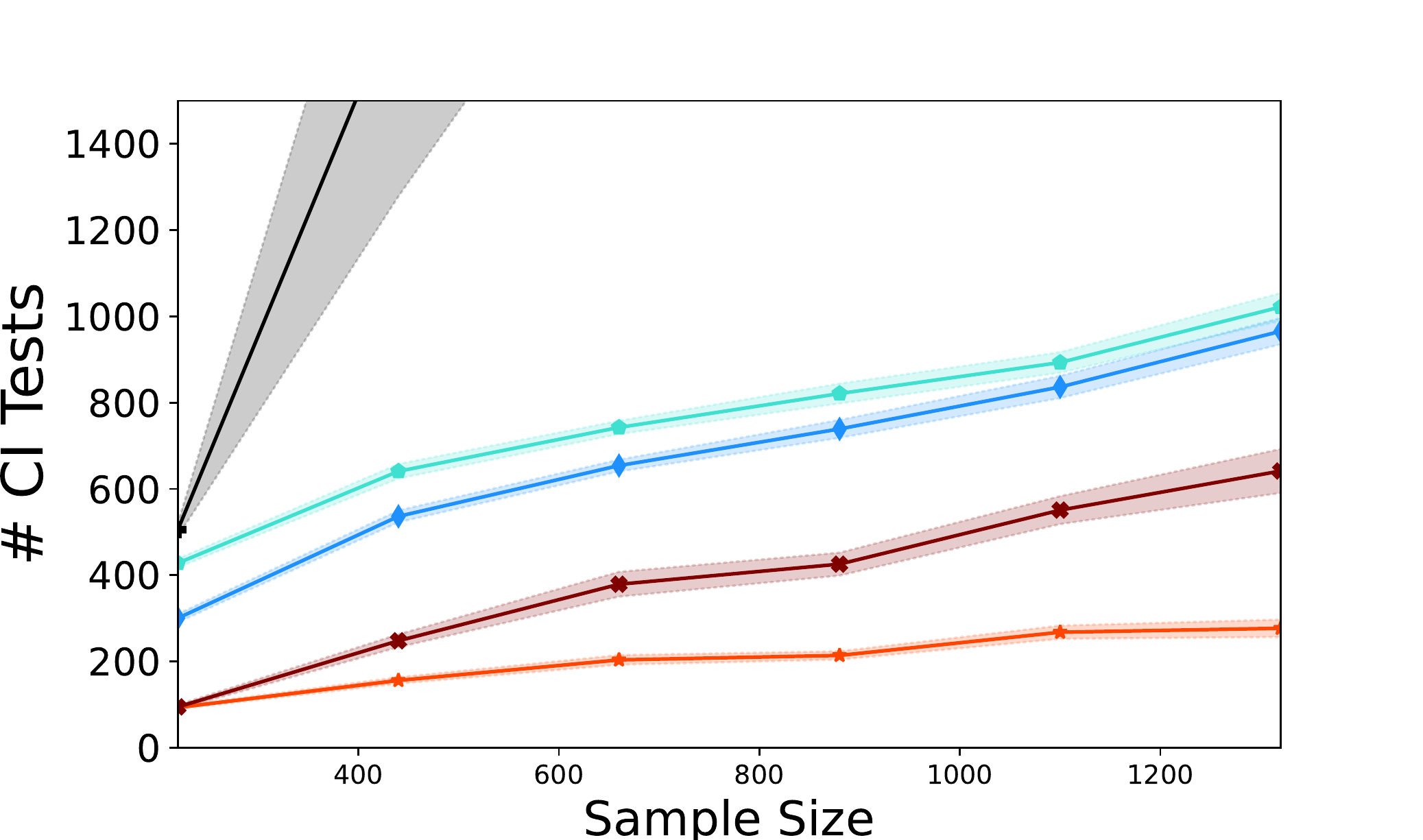}
        \hfill
        \includegraphics[width=0.3\textwidth]{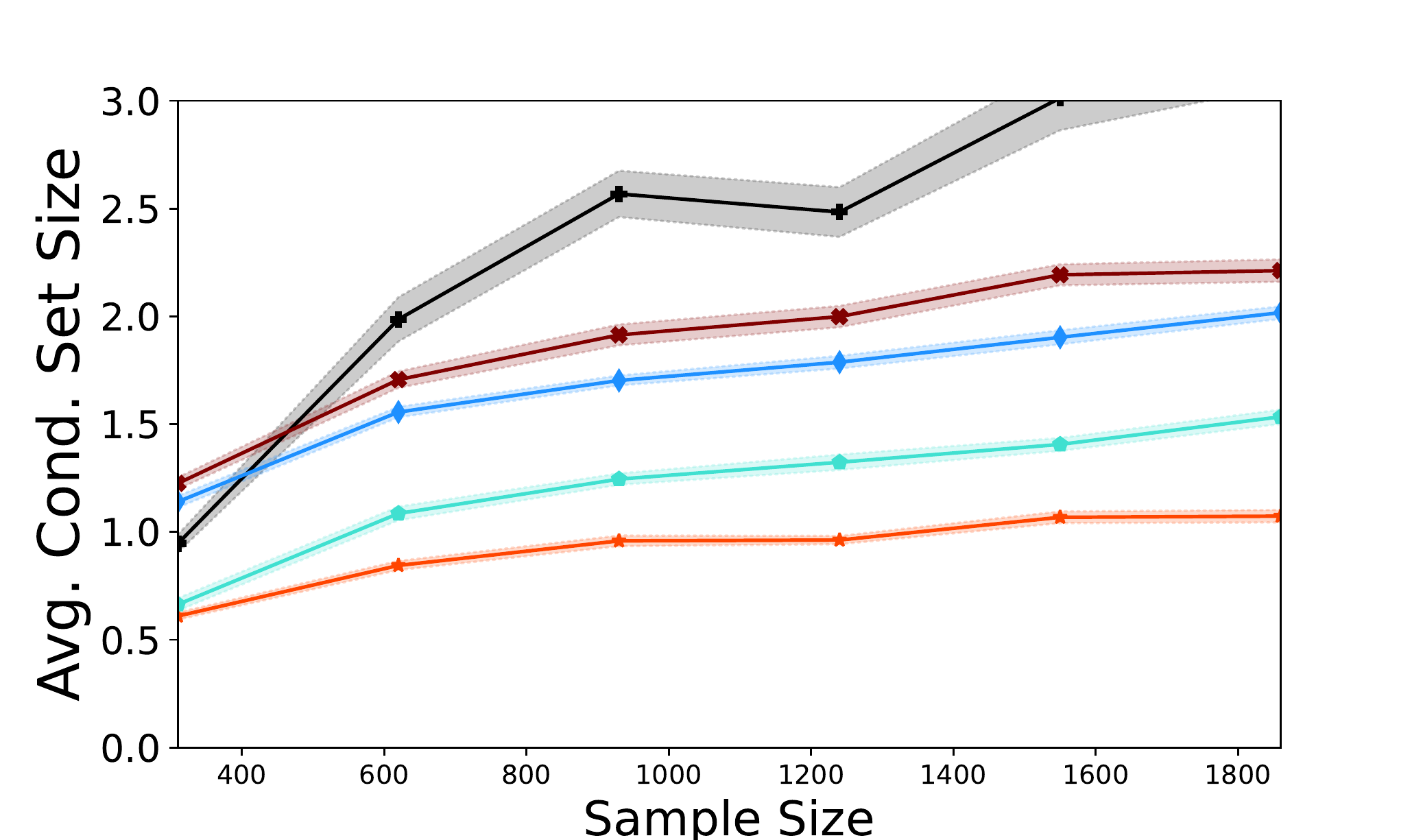}\hfill
        \caption{Performance (run time, number of CI tests, and the average conditioning size) of structure learning algorithms on the insurance network.}
        \label{fig: insurance time}
    \end{subfigure}
    \hfill
    \begin{subfigure}[b]{1\textwidth}
        \centering
        \includegraphics[width=0.3\textwidth]{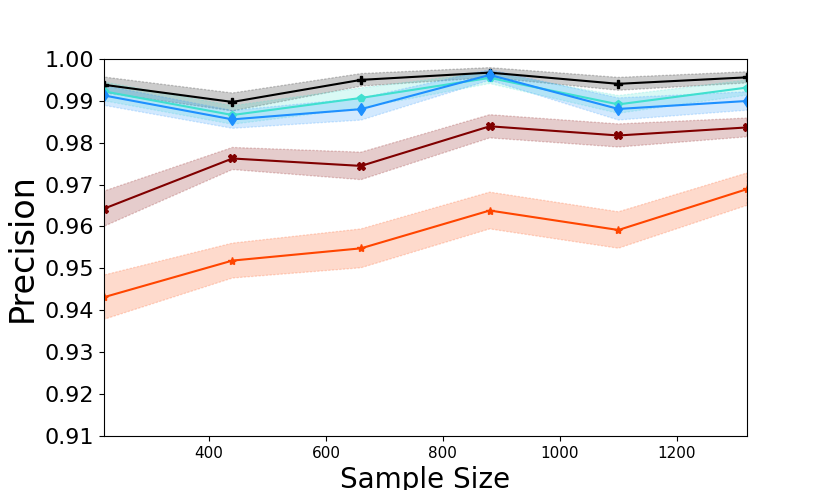}
        \hfill
        \includegraphics[width=0.3\textwidth]{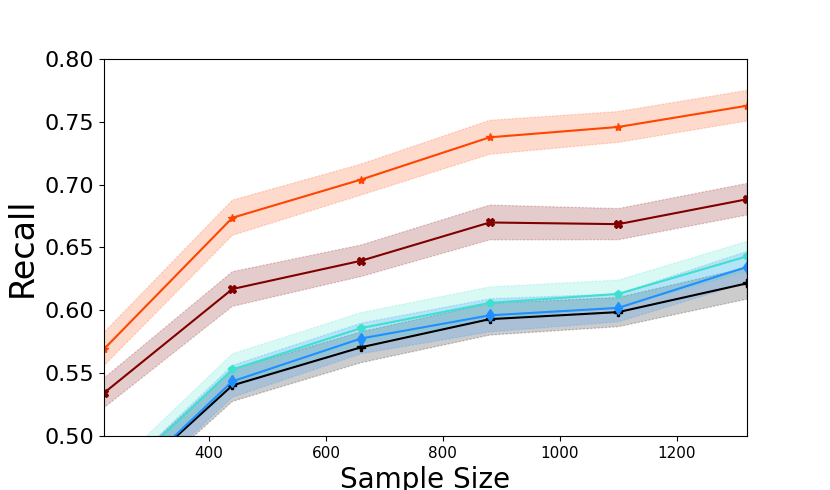}
        \hfill
        \includegraphics[width=0.3\textwidth]{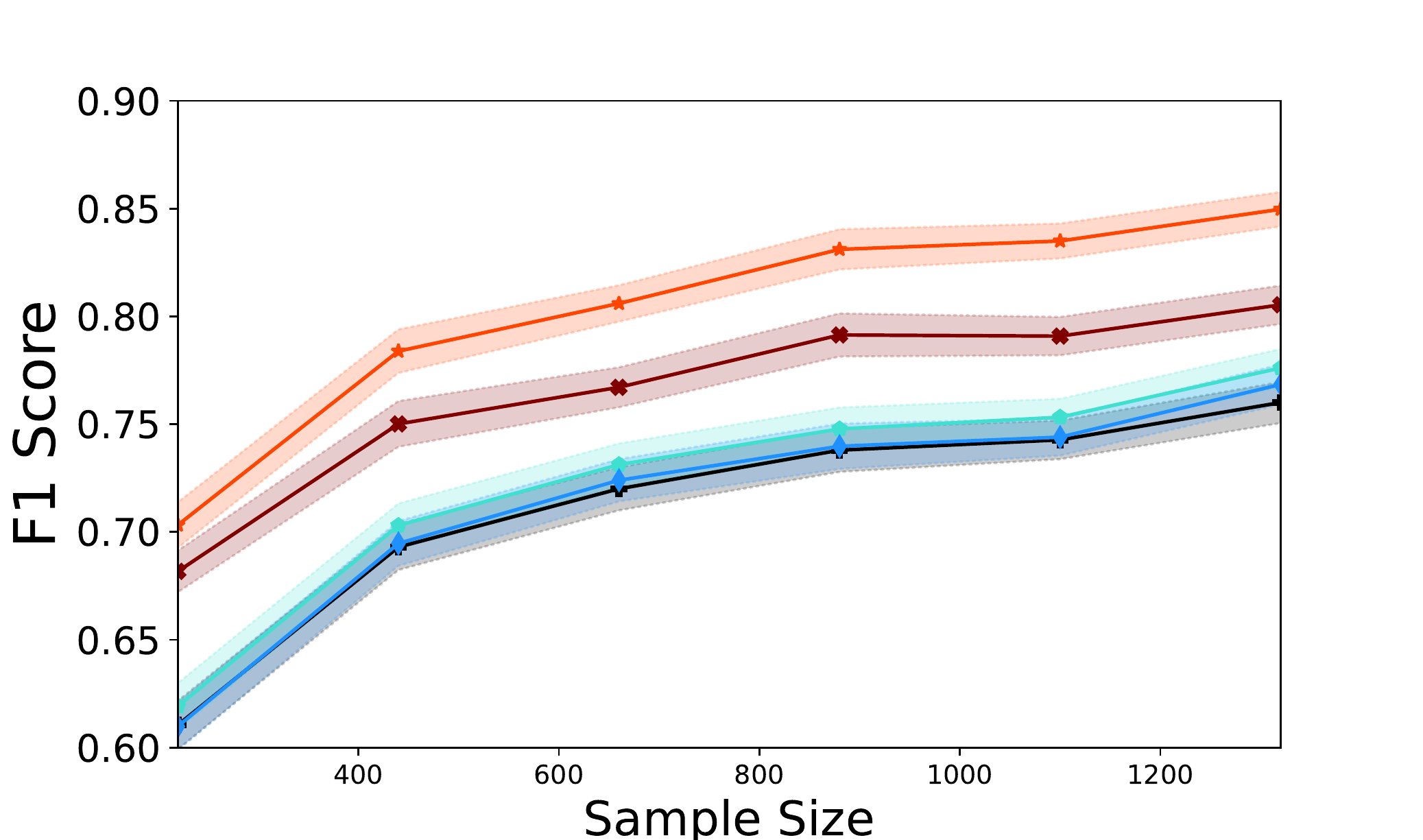}\hfill
        \caption{Performance (Precision, Recall, and F1 score) of structure learning algorithms on the insurance network.}
        \label{fig: insurance F1}
    \end{subfigure}
    \hfill
    \begin{subfigure}[b]{1\textwidth}
        \centering
        \includegraphics[width=0.3\textwidth]{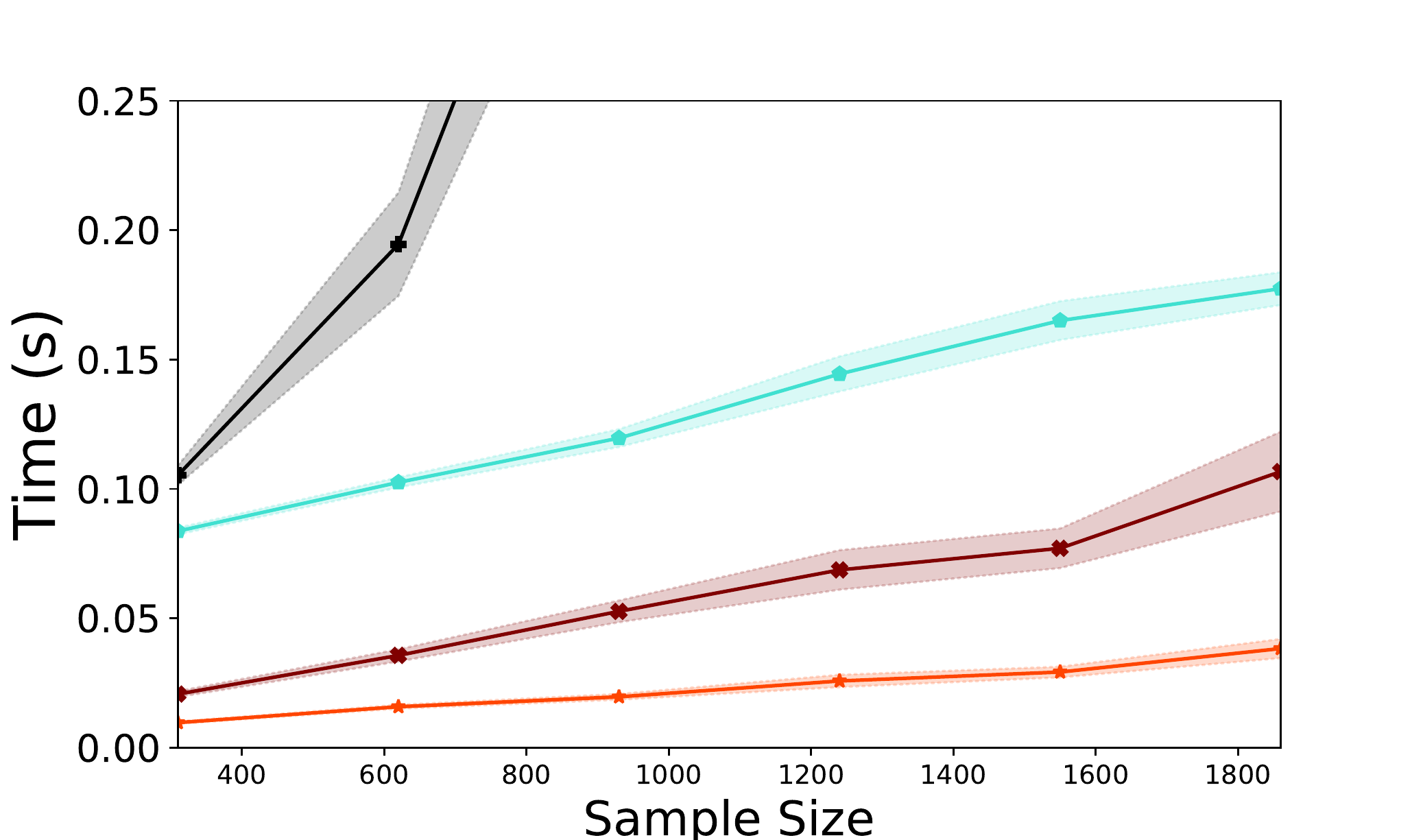}
        \hfill
        \includegraphics[width=0.3\textwidth]{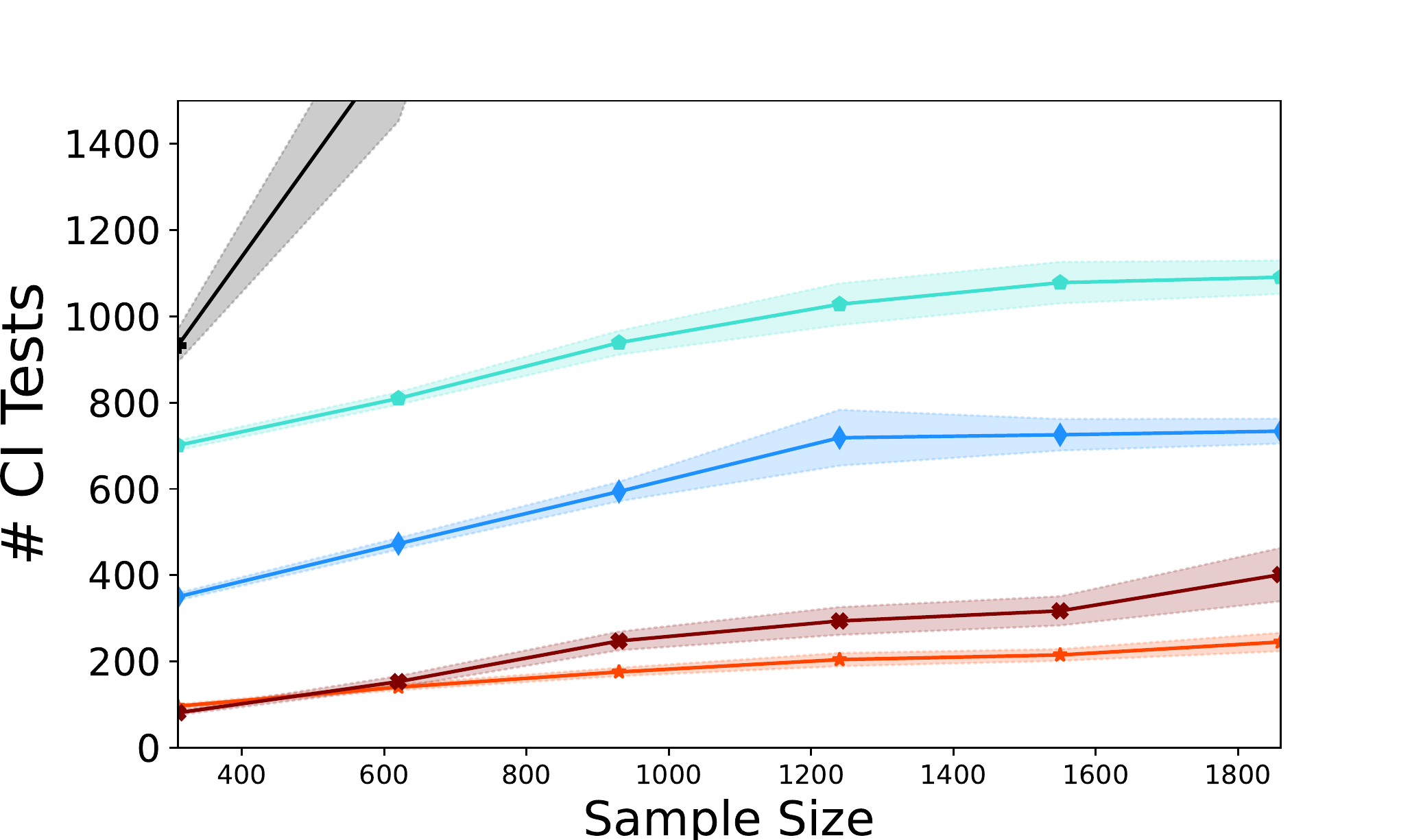}
        \hfill
        \includegraphics[width=0.3\textwidth]{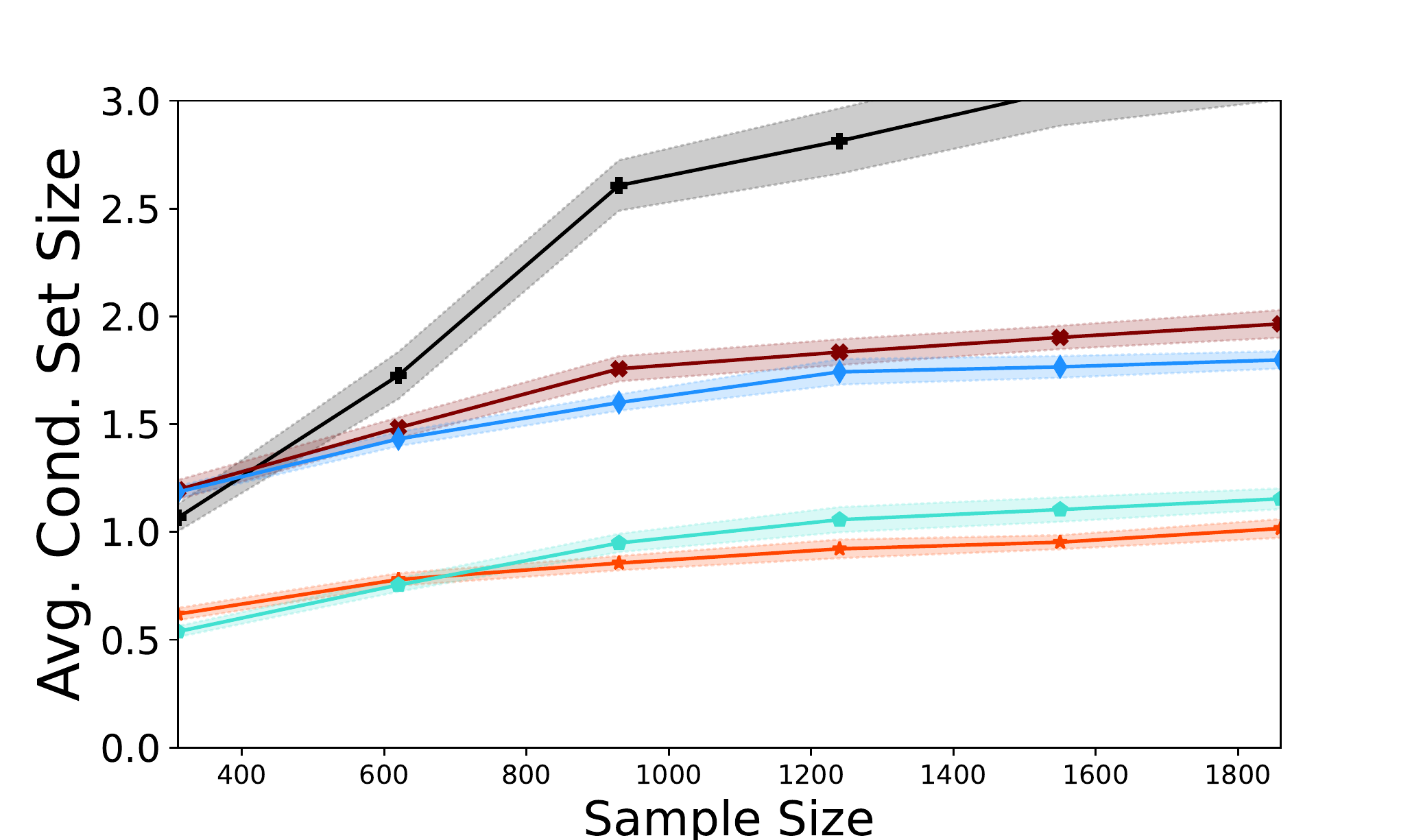}\hfill
        \caption{Performance (run time, number of CI tests, and the average conditioning size) of structure learning algorithms on the alarm network.}
        \label{fig: alarm time}
    \end{subfigure}
    \hfill
    \begin{subfigure}[b]{1\textwidth}
        \centering
        \includegraphics[width=0.3\textwidth]{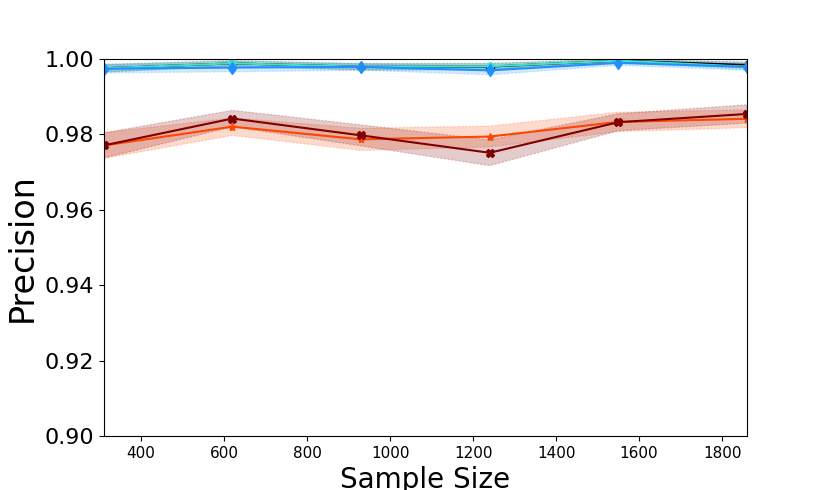}
        \hfill
        \includegraphics[width=0.3\textwidth]{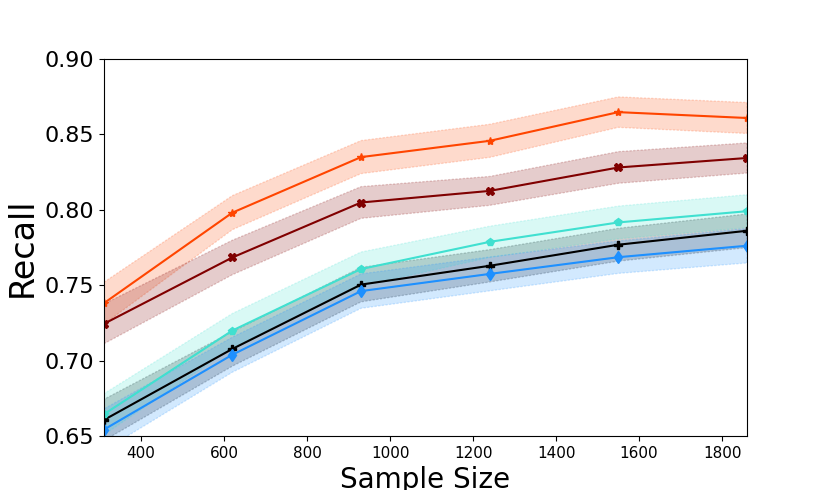}
        \hfill
        \includegraphics[width=0.3\textwidth]{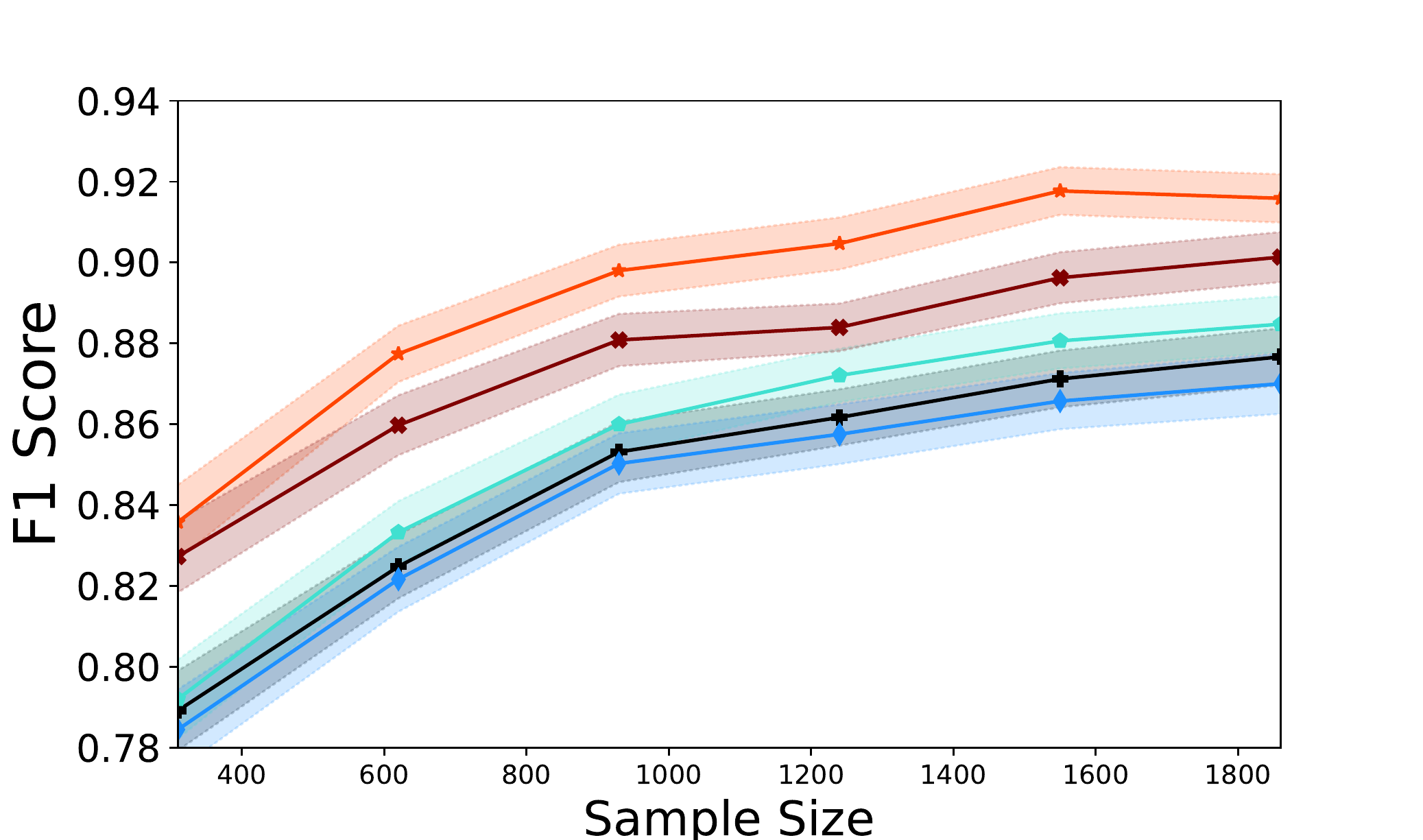}\hfill
        \caption{Performance (Precision, Recall, and F1 score) of structure learning algorithms on the alarm network.}
        \label{fig: alarm F1}
    \end{subfigure}
    \hfill
    \begin{subfigure}[b]{0.8\textwidth}
        \centering
         \includegraphics[width=\textwidth]{figures/legend.pdf}
    \end{subfigure}
    \caption{Effect of the sample size on the performance of structure learning algorithms on two benchmark structures, where the sample size varies from $=10|\mathbfcal{O}|$ to $=60|\mathbfcal{O}|$. The parameters of the experiments are preserved as in Table \ref{table: exp1}, except for the sample size.}
    \label{fig: ins alarm}
\end{figure*}
    
    \begin{table*}[ht]
	    \caption{Performance of various algorithms on the benchmark structures, when sample size $=50|\mathbfcal{O}|$.}
	    \fontsize{9}{10.5}\selectfont
	    \centering
	    \begin{tabular}{N M{0.3cm}|M{1.2cm}|M{1.2cm} M{1.1cm} M{1.1cm} M{1.1cm} M{1.1cm} M{1.1cm}}
    		\toprule
    		&\multicolumn{2}{c|}{Structure}
 			& Mildew
 			& Mildew
 			& Mildew
 			& Water
 			& Water
 			& Water
			\\
			&\multicolumn{2}{c|}{$(|\mathbfcal{O}|, |\mathbfcal{L}|, |\mathbfcal{S}|)$}
 			& (31,4,0)
 			& (31,0,4)
 			& (29,3,3)
 			& (29,3,0)
 			& (29,0,3)
 			& (26,3,3)
			\\
			\hline
			& \multirow{5}{*}{\rotatebox[origin=c]{90}{L-MARVEL}}
			& \#CI tests
            & \textbf{359} & \textbf{194} & \textbf{426} & 2365 & \textbf{1130} & 1368
			\\
			& 
			& Runtime
			& \textbf{0.06} & \textbf{0.04} & \textbf{0.08} & \textbf{0.32} & 0.21 & 0.25
			\\
			& 
			& F1-score
			& \textbf{0.90} & \textbf{0.92} & \textbf{0.89} & \textbf{0.82} & \textbf{0.87} & \textbf{0.73}
			\\
			&
			& Precision
			& 0.95 & 0.99 & 0.96 & 0.97 & 0.98 & 0.95
			\\
			&
			& Recall
			& \textbf{0.85} & \textbf{0.87} & \textbf{0.83} & \textbf{0.72} & \textbf{0.79} & \textbf{0.60}
			\\
			\hline
			& \multirow{5}{*}{\rotatebox[origin=c]{90}{RFCI}}
			& \#CI tests
			& 896 & 1085 & 937 & 1472 & 1398 & 1173
			\\
			& 
			& Runtime
			& 0.20 & 0.23 & 0.19 & 0.21 & \textbf{0.29} & \textbf{0.22}
			\\
			& 
			& F1-score
			& 0.77 & 0.84 & 0.79 & 0.67 & 0.69 & 0.60
			\\
			&
			& Precision
			& \textbf{0.98} & 1.00 & 0.99 & 0.97 & 0.98 & 0.97
			\\
			&
			& Recall
			& 0.64 & 0.73 & 0.66 & 0.51 & 0.53 & 0.44
			\\
			\hline
			& \multirow{5}{*}{\rotatebox[origin=c]{90}{FCI}}
			& \#CI tests
			& 1751 & 7251 & 10999 & 149674 & 12912 & 78903
			\\
			& 
			& Runtime
			& 0.33 & 1.57 & 2.26 & 29.99 & 2.99 & 19.00
			\\
			& 
			& F1-score
			& 0.72 & 0.81 & 0.74 & 0.57 & 0.61 & 0.50
			\\
			&
			& Precision
			& \textbf{0.98} & \textbf{1.00} & \textbf{1.00} & 0.98 & 0.98 & 0.98
			\\
			&
			& Recall
			& 0.57 & 0.69 & 0.59 & 0.41 & 0.45 & 0.34
			\\
			\hline
			& \multirow{5}{*}{\rotatebox[origin=c]{90}{MBCS*}}
			& \#CI tests
			& 1076 & 336 & 816 & 8300 & 3927 & 3946
			\\
			& 
			& Runtime
			& 0.28 & 0.12 & 0.25 & 1.98 & 1.07 & 1.12
			\\
			& 
			& F1-score
			& 0.81 & 0.89 & 0.82 & 0.68 & 0.74 & 0.61
			\\
			&
			& Precision
			& 0.97 & 0.99 & 0.98 & \textbf{1.00} & \textbf{0.99} & \textbf{0.99}
			\\
			&
			& Recall
			& 0.70 & 0.81 & 0.71 & 0.52 & 0.59 & 0.45
			\\
			\hline
			& \multirow{5}{*}{\rotatebox[origin=c]{90}{M3HC}}
			& \#CI tests
			& 708 & 747 & 808 & \textbf{1591} & 1501 & \textbf{1285}
			\\
			&
			& Runtime
			& 8.41 & 9.93 & 17.33 & 36.65 & 78.99 & 61.48
			\\
			& 
			& F1-score
			& 0.76 & 0.79 & 0.75 & 0.65 & 0.63 & 0.57
			\\
			&
			& Precision
			& \textbf{0.98} & \textbf{1.00} & 0.99 & 0.97 & 0.98 & 0.97
			\\
			&
			& Recall
			& 0.62 & 0.66 & 0.61 & 0.48 & 0.47 & 0.40
			\\
			\bottomrule
			\end{tabular}
	    \label{table: exp appendix}
    \end{table*}

\section{Specific excluded structure} \label{sec: apd chordal}
In this section, we discuss the specific structure that is excluded from the result of Theorem \ref{thm: test removability}.
Formally, this structure is a MAG $\mathcal{M}$ that contains a specific type of cycle, which we call non-chordal: A cycle $(V_0,V_1,...,V_m,V_{m+1}=V_0)$ such that I) $V_i$ and $V_{i+1}$ are neighbors for every $0\leq i\leq m$, and II) the inducing subgraph of $\mathcal{M}$ over the vertices $\{V_0,...,V_m\}$ does not contain any other edges.
We show that this certain structure of MAGs represents a very restrictive structure of the DAG $\GV$. Consider the DAG $\GV$ in Figure \ref{fig: chordal-dag}, where $\mathbfcal{O}=\{O_1,O_2,O_3,O_4\}$ and $\mathbfcal{S}=\{S_{12},S_{23},S_{34},S_{41}\}$. The corresponding MAG is shown in Figure \ref{fig: chordal-mag}. As seen in Figure \ref{fig: chordal-mag}, the non-chordal cycle $(O_1,O_2,O_3,O_4,O_1)$ appears in the MAG structure. We claim such a cycle can only happen if all of the following conditions are satisfied:
\begin{itemize}[leftmargin=*]
    \item Each pair $(O_i,O_{i+1})$ have a specific selection variable $S_{i(i+1)}$ such that $O_i,O_{i+1}\in\Anc{S_{i(i+1)}}$, and none of the other observed variables of the cycle are ancestors of $S_{i(i+1)}$. 
    Note that if for instance $O_1\in\Anc{S_{23}}$ in the example above, then $O_1$ would be adjacent to $O_3$ in $\GV[\mathbfcal{O}|\mathbfcal{S}]$, since $(O_1,S_{23},O_3)$ is an inducing path. 
    So for the resulting MAG to have a non-chordal cycle, each pair of the observed variables must have their own specific selection variable.
    \item None of the pairs of variables $(O_i,O_j)$ must be adjacent if $j\neq(i-1),(i+1)$. That is, the induced subgraph of the DAG $\GV$ over $O_i$s must not contain any edges other than the edges of the cycle. Otherwise, the cycle in MAG $\GV[\mathbfcal{O}|\mathbfcal{S}]$ would contain a chord.
    \item None of the pairs of variables $(O_i,O_j)$ must have common latent confounders if $j\neq(i-1),(i+1)$. Otherwise, as in the case above, this would form a chord in the cycle.
\end{itemize}
Not allowing the aforementioned specific structure, the result of Theorem \ref{thm: test removability} is guaranteed. Note that it is mandatory to exclude this structure, as such structures have induced sub-graphs with no removable variables.
\begin{lemma}\label{lem: non-chordal no rem}
    Suppose $\mathcal{G}$ is a MAG with non-chordal cycle $c=(O_0,...,O_m)$. None of the vertices $\{O_0,...,O_m\}$ are removable in any induced sub-graph of $\mathcal{G}$ that contains the cycle $c$.
\end{lemma}
\begin{proof}
    Suppose $\mathcal{H}$ is an induced sub-graph of $\mathcal{G}$ that contains the cycle $c$. Take an arbitrary vertex $O_i$. $O_{i-1},O_{i+1}\in \N{O_i}$, and $O_{i-1}\notin\Adj{O_{i+1}}$ since $c$ is non-chordal. From Theorem \ref{thm: graph-rep}, $O_i$ is not removable, which completes the proof.
\end{proof}
\begin{figure}[t] 
    \centering
	\tikzstyle{block} = [circle, inner sep=1.3pt, fill=black]
	\tikzstyle{input} = [coordinate]
	\tikzstyle{output} = [coordinate]
	\begin{subfigure}[b]{0.34\textwidth}
	    \centering
	    \begin{tikzpicture}
            \tikzset{edge/.style = {->,> = latex'}}
            % vertices
            \node[block] (s12) at  (0.5,2) {};
            \node[] ()[above left=-0.1cm and -0.1cm of s12]{$S_{12}$};
            \node[block] (s23) at  (2,0.5) {};
            \node[] ()[above right=-0.1cm and -0.1cm of s23]{$S_{23}$};
            
            \node[block] (s34) at  (0.5,-1) {};
            \node[] ()[above left=-0.45cm and -0.05cm of s34]{$S_{34}$};
            
            \node[block] (s41) at  (-1,0.5) {};
            \node[] ()[above left=-0.1cm and -0.1cm of s41]{$S_{41}$};
            
            \node[block] (o1) at  (0,1) {};
            \node[] ()[above left=-0.1cm and -0.1cm of o1]{$O_1$};
            \node[block] (o2) at  (1,1) {};
            \node[] ()[above right=-0.1cm and -0.05cm of o2]{$O_2$};
            \node[block] (o3) at  (1,0) {};
            \node[] ()[above right=-0.5cm and -0.1cm of o3]{$O_3$};
            \node[block] (o4) at  (0,0) {};
            \node[] ()[above left=-0.5cm and -0.1cm of o4]{$O_4$};
            %edges
            \draw[edge, dashed, bend left=20] (o1) to (s12);
            \draw[edge, dashed, bend right=20] (o2) to (s12);
            \draw[edge, dashed, bend left=20] (o2) to (s23);
            \draw[edge, dashed, bend right=20] (o3) to (s23);
            \draw[edge, dashed, bend left=20] (o3) to (s34);
            \draw[edge, dashed, bend left=20] (o4) to (s41);
            \draw[edge, dashed, bend right=20] (o4) to (s34);
            \draw[edge, dashed, bend right=20] (o1) to (s41);
        \end{tikzpicture}
    \caption{$\GV$}\label{fig: chordal-dag}
	\end{subfigure}
    \begin{subfigure}[b]{0.34\textwidth}
	    \centering
	    \begin{tikzpicture}
            \tikzset{edge/.style = {->,> = latex'}}
            % vertices
            \node[block] (o1) at  (0,1) {};
            \node[] ()[above left=-0.1cm and -0.1cm of o1]{$O_1$};
            \node[block] (o2) at  (1,1) {};
            \node[] ()[above right=-0.1cm and -0.05cm of o2]{$O_2$};
            \node[block] (o3) at  (1,0) {};
            \node[] ()[above right=-0.5cm and -0.1cm of o3]{$O_3$};
            \node[block] (o4) at  (0,0) {};
            \node[] ()[above left=-0.5cm and -0.1cm of o4]{$O_4$};
            \node[]()at (0.5,-1){};
            %edges
            \draw[edge, style={-}] (o1) to (o2);
            \draw[edge, style={-}] (o2) to (o3);
            \draw[edge, style={-}] (o3) to (o4);
            \draw[edge, style={-}] (o4) to (o1);
        \end{tikzpicture}
    \caption{$\GV[\mathbfcal{O}\vert\mathbfcal{S}]$}\label{fig: chordal-mag}
	\end{subfigure}
    \caption{A structure where every pair of observed vertices have its own specific selection variable, shared among only the variables of this pair. This results in a non-chordal MAG over the observe variables, if none of the pairs $(O_1,O_3)$ and $(O_2,O_4)$ have neither an edge in the DAG $\GV$, nor a latent common confounder.}
    \label{fig: chordal} 
    
\end{figure}
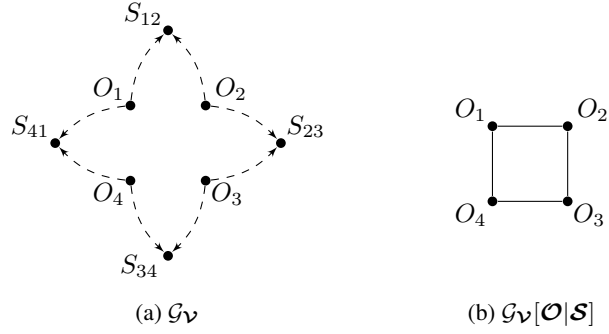
	
The following result indicates that given the aforementioned assumption, that is, if no non-chordal cycle exists in $\mathcal{M}=\GV[\mathbfcal{O}\vert\mathbfcal{S}]$, then a removable variable always exists in any subgraph of $\mathcal{M}$, which completes our discussion.
\begin{lemma}\label{lem: rem exists}
    Suppose the edge-induced subgraph of $\mathcal{M}$ over the undirected edges is chordal. 
    Let $\mathcal{G} = \GV[\mathbf{V}|\mathbfcal{S}]$ for some $\mathbf{V}\subseteq \mathbfcal{O}$.
    Then there exists $X\in\mathbf{V}$ such that $X$ is removable in $\mathcal{G}$.
\end{lemma}
\begin{proof}
    We consider the following two cases and introduce a removable variable at each case:
    \begin{enumerate}
        \item $\mathcal{G}$ has at least one directed or bidirected edge:
        Take $X$ as a vertex that has at least one arrowhead incident to it (i.e., it has at least a parent or a spouse), and satisfies the following property:
        \[\textit{De}_\mathcal{G}(X)\cap\mathbf{V}\setminusA\{X\}=\varnothing.\]
        We first show that such a vertex exists. Suppose not. Start from a vertex with an arrowhead incident to it and move to one of its children. Since the vertex we are in now has other descendants, again move to one of its children. Continuing in the same manner, we traverse over a directed cycle, which is in contradiction with the definition of MAGs.  
        
        Now we show that this variable $X$ is removable.
        Since $X$ has no other descendants, $\Ch{X}=\varnothing.$ Furthermore, $\N{X}=\varnothing$ by definition of MAG. Now Theorem \ref{thm: graph-rep} implies that $X$ is removable.
        \item $\mathcal{G}$ is an undirected graph:
        Since $\mathcal{M}$ is chordal over its undirected edges, $\mathcal{G}$ is chordal too.
        Chordal graphs have a perfect elimination ordering \cite{fulkerson1965incidence, blair1993introduction}.
        Let $X$ be the first vertex in this ordering.
        By definition of perfect elimination ordering, all of the neighbors of $X$ are adjacent.
        From Theorem \ref{thm: graph-rep}, $X$ is removable
    \end{enumerate}
\end{proof}	

Lemmas \ref{lem: non-chordal no rem} and \ref{lem: rem exists} indicate that the assumption that the induced subgraph of $\mathcal{M}$ on the undirected edges is chordal is the necessary and sufficient condition so that there exists a removable variable at every subgraph of $\mathcal{M}$.

\end{document}